%% file: icml2026.tex
\documentclass{article}

\usepackage{microtype}
\usepackage{graphicx}
\usepackage{booktabs} 

\usepackage[accepted]{icml2026}

\usepackage{hyperref}
\hypersetup{
    colorlinks = true,
    linkcolor = red,
    anchorcolor = red,
    urlcolor = red,
    citecolor = Blue,
}

\usepackage{amsmath}
\usepackage{amssymb}
\usepackage{mathtools}
\usepackage{amsthm}
\usepackage{thmtools}

\theoremstyle{plain}
\newtheorem{theorem}{Theorem}[section]

\newtheorem{lemma}[theorem]{Lemma}

\theoremstyle{definition}

\newtheorem{assumption}[theorem]{Assumption}
\theoremstyle{remark}

\usepackage[textsize=tiny]{todonotes}

\usepackage{tcolorbox}
\usepackage{cuted}
\usepackage{url}
\usepackage{subcaption}     
\usepackage{multirow}
\usepackage{makecell}
\usepackage{bbm}
\usepackage{enumitem}
\usepackage{lipsum}        

\usepackage[dvipsnames]{xcolor}
\usepackage{xspace}
\makeatletter
\DeclareRobustCommand\onedot{\futurelet\@let@token\@onedot}
\def\@onedot{\ifx\@let@token.\else.\null\fi\xspace}
\def\eg{\emph{e.g}\onedot} 
\def\ie{\emph{i.e}\onedot}

\makeatother
\usepackage{wrapfig}

\input{math_commands.tex}

\icmltitlerunning{Mitigating Staleness in Asynchronous Pipeline Parallelism via \SOLUTION{}}

\usepackage[capitalize,noabbrev]{cleveref}

\Crefname{section}{Section}{Sections}
\Crefname{table}{Table}{Tables}
\Crefname{figure}{Figure}{Figures}
\Crefname{assumption}{Assumption}{Assumptions}
\Crefname{lemma}{Lemma}{Lemmas}
\Crefname{appendix}{Appendix}{Appendices}



\begin{document}

\twocolumn[
  \icmltitle{Mitigating Staleness in Asynchronous Pipeline Parallelism via \SOLUTION{}}



  \icmlsetsymbol{equal}{*}

  \begin{icmlauthorlist}
    \icmlauthor{Hyunji Jung}{equal,postech}
    \icmlauthor{Sungbin Shin}{equal,postech}
    \icmlauthor{Namhoon Lee}{postech}
  \end{icmlauthorlist}

  \icmlaffiliation{postech}{POSTECH}

  \icmlcorrespondingauthor{Namhoon Lee}{namhoon.lee@postech.ac.kr}

  \icmlkeywords{Machine Learning, ICML}

  \vskip 0.3in
]



\printAffiliationsAndNotice{\icmlEqualContribution}



\input{tex/0_abstract}

\input{tex/1_introduction}
\input{tex/3_analysis}
\input{tex/4_solution}
\input{tex/5_experiments}
\input{tex/6_related_work}
\input{tex/7_conclusion}

\input{tex/acknowledgements}

\input{tex/impact_statement}


\bibliography{reference}
\bibliographystyle{icml2026}

\newpage
\onecolumn
\tableofcontents
\newpage
\appendix
\onecolumn
\input{tex/appendix}


\end{document}

%% file: math_commands.tex
\usepackage{amsmath, amsfonts, amsthm, bm}

\def\eqref#1{equation~\ref{#1}}

\def\1{\bm{1}}

\newcommand{\Vector}[1]{\operatorname{vec}\!\left(#1\right)}

\DeclareMathAlphabet{\mathsfit}{\encodingdefault}{\sfdefault}{m}{sl}
\SetMathAlphabet{\mathsfit}{bold}{\encodingdefault}{\sfdefault}{bx}{n}

\def\gG{{\mathcal{G}}}

\def\gS{{\mathcal{S}}}

\def\gU{{\mathcal{U}}}

\newcommand{\E}{\mathbb{E}}

\newcommand{\R}{\mathbb{R}}

\DeclareMathOperator{\Tr}{Tr}

\def\R{\mathbb{R}}

\newcommand{\poly}{\mathtt{poly}\xspace}

\def\solution{basis rotation}
\def\Solution{Basis rotation}
\def\SOLUTION{Basis Rotation}
\def\misalign{basis misalignment}

\def\MISALIGN{Basis Misalignment}
\def\ourmethod{basis rotation}
\def\Ourmethod{Basis rotation}
\def\OURMETHOD{Basis Rotation}
\def\estimation{\texttt{eigenbasis-estimation}}

\def\onesided{unilateral}

\def\twosided{bilateral}

\def\covariance{$2^{\text{nd}}$}
\def\mean{$1^{\text{st}}$}
\newcommand{\smn}{c}
\newcommand{\Smn}{C}

\newcommand{\update}[1]{{\color{blue}#1}}

%% file: tex/0_abstract.tex
\begin{abstract}
Asynchronous pipeline parallelism maximizes hardware utilization by eliminating the pipeline bubbles inherent in synchronous execution, offering a path toward efficient large-scale distributed training.
However, this efficiency gain can be compromised by gradient staleness, where the immediate model updates with delayed gradients introduce noise into the optimization process.
Crucially, we identify a critical, yet often overlooked, pathology: this delay scales linearly with pipeline depth, fundamentally undermining the very scalability that the method originally intends to provide.
We trace this pathology to a specific property of the optimization landscape: the misalignment between the Hessian eigenbasis and the standard coordinate basis, which triggers oscillations in the update trajectories of coordinate-wise adaptive optimizers. 
We identify that these oscillations cause delayed updates to diverge from their true counterparts, invalidating their use for current iterations.
This insight is formalized through theoretical analysis, including a convergence bound showing that basis misalignment amplifies the delay penalty, and substantiated with empirical evaluation.
To address this, we propose \ourmethod{}, a framework that rotates the optimizer's coordinate system to align with the Hessian eigenbasis, keeping delayed updates useful.
We theoretically demonstrate that \ourmethod{} minimizes basis misalignment, thereby counteracting the conditions that amplify delay penalties.
Empirically, in training up to a 3B-parameter LLM, basis rotation reduces the required iterations by 81.7\% compared to the best-performing asynchronous baseline.\footnote{Our code is available at \url{https://github.com/LOG-postech/basis-rotation}.
}
\end{abstract}

%% file: tex/1_introduction.tex
\section{Introduction}

Training large-scale LLMs requires partitioning the model across multiple devices, as the memory footprint of such models far exceeds the capacity of individual accelerators.
Pipeline parallelism addresses this by dividing the model into sequential stages, each allocated to a separate device.
As a result, it has become a cornerstone of LLM training---alongside data, tensor, and context parallelism---as models continue to grow \citep{dubey2024llama,adler2024nemotron,liu2024deepseek,yang2025qwen2,team2025kimi}.

However, the efficiency of this approach is fundamentally constrained by its synchronous design, which mandates that each stage wait for the completion of backward passes of every other stage before updating its weights \citep{huang2019gpipe, fan2021dapple, li2021chimera}.
This dependency results in suboptimal hardware utilization by introducing significant idle periods, commonly referred to as pipeline bubbles.

Asynchronous pipeline parallelism aims to mitigate these idle periods by allowing each stage to proceed with subsequent computations without waiting for the completion of backward passes of other stages \citep{narayanan2019pipedream, narayanan2021memory}.
While this approach significantly increases hardware utilization, it introduces gradient staleness as a consequence of the temporal gap between gradient calculation and application; \ie, gradients arrive at the update step after the model has undergone multiple intervening weight updates (see \cref{fig:bg}).
This delayed arrival has emerged as a primary challenge in asynchronous training, as it often degrades convergence stability and final model performance \citep{yang2021pipemare,ajanthan2025nesterov}. 

\begin{figure}[!t]
    \centering
    \begin{subfigure}{\linewidth}
        \includegraphics[width=0.9\linewidth,trim={2cm 7cm 2cm 5cm},clip]{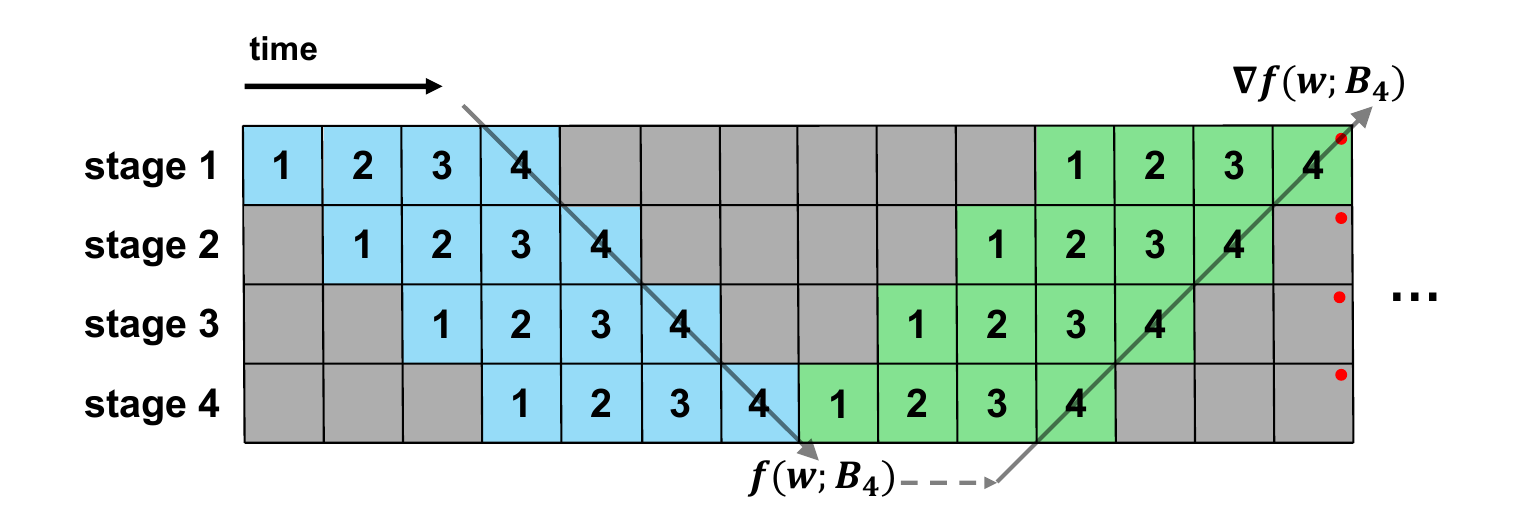}
        \caption{Synchronous pipeline parallelism}
        \label{fig:bg-sync}
    \end{subfigure}
    \begin{subfigure}{\linewidth}
        \includegraphics[width=0.9\linewidth,trim={2cm 8cm 2cm 5cm},clip]{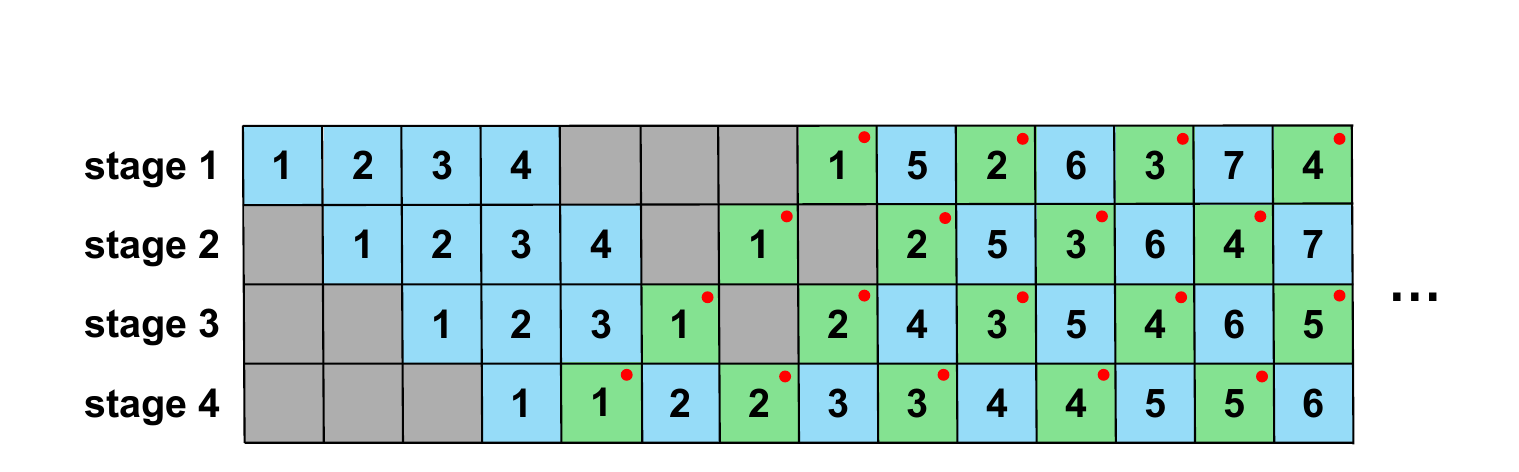}
        \caption{Asynchronous pipeline parallelism}
        \label{fig:bg-async}
    \end{subfigure}    
    \begin{subfigure}{\linewidth}
        \includegraphics[width=0.9\linewidth,trim={2cm 8cm 2cm 5cm},clip]{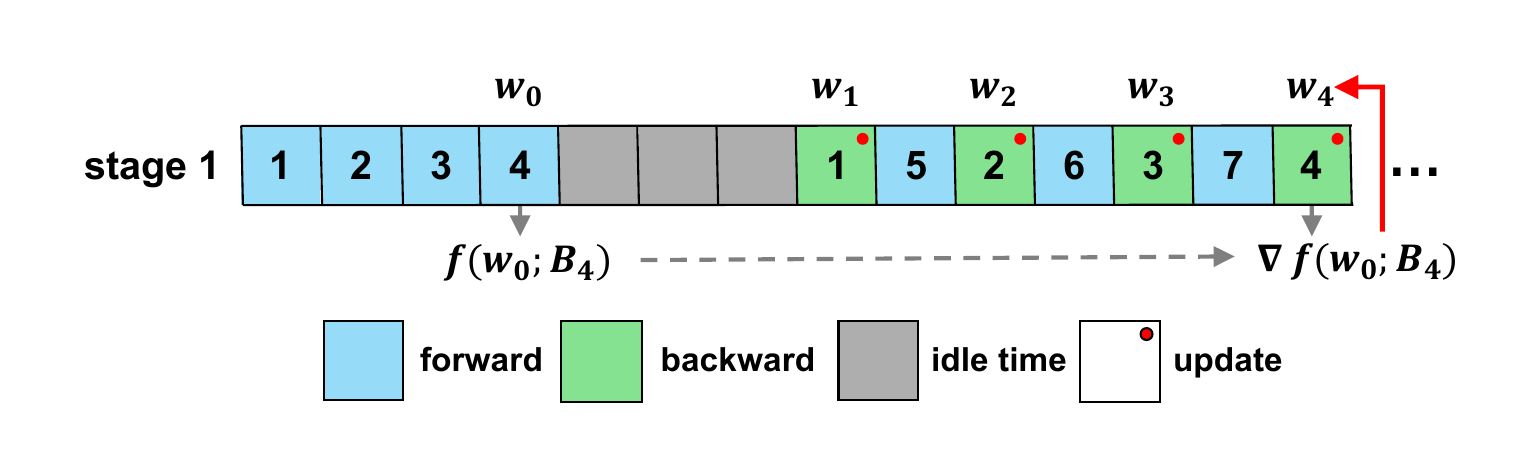}
        \caption{Delay in asynchronous pipeline parallelism}
        \label{fig:bg-async-delay}
    \end{subfigure}    
    \caption{
    (a--b) Schematic diagrams showing how micro-batches (blocks 1--7) are processed over time:
    a micro-batch travels from stage 1 to stage 4 in a forward pass (blue), and then goes back to stage 1 through a backward pass (green).
    Once the gradient becomes available after the backward pass, the model is updated (red dots indicate the time points).
    Asynchronous pipelining removes idle periods by processing subsequent micro-batches immediately after completing a backward pass without waiting for the completion of the pipeline cycle.
    (c)
    An illustration of model update with delayed gradient at stage 1: $w_3$ is updated to $w_4$ with $\nabla f(w_0; B_4)$.
    }
    \label{fig:bg}
    \vspace{-1em}
\end{figure}

Notably, we find that gradient staleness presents a critical bottleneck, particularly for large models.
This is because gradient delay increases alongside model size and pipeline depth; indeed, the number of stages can easily reach tens or hundreds in large-scale configurations\footnote{We provide an analysis of model-stage scaling in \cref{app:stage-analysis}.}.
In fact, we observe that increasing the number of stages for a fixed model results in a drastic $5.81$-fold slowdown in convergence speed (see \cref{fig:teaser-stages}).
This suggests that staleness is not merely a nuisance, but a fundamental barrier to the scalability of asynchronous pipelining, despite the fact that asynchronous pipeline parallelism aims to facilitate large-scale training.

Our analysis reveals that this degradation is deeply rooted in the interaction between delayed updates and the characteristics of Adam \citep{kingma2014adam}, the de facto optimizer for LLM pre-training \citep{touvron2023llama2,dubey2024llama,liu2024deepseek}.
Specifically, we identify \misalign{}—a condition where the Hessian eigenbasis is not aligned with the standard coordinate basis—as the central reason staleness damages convergence.
Since Adam's coordinate-wise adaptivity becomes ineffective under \misalign{}, the update direction changes rapidly between when a gradient is computed and when it is applied, so the delayed update no longer matches the non-delayed counterpart.
We substantiate this intuition through empirical observations and convergence analysis, showing that misalignment directly amplifies the impact of staleness on convergence.

\begin{figure}[!t]
    \centering
    \begin{subfigure}{0.48\linewidth}
        \includegraphics[width=\linewidth]{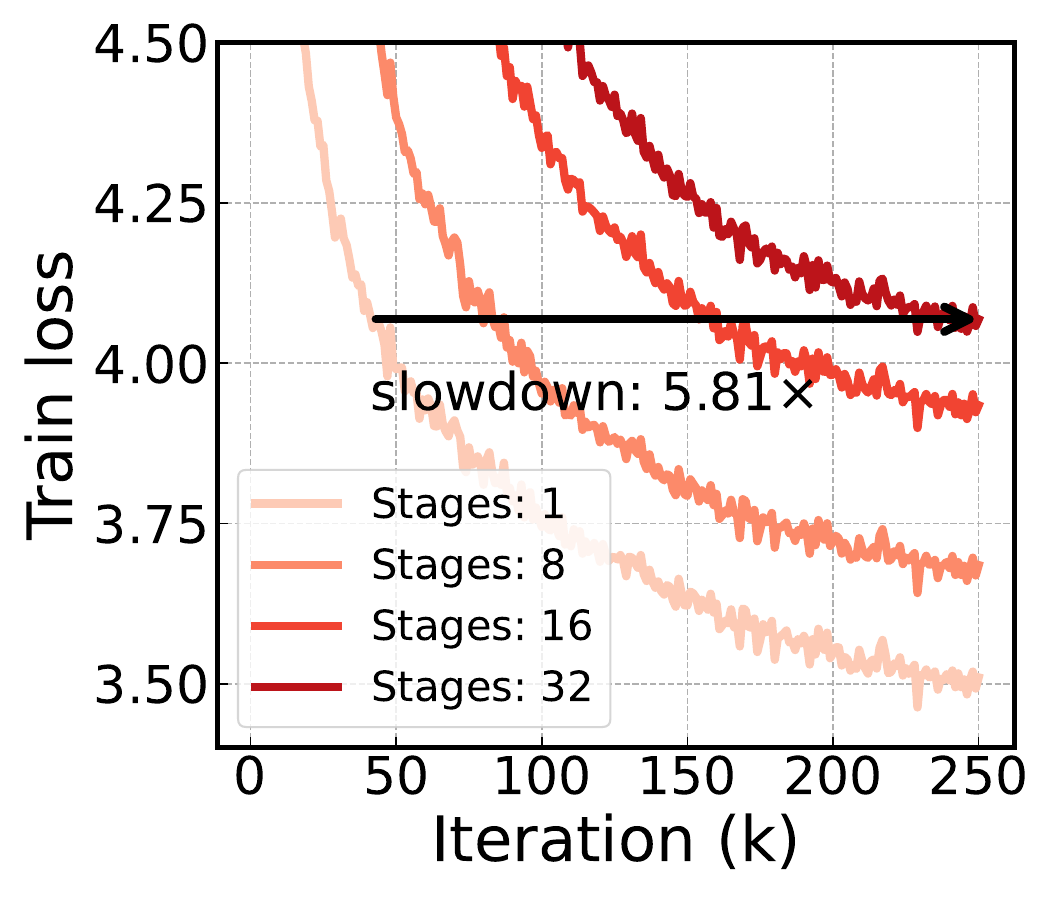}
        \caption{Impact of delay}
        \label{fig:teaser-stages}
    \end{subfigure}
    \begin{subfigure}{0.48\linewidth}
        \includegraphics[width=\linewidth]{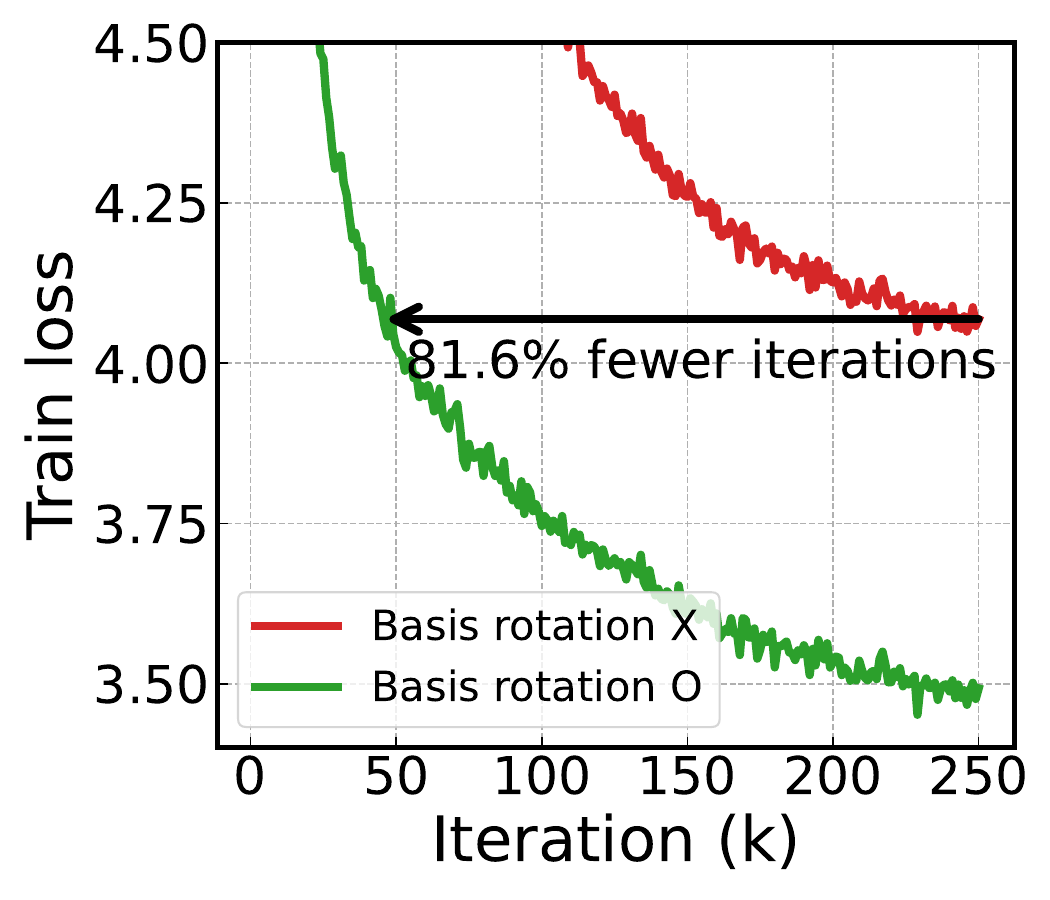}
        \caption{Effectiveness of solution}
        \label{fig:teaser-rotation}
    \end{subfigure}    
    \caption{
    Summary of this work.
    (a) Impact of pipeline depth (\ie, number of stages) on convergence of asynchronous pipeline parallel LLM pre-training.
    In all cases, the model itself is kept the same while the number of stages is divided to be different.
    Increased delay leads to significant degradation on convergence speed.
    (b) \Solution{} substantially accelerates convergence in the presence of a large delay (here, for the case of 32 stages).
    }
    \label{fig:teaser}
    \vspace{-1em}
\end{figure}

To address this, we propose mitigating the impact of staleness through \solution{}, a framework that transforms the optimization space to align the Hessian eigenbasis with the standard basis.
By the transformation, it effectively leverages the curvature-aware adaptivity and straightens the local trend of the update trajectory.
This linearized trajectory keeps the delayed update directions better aligned with the actual optimization path.
Consequently, basis rotation allows the use of delayed gradient information while mitigating staleness in the update direction.
We further introduce several practical rotation strategies using different Hessian approximations and provide a theoretical analysis of their respective approximation qualities.

Empirical evaluations on LLM pre-training demonstrate that our proposed solution significantly neutralizes the degradation caused by gradient delay;
for example, \solution{} helps to achieve the same training loss in up to $81.6\%$ fewer iterations than the standard asynchronous pipeline parallel training (\cref{fig:teaser-rotation}).
Furthermore, introducing a stage-aware rotation strategy that allocates computational budget proportionally to per-stage delay yields an additional 29.2\% speedup.
These results suggest that \solution{} is a vital component for enabling high-fidelity, large-scale asynchronous pipeline parallelism.

Our key contributions are summarized as follows:
\begin{itemize}[noitemsep,nolistsep,topsep=-\parskip, leftmargin=2ex]
    \item \underline{Identifying a critical issue in scaled pipelines}:
    We show that asynchronous pipeline parallel training suffers from significant convergence and model performance degradation as the number of stages increases, which has received limited formal study in the literature.
    \item \underline{Demystifying the mechanism of failure}:
    We point out \misalign{} as the primary reason Adam-type optimizers are sensitive to delay.
    We provide both empirical evidence and theoretical convergence analysis to show how this misalignment exacerbates the negative impact of stale gradients (\cref{sec:analysis}).
    \item \underline{The \solution{} framework}:
    To mitigate these effects, we propose \solution, a method designed to realign the optimization trajectory and counteract the delay-induced instability under \misalign{} (\cref{sec:solution}).
    \item \underline{Empirical validation}:
    We provide extensive experimental results, including 3B-scale pre-training, demonstrating that our solution effectively alleviates the performance penalties inherent in asynchronous pipelines (\cref{sec:experiments}).
\end{itemize}

%% file: tex/3_analysis.tex
\section{Understanding the Impact of Delay}
\label{sec:analysis}

\begin{figure}[t]
  \centering
  \begin{subfigure}{0.49\columnwidth}
    \centering
    \includegraphics[width=\linewidth]{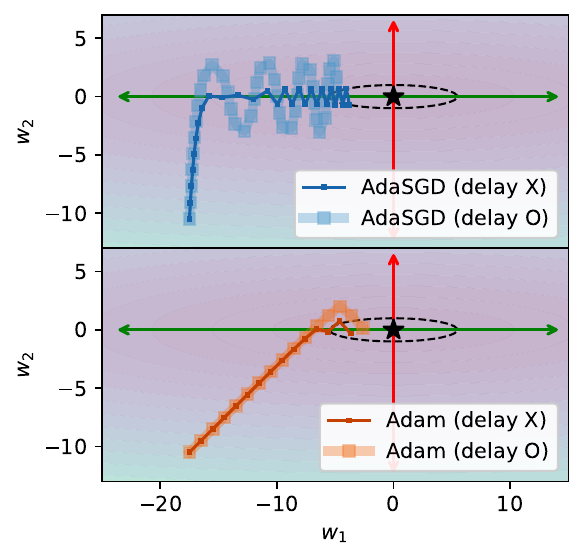}
  \end{subfigure}\hfill
  \begin{subfigure}{0.49\columnwidth}
    \centering
    \includegraphics[width=\linewidth]{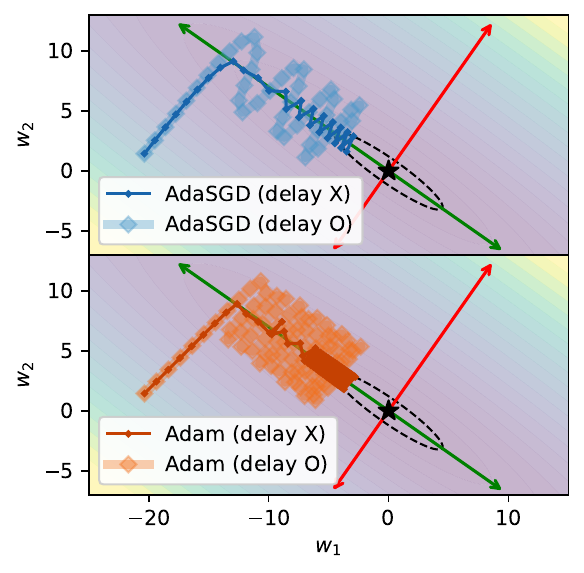}
  \end{subfigure}

  \vspace{0.4em}

  \begin{subfigure}{0.49\columnwidth}
    \centering
    \includegraphics[width=\linewidth]{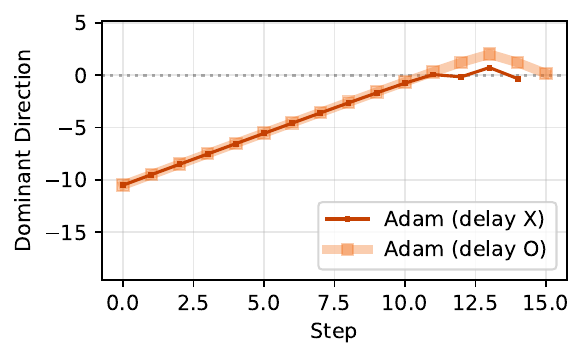}
    \caption{Basis-aligned}
    \label{fig:quadratic-dominant_aligned}
  \end{subfigure}\hfill
  \begin{subfigure}{0.49\columnwidth}
    \centering
    \includegraphics[width=\linewidth]{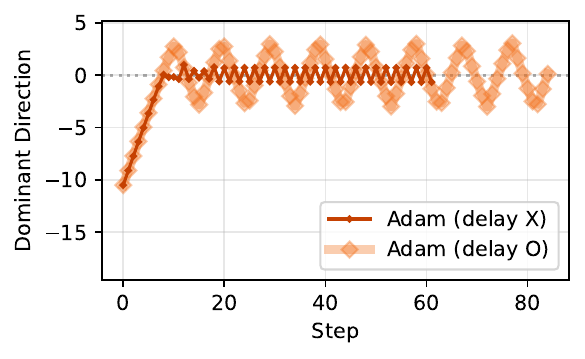}
    \caption{Basis-misaligned}
    \label{fig:quadratic-dominant_misaligned}
  \end{subfigure}

\caption{
Impact of basis alignment on the effect of delay.
(Top) Trajectories of \textcolor{blue}{AdaSGD} and \textcolor{orange}{Adam} with and without delay.
(Bottom) Update of Adam along the dominant direction (red arrows in the top panel).
(a) When the Hessian eigenbasis is aligned with the standard coordinate basis, Adam yields a stable trajectory and robust to delay; under misalignment (b), it oscillates along the dominant direction, and training suffers from delay.
See \cref{app:sub:toy-exp-details} for the explanation of the use of AdaSGD and experimental details.
}
\label{fig:quadratic}
\vspace{-1.5em}
\end{figure}
While asynchronous pipeline parallelism offers significant throughput advantages, we find that the resulting gradient delay can lead to pathological training behaviors where increasing model scale can critically damage model performance.
To resolve this issue, we begin by understanding the fundamental mechanisms of how delayed gradients interact with the optimization landscape.
We attribute the degradation from delay primarily to \misalign{}.
Under such conditions, coordinate-wise adaptive optimizers such as Adam become vulnerable to delays.

\subsection{Root of Degradation: \MISALIGN{}} 
\label{subsec:mechanism}
The heterogeneous curvature of Transformer loss landscapes renders a single global step size ineffective \citep{zhang2024transformers,zhang2024adam}, establishing coordinate-wise adaptive optimizers like Adam as the de facto standard \citep{touvron2023llama2,dubey2024llama}.
By adapting step sizes for each parameter individually, these optimizers effectively mitigate the oscillations that typically arise when the step size is excessive relative to local curvature \citep{kingma2014adam,zhang2020gradientclipping,pan2023toward}.

To illustrate this effect, consider the quadratic objective $\min_w \frac{1}{2} w^\top H w$ (see \cref{fig:quadratic}), first without delay.
When the Hessian $H$ is diagonal, as depicted in \cref{fig:quadratic-dominant_aligned}, AdaSGD~\citep{wang2020adasgdbridginggapsgd}---which applies a single learning rate uniformly across all coordinates---exhibits oscillatory updates along the dominant eigenvector direction (red arrows) while progressing slowly along non-dominant directions (green arrows).
In contrast, Adam effectively suppresses these oscillations, yielding a nearly direct trajectory toward the optimum.
However, this coordinate-wise adaptivity critically depends on the alignment between the Hessian eigenbasis and the coordinate system underlying the adaptive method.
As shown in \cref{fig:quadratic-dominant_misaligned}, under \misalign{}, Adam's effective adaptivity diminishes and its trajectory resembles the behavior of AdaSGD with severe oscillations occur along the dominant eigenvector direction. 

Crucially, we identify this oscillation as the primary mechanism that amplifies sensitivity to delay.
When the trajectory oscillates rapidly, delayed gradients are likely to be stale---pointing in outdated or even adversarial directions relative to the current iterate---thereby substantially degrading convergence, as in \cref{fig:quadratic-dominant_misaligned}.
On the other hand, a smooth trajectory with reduced oscillations ensures that the delayed update remains closely aligned with its non-delayed counterpart, rendering Adam robust to gradient staleness, as in \cref{fig:quadratic-dominant_aligned}.
When the update directions $\{u_t\}$ are locally consistent (i.e., $u_t \approx u$), the delayed trajectory $\{\tilde{x}_t\}$ closely approximates the original trajectory $\{x_t\}$, provided the gradient noise is small.
We briefly illustrate this via induction (see \cref{app:linear_traj_analysis} for details).
Assume $\tilde{x}_i \approx x_i$ for $i \leq k-1$. Then, the moments $\tilde{m}_k$ and $\tilde{v}_k$ computed with delayed gradients $\tilde{g}_k$ at $\tilde{x}_{k-1-\tau}$ are approximated as follows:
\begin{align}
\tilde{m}_k &= \beta_1 \tilde{m}_{k-1}+(1-\beta_1)\tilde{g}_{k} \nonumber\\ 
&\approx \beta_1 m_{k-\tau-1}+(1-\beta_1)g_{k-\tau}=m_{k-\tau}\\
\tilde{v}_k &= \beta_2 \tilde{v}_{k-1}+(1-\beta_2)\tilde{g}_{k}^2 \nonumber\\
&\approx \beta_2 v_{k-\tau-1}+(1-\beta_2)g_{k-\tau}^2=v_{k-\tau},
\end{align}
Consequently, the resulting update direction satisfies $\tilde{u}_k \approx u_{k-\tau} \approx u_k$, ensuring that the subsequent iterate $\tilde{x}_{k}$ remains aligned with its non-delayed counterpart $x_{k}$.


\subsection{Empirical Observation}
\label{subsec:analysis-empirical}

Building on the high-level intuition from quadratic optimization, we design a spiral loss landscape to better simulate the non-stationary geometry of deep neural networks where the Hessian eigenbasis evolves along the training trajectory.
Figure~\ref{fig:spiral_full} shows the optimization trajectory of Adam trained without delay.
At each randomly selected point along this trajectory, we measure the slowdown ratio $T_{\text{delay}} / T_{\text{no-delay}}$, defined as the iteration ratio required to advance a fixed angular interval with and without delay.
Each of these measurements corresponds to a specific point plotted in \cref{fig:delay_conv_speed}.
See \cref{app:sub:toy-exp-details} for further details.


\begin{figure}[t]
    \centering 
    \begin{subfigure}{0.8\linewidth}
        \includegraphics[width=\linewidth]{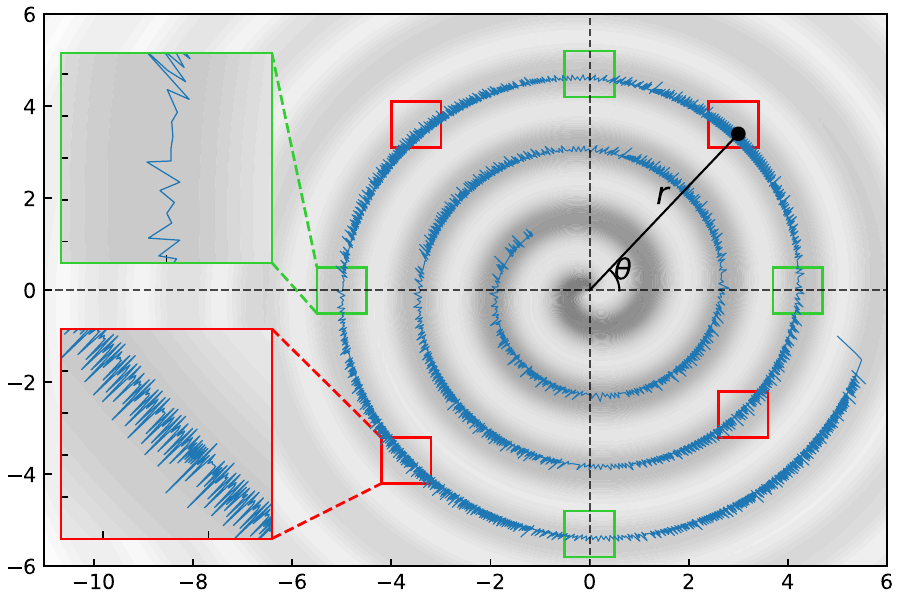}
        \caption{Optimization trajectory}
        \label{fig:spiral_full}
    \end{subfigure}
    \vspace{0.2cm} 
    \begin{subfigure}{\linewidth}
        \includegraphics[width=\linewidth]{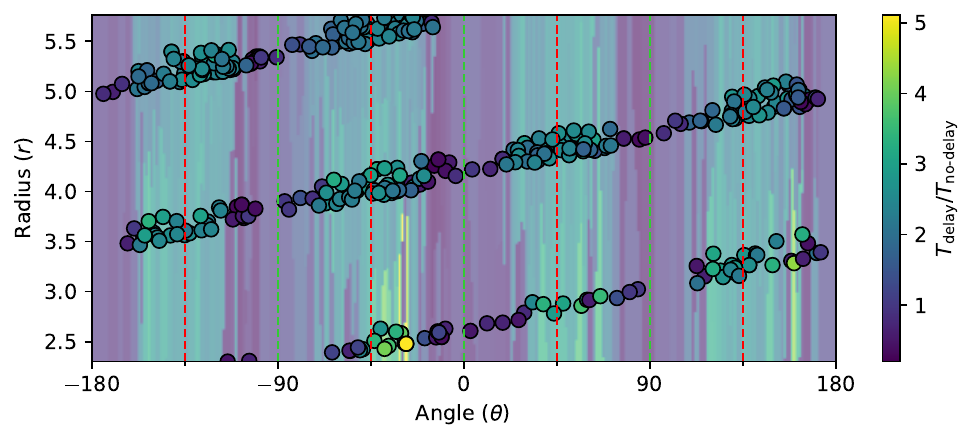}
        \caption{Impact of delay on convergence speed}
        \label{fig:delay_conv_speed}
    \end{subfigure}
    
    \caption{
    (a) Optimization trajectory of Adam on a spiral loss.
    The optimizer maintains a stable trajectory in basis-aligned regions (green boxes) but exhibits severe oscillations in misaligned regions (red boxes).
    (b) Slowdown ratio $T_{\text{delay}}/T_{\text{no-delay}}$ in different regions. 
    The ratio is minimized near basis-aligned regions (green dotted lines), whereas it is maximized in misaligned regions (red dotted lines). 
    This demonstrates that \misalign{} significantly amplifies the deleterious effects of delayed gradients.}
    \label{fig:toy-example}
\end{figure}

As shown in \cref{fig:spiral_full}, Adam exhibits a stable trajectory with minimal oscillation when the Hessian eigenbasis aligns with the standard coordinate basis (see green boxes).
In contrast, significant oscillation along the dominant direction emerges for the misaligned regions (see red boxes).
\cref{fig:delay_conv_speed} directly correlates this instability with sensitivity to delay.
We observe that the slowdown ratio is minimized near basis-aligned regions, indicating that delay has a negligible impact on convergence (see green dotted lines). 
In contrast, this ratio increases in misaligned regions where delay substantially impedes the convergence (see red dotted lines).
These results support that \misalign{} renders coordinate-wise adaptive scheme of Adam-type optimizers susceptible to delay-induced convergence instability.

\subsection{Theoretical Verification}
\label{sec:analysis-convergence}

Beyond the empirical observations, we provide a formal convergence analysis to demonstrate that Adam suffers significantly from delayed gradients under \misalign{}.
Specifically, we prove that the negative effect of delay $\tau$ on the convergence rate is exacerbated for the basis-misaligned case compared to the aligned case.
We first present assumptions necessary for our analysis.
\begin{assumption}
\label[assumption]{assumption:bounded-stochasticity}
    (Coordinate-wise bounded noise)
    For each coordinate $i \in \{1, \cdots, d\}$, there exists $\sigma_i$ such that 
    $\E_\xi [\nabla_i f(w; \xi) - \E_\xi [\nabla_i f(w; \xi)]]^2 \leq \sigma_i^2$ for all $w \in \R^d$ where $\xi$ denotes the stochasticity of data.
\end{assumption}

\begin{assumption}
\label[assumption]{assumption:smoothness}
    ($\smn$-coordinate-wise $\ell_\infty$ smoothness)
    $f$ is $c$-smooth coordinate-wisely w.r.t. $\ell_\infty$ norm with $\smn = (\Smn_1, \cdots, \Smn_d) \in \R^d$, \ie,
    for each coordinate $i \in \{1, \cdots, d\}$,
    $ \lvert \nabla_i f(w) - \nabla_i f(w') \rvert \leq \Smn_i \| w - w' \|_\infty $ for any $w,w' \in \R^d$.
    This implies that $f$ is $\Smn$-smooth w.r.t. $\ell_\infty$ norm where $\Smn = \sum_{i=1}^d \Smn_i$, \ie, $\| \nabla f(w) - \nabla f(w') \|_1 \leq \Smn \|w - w' \|_\infty $ for all $w,w' \in \R^d$.
\end{assumption}

These coordinate-wise assumptions enable capturing Adam's coordinate-wise adaptivity in the convergence rate.
Specifically, $\Smn$ locally corresponds to the $(1,1)$-norm of the Hessian matrix, defined as $\| \nabla^2 f(w) \|_{1,1} := \sum_{i,j} | (\nabla^2 f(w))_{i,j} |$.
For a fixed eigenvalue spectrum, this norm is minimized when the Hessian is diagonal and increases under basis misalignment.
Thus, it acts as a promising proxy for basis misalignment \citep{xie2024adam,zhang2024understanding}.
We are now ready to present the convergence theorem for asynchronous Adam under delay.
\begin{theorem}
\label{thm:convergence}
Let $f$ be a non-convex function where \cref{assumption:bounded-stochasticity,assumption:smoothness} hold.
Define $\Delta_0 \triangleq (f(w_0) - \min_w f(w)) $ as the initial suboptimality gap.
Assume that initial second moment $v_0$ and step size $\eta$ satisfy $v_0 + \epsilon > (\sum_{i=1}^d \sigma_i^2 + \| \nabla f(w_0) \|_\infty^2 + \sum_{i=1}^d \Smn_i^2 \eta^2) / \poly(T)$.
Then, with appropriate choices of $\eta$ and $\beta_2$, the convergence rate of asynchronous Adam with $\beta_1=0$ under delay $\tau$ is as follows:
\begin{align*}
 \min_{\frac{T}{2} < t \leq T}&  \E \| \nabla f(w_t) \|_1 \\ = \mathcal{O} & \left (\sqrt{\frac{(1 +  d \tau) \Delta_0 \Smn}{T}}  +\sqrt{\sum_{i=1}^d \sigma_i} \left(\frac{(1 + d \tau) \Delta_0 \Smn}{T}\right)^{1/4} \right. \\ & \left. + \sum_{i=1}^d \sigma_i \left(\frac{\log T}{T}\right)^{1/4}\right).
\end{align*}
\end{theorem}

The proof is presented in \cref{app:convergence-proof}.

This result reveals that the impact of delay $\tau$ is tightly coupled with basis misalignment $\Smn$ as $\tau$ enters the bound multiplicatively with $\Smn$. 
Specifically, for a fixed delay $\tau$, the relative contribution of the delay-dependent terms to the total bound increases with $\Smn$, thereby exacerbating the penalty for basis misalignment.
Thus, while the effects of delay are suppressed for basis-aligned cases (\ie, small $\Smn$), they are aggressively amplified for misaligned cases (\ie, large $\Smn$) where the delay-dependent terms begin to dominate the bound.
This formalizes our empirical observation that the impact of delay increases when the Hessian eigenbasis is misaligned with the standard basis.

\paragraph{Remark} 
{The result indicates that delay $\tau$ slows down the deterministic convergence rate from $\mathcal{O}(\sqrt{1/T}$) to $\mathcal{O}(\sqrt{\tau / T})$, which aligns with the analysis from the asynchronous optimization literature \citep{stich2019error,koloskova2022sharper}.
The additional dimension factor $d$ and the $\tau$ dependence appearing in the stochastic term stem from the coordinate-wise adaptive denominators of Adam.}
Also, when $\tau=0$, the result recovers the convergence rate of Adam under $\ell_\infty$ smoothness without delay \citep{xie2024adam}.
The extension to $\beta_1>0$ is given in \cref{app:convergence-momentum}.

\paragraph{Extension to Stage-Dependent Delay} 
Beyond the uniform delay setting, we extend the analysis to the stage-dependent delay setting where each stage incurs a distinct delay, to provide a more precise characterization of the convergence behavior in practical asynchronous pipeline parallelism. 
Let $\mathcal{S}_1, \dots, \mathcal{S}_K$ denote the partition of parameter 
coordinates across $K$ pipeline stages, where coordinate $i \in \mathcal{S}_k$ 
incurs a delay $\tau_i = K - k$.
The bound in \cref{thm:convergence} then carries over with $\tau$ replaced by 
a tighter, stage-aware effective delay:
\begin{equation}
    \scalebox{0.97}{$\displaystyle
    \tau' \triangleq \sqrt{\frac{\sum_{i=1}^d \Smn_i^2 \tau_i^2}
    {\sum_{i=1}^d \Smn_i^2}} = \sqrt{\frac{\sum_{k=1}^K (K-k)^2 
    \sum_{i \in \mathcal{S}_k} \Smn_i^2}{\sum_{i=1}^d \Smn_i^2}}$}.
    \label{eq:effective-delay}
\end{equation}
This formulation reveals that earlier pipeline stages, where the $\tau_i$ is most severe, have a larger impact on $\tau'$ and thus dominate the convergence degradation.
The formal statement and proof are provided as \cref{thm:stage-dependent-convergence} in \cref{app:convergence-stage-delay}.

%% file: tex/4_solution.tex
\section{Mitigating the Impact of Delay}
\label{sec:solution}

Our previous analysis suggests that the negative effects of delay can be mitigated in a basis-aligned space.
We propose to achieve this through \solution{} which transforms the optimization space to align the Hessian eigenbasis with the standard basis.

\subsection{\SOLUTION{}}
\label{subsec:solution}
\begin{algorithm}[t]
\caption{Adam with \SOLUTION{}}
\begin{algorithmic}[1]
\FOR{$t = 1,2,\dots,T$}
    \STATE Sample batch $B_t$
    \STATE $G_t \gets \nabla f_W(W_{t-1};B_t) \in \mathbb{R}^{m\times n}$
    \STATE $M_t \gets \beta_1 M_{t-1} + (1-\beta_1) G_t$
    \IF{$t \bmod \text{freq} = 0$}
    \STATE $U,V$\scalebox{0.9}{$\gets$ \texttt{Eigenbasis-Estimation}$(G_t, M_t, U, V)$}
    \ENDIF
    \STATE $\widetilde{G}_t \gets U^\top G_t V$
    \STATE $\widetilde{M}_t \gets U^\top M_t V$
    \STATE $\widetilde{V}_t \gets \beta_2 \widetilde{V}_{t-1} + (1-\beta_2)\widetilde{G}_t\odot\widetilde{G}_t$
    \STATE $W_{t} \gets W_{t-1} - \eta_t
    U\!\left(\frac{1}{\sqrt{\widetilde{V}_t + \epsilon}}
    \odot \widetilde{M}_t\right)V^\top$
\ENDFOR
\end{algorithmic}
\label{algo:adam_basis_rotation}
\end{algorithm}

Standard Adam update
\begin{equation}\label{section4:eq:adam}
    w_{t} = w_{t-1} - \eta_{t-1} \frac{\text{EMA}(\nabla f(w_{t-1}))}{\sqrt{\text{EMA}(\nabla f(w_{t-1})^2) + \epsilon} }
\end{equation}
relies on coordinate-wise scaling.
As demonstrated in \cref{sec:analysis}, this scaling becomes ineffective when the Hessian eigenbasis is misaligned with the standard basis, causing delayed gradients to introduce significant noise into the optimization process.
The core idea of \solution{} is to  mitigate the impact of delay by taking an optimization step in the basis-aligned coordinate system $\tilde{w}=\gU^\top w$:
\begin{equation}\label{section4:eq:rotated}
    \tilde{w}_{t} = \tilde{w}_{t-1} - \eta_{t-1} \frac{\text{EMA}(\gU^\top \nabla f(\gU \tilde{w}_{t-1}))}{\sqrt{\text{EMA}((\gU^\top \nabla f(\gU \tilde{w}_{t-1}))^2) + \epsilon}},
\end{equation}
where $\gU$ is a rotation matrix whose columns are estimated eigenvectors of the Hessian. Projecting the update back onto the original space gives
\begin{equation}\label{section4:eq:rotated2}
    w_{t} = w_{t-1} - \eta_{t-1} \gU \frac{\text{EMA}(\gU^\top \nabla f(w_{t-1}))}{\sqrt{\text{EMA}((\gU^\top \nabla f(w_{t-1}))^2) + \epsilon }}.
\end{equation}
Crucially, when $\gU$ perfectly aligns with the Hessian eigenbasis, the Hessian eigenbasis in the rotated space matches exactly with the standard coordinate basis.
By ensuring that adaptive scaling is applied in the basis-aligned space, \solution{} restores the efficacy of curvature-aware adaptivity and mitigates the adverse effects of gradient delay. Detailed derivation of \cref{section4:eq:rotated2} is given in \cref{app:adam_rotation_equi}.

The full procedure of \solution{} for a weight matrix $W \in \mathbb{R}^{m \times n}$ is detailed in \cref{algo:adam_basis_rotation}.
To mitigate the prohibitive computational and memory costs for computing $\gU$ in \cref{section4:eq:rotated2}, \solution{} adopts two structural assumptions on the Hessian: (i) block-diagonality and (ii) Kronecker factorization.
The former assumes a block-diagonal Hessian, enabling matrix-wise rotation rather than rotating the entire parameter space at once.
This assumption is empirically supported in Transformer \citep{zhang2024transformers, zhang2024adam, abreu2025potential}.
The latter allows the rotation matrix $\mathcal{U}_W \in \R^{mn \times mn}$ to be factorized into $U \in \R^{m \times m}$ and $V \in \R^{n \times n}$, making the rotation computationally tractable for large models \citep{martens2015optimizing, gupta2018shampoo, vyas2024soap}.
Together with infrequent basis updates, \solution{} remains efficient at scale.

\subsection{Eigenbasis Estimation}
\label{subsec:solution:Hessian_approximation}
We now introduce efficient \estimation{} strategies which theoretically induce basis alignment under the assumptions of \cref{subsec:solution}.
These strategies are categorized along two design axes: approximation source ($\mathcal{S}$) and rotation geometry ($\mathcal{G}$).
This taxonomy provides a unified design space for balancing estimation fidelity against memory efficiency, depending on the demands and constraints of the training environment.
A detailed memory overhead analysis of each strategy is provided in \cref{app:memory_overhead}.

\begin{algorithm}[t]
\caption{\texttt{Eigenbasis-Estimation}}
\begin{algorithmic}[1]\label{algo:basis_rotation_type}
\REQUIRE Approximation source $\mathcal{S}$, Rotation geometry $\mathcal{G}$, $G_t, M_t, U, V$
\IF{$\mathcal{S} = 2^{\text{nd}}$}
\STATE $L \gets \beta_2 L + (1-\beta_2) G_tG_t^\top$ 
\STATE $U \gets \text{Power}(L, U)$
    \IF{$\mathcal{G} = \text{Bilateral}$}
    \STATE $R \gets \beta_2 R + (1-\beta_2) G_t^\top G_t$
    \STATE $V \gets \text{Power}(R, V)$
    \ENDIF
    \IF{$\mathcal{G} = \text{Unilateral}$}
    \STATE $V \gets I$
    \ENDIF
\ENDIF

\IF{$\mathcal{S} = 1^{\text{st}}$}
\STATE $U \gets \text{Power} (M_tM_t^\top, U)$
    \IF{$\mathcal{G} = \text{Bilateral}$}
    \STATE $V \gets \text{Power} (M_t^\top M_t, V)$
    \ENDIF
    \IF{$\mathcal{G} = \text{Unilateral}$}
    \STATE $V \gets I$
    \ENDIF
\ENDIF
\end{algorithmic}
\end{algorithm}

\begin{figure*}[!t]
    \centering
    \begin{subfigure}{0.19\linewidth}
        \includegraphics[width=\linewidth]{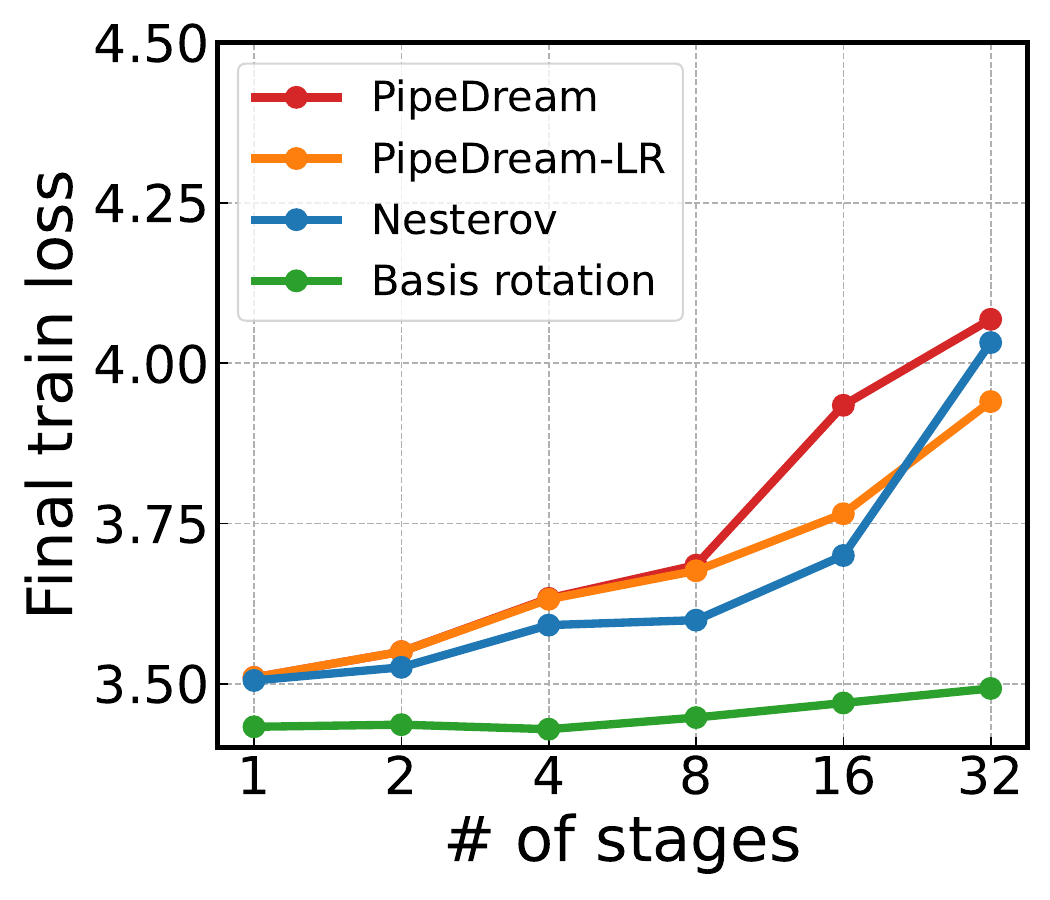}
        \caption{$P$ vs. final loss}
        \label{fig:result-main-numstages}
    \end{subfigure}
    \begin{subfigure}{0.19\linewidth}
        \includegraphics[width=\linewidth]{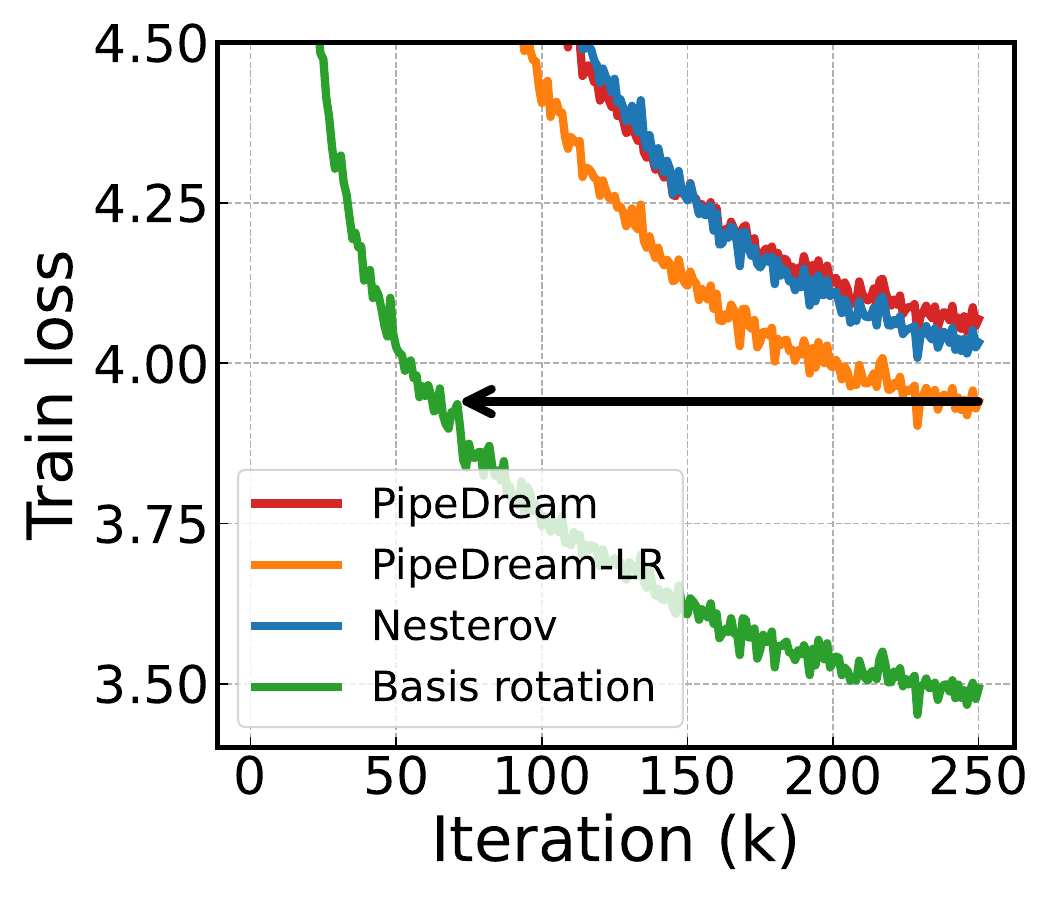}
        \caption{Training curves}
        \label{fig:result-main-stage32}
    \end{subfigure}
    \hfill
    \begin{subfigure}{0.19\linewidth}
        \includegraphics[width=\linewidth]{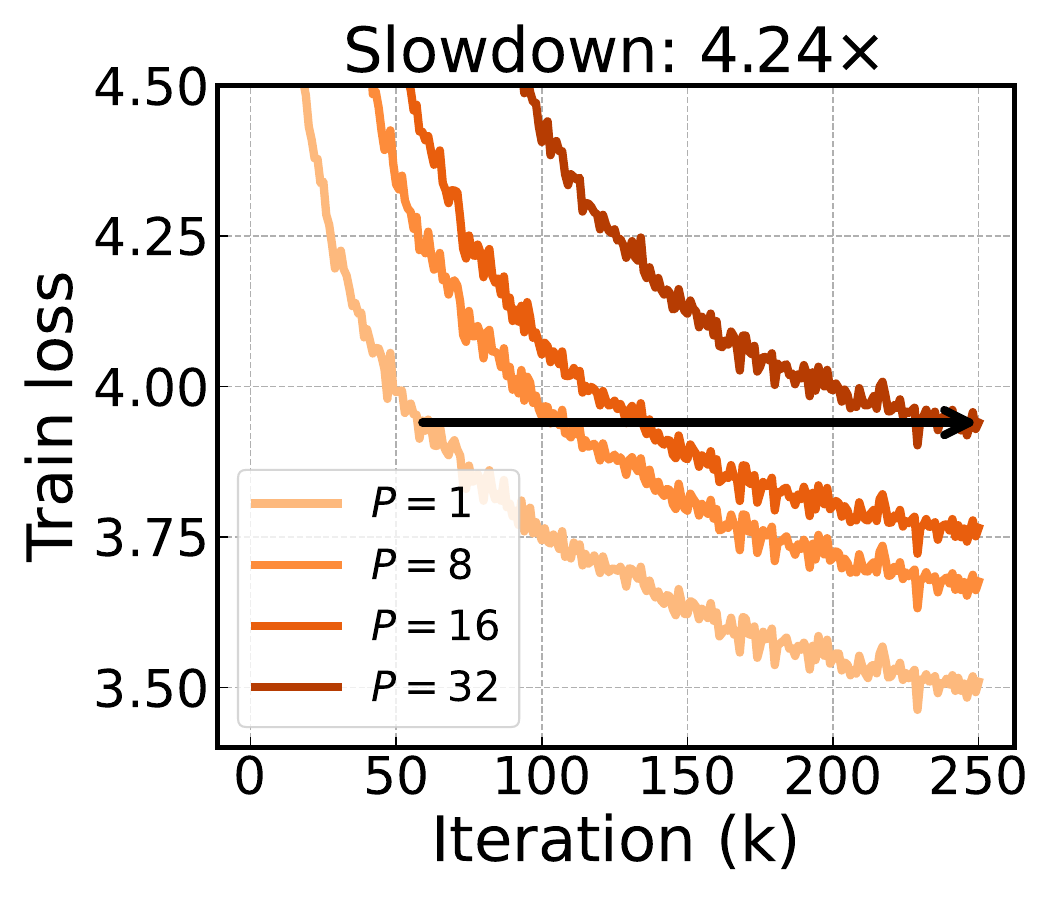}
        \caption{PipeDream-LR}
        \label{fig:result-main-slowdown-lradamw}
    \end{subfigure}
    \begin{subfigure}{0.19\linewidth}
        \includegraphics[width=\linewidth]{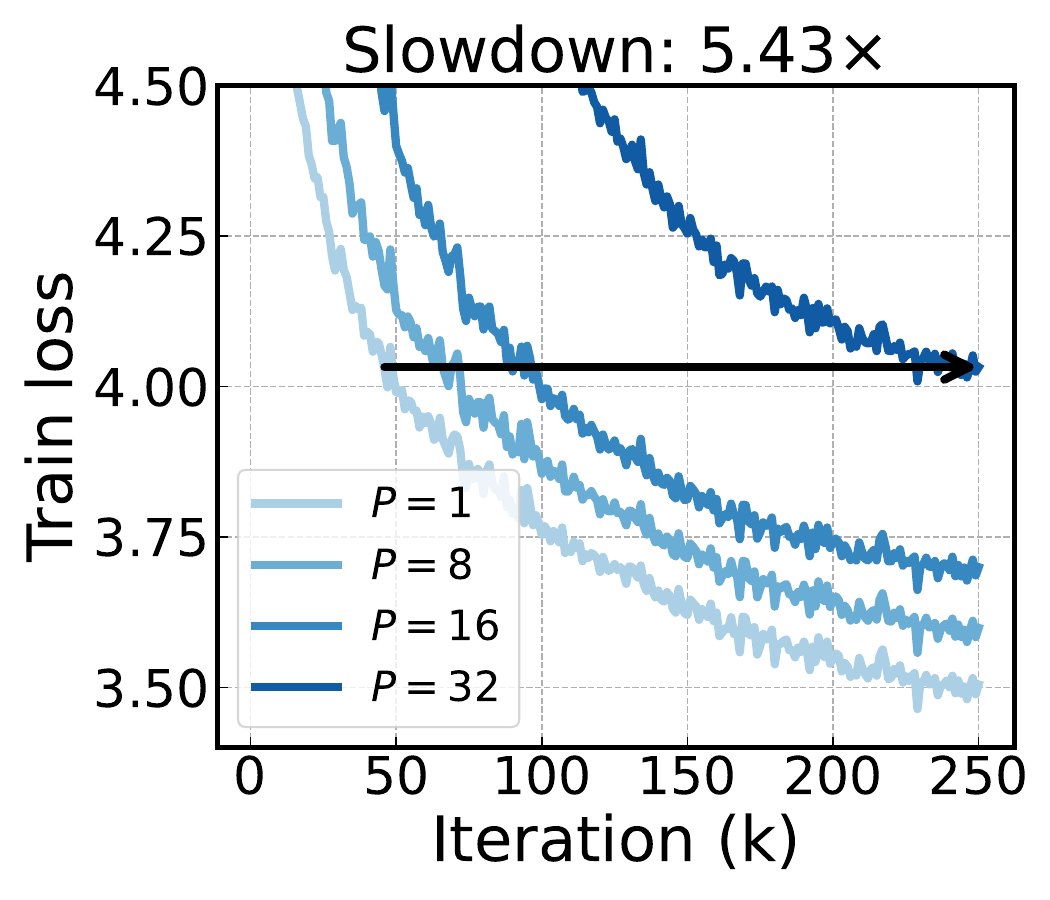}
        \caption{Nesterov}
        \label{fig:result-main-slowdown-nadamw}
    \end{subfigure}
    \begin{subfigure}{0.19\linewidth}
        \includegraphics[width=\linewidth]{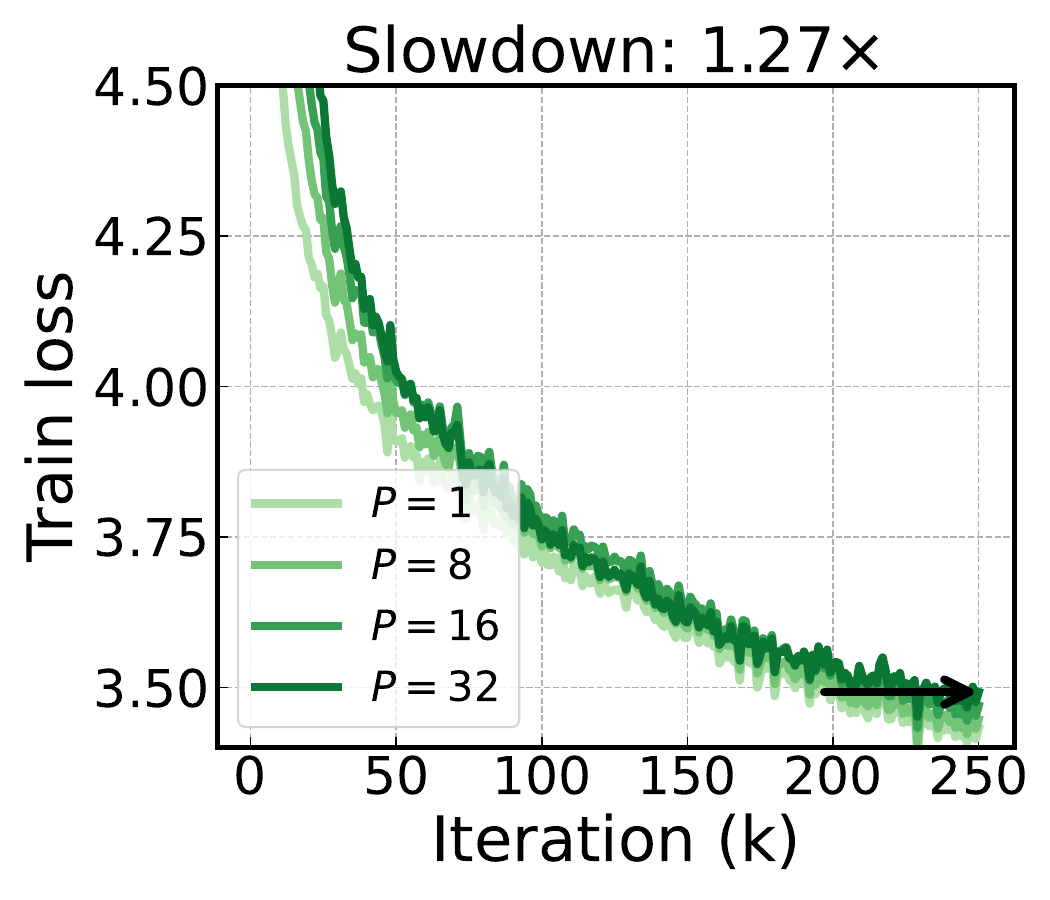}
        \caption{\Ourmethod{}}
        \label{fig:result-main-slowdown-rotatedadam}
    \end{subfigure}
    \caption{
    Performance of different methods when increasing the number of stages $P$ for the same model (with $95$M parameters).
    (a)
    \Ourmethod{} maintains stable performance while baselines suffer significantly from delay.
    (b)
    \Ourmethod{} shows much faster convergence under large delay compared to baselines (for $P=32$).
    (c-e) 
    \Ourmethod{} reduces slowdown, \ie, the iteration ratio required to reach target loss for $P=32$ relative to $P=1$, by a large margin compared to baselines.
    More results including other baselines and different settings are presented in \cref{fig:result-appendix-optmain,fig:result-appendix-slowdown,fig:result-appendix-val,fig:result-appendix-dc,tab:result-preconditioned-slowdown,fig:result-appendix-moe} of \cref{app:additional-exp-results}.
    }
    \label{fig:result-main}
\end{figure*}

The first axis $\mathcal{S}$ determines the statistical source used to approximate the Hessian.
The high-fidelity version (\covariance{}) utilizes the second moments $L = \E [GG^\top]$ and $R = \E[G^\top G]$ to approximate the empirical Fisher, where $G$ is the gradient matrix.
This quantity has frequently been used as a computationally efficient proxy for the Hessian \citep{roux2007topmoumoute, singh2020woodfisher,frantar2021m}.
In contrast, $\gS=$ \mean{} variant employs first-order moments ($\E[GG^\top] \approx \E[G]\E[G]^\top$), eliminating dedicated storage for $L$ and $R$ by leveraging the existing momentum buffer.

The second axis $\mathcal{G}$ determines the rotation geometry used to diagonalize the approximating matrix.
The \twosided{} version applies a two-sided rotation to capture the full Kronecker-factored structure, while the \onesided{} variant applies a one-sided rotation to the smaller dimension of the gradient.
This reduces overhead by storing eigenvectors of either $L$ or $R$, which is particularly beneficial for rectangular matrices.

These strategies can theoretically reduce the $(1,1)$-norm of the Hessian, which is a proxy for the degree of misalignment as introduced in \cref{sec:analysis-convergence}.
Define the rotated Hessians $H_{U}$ and $H_{U,V}$ corresponding to the transformations $\widetilde{W} = U^\top W$ (\onesided{}) and $\widetilde{W} = U^\top W V$ (\twosided{}) respectively.
Then, the following theorem holds.

\begin{restatable}{theorem}{covarianceHessianapproximation}\label{thm:covariance_Hessian_approximation}
Let $U$ and $V$ be the matrices whose columns are eigenvectors of $\mathbb{E}[G G^\top]$ and $\mathbb{E}[G^\top G]$ respectively.
If Hessian admits a Kronecker-factorized empirical Fisher, the following inequalities hold:
\begin{align*}
    ||H_{U,V}||_{(1,1)} \le ||H_{U}||_{(1,1)} \le ||H||_{(1,1)}.
\end{align*}
Moreover, $\|H_{U,V}\|_{(1,1)}$ achieves the global minimum over all rotations.
\end{restatable}
The theorem shows that the \twosided{} rotation minimizes basis misalignment while both versions provide better alignment than the original space.
The proof and further result for $\gS=$\mean{} are provided in \cref{app:approximation-comparison-proof}.

The full procedure covering both axes is presented in \cref{algo:basis_rotation_type}.
We use a single step of power iteration followed by QR decomposition to efficiently compute eigenvectors \cite{wang20244}.
While our \estimation{} framework shares similarities with recent optimizers—such as SOAP (\citet{vyas2024soap}; $\mathcal{S}=$ \covariance{}, $\mathcal{G}=$ \twosided{}) and the full-rank version of GaLore (\citet{zhao2024galore}; $\mathcal{S}=$ \mean{}, $\mathcal{G}=$ \onesided{})—we distinguish our approach by unifying these strategies into a single framework rather than adopting individual optimizers.
This controlled setup allows us to isolate the effects of Hessian geometry from other implementation variables.
See \cref{app:connection_recent_optimizers} for a detailed comparison.

%% file: tex/5_experiments.tex
\section{Experiments}
\label{sec:experiments}

\subsection{Experimental Setup}

We evaluate the performance of \solution{} on the language modeling task with a standard decoder-only Transformer \citep{vaswani2017attention} ranging from $95$M to $1$B parameters.
The models are trained on $1$B tokens randomly selected from OpenWebText \citep{Gokaslan2019OpenWeb}.
We evaluate \ourmethod{} against three baselines which are primary strategies for asynchronous pipeline parallelism in the literature:
(1) PipeDream \citep{narayanan2019pipedream} which serves as the vanilla baseline by not explicitly addressing delay,
(2) PipeDream-LR \citep{yang2021pipemare} which schedules the stage-wise learning rate depending on the level of delay, and
(3) Nesterov \citep{ajanthan2025nesterov} which incorporates Nesterov momentum to address delayed gradients.
We use $\gS=$ \covariance{}, $\gG=$ \twosided{} strategy for \estimation{} and set the basis update frequency to $10$ unless specified otherwise.
Finally, we employ weight stashing \citep{narayanan2019pipedream} across all methods to ensure consistency between the weights used in the forward and backward passes.
Experimental details can be found in \cref{app:sub:LLM-exp-details}.

\subsection{Main Results}

\label{sec:exp-main}

We first evaluate the robustness of \ourmethod{} against delay by increasing the number of pipeline stages $P$ for the same model.
The results summarized in \cref{fig:result-main-numstages} demonstrate that \ourmethod{} consistently outperforms baselines with the performance gap widening significantly as $P$ increases.
While prior work has demonstrated the benefits of optimization in the Hessian eigenbasis under the standard zero-delay regime \citep{vyas2024soap,eschenhagen2025purifying}, our results indicate that these advantages are substantially amplified in the presence of large delays.
Notably, \cref{fig:result-main-stage32} demonstrates that at $P=32$, \ourmethod{} achieves the same training loss with $71.6\%$ fewer iterations than the best-performing baseline.
To quantify robustness to delay, the right three panels of \cref{fig:result-main} report the slowdown of each method, defined as the ratio of the number of iterations required to reach a fixed loss threshold at $P = 32$ relative to $P = 1$.
As shown in \cref{fig:result-main-slowdown-lradamw,fig:result-main-slowdown-nadamw}, even the best-performing baseline shows a large slowdown of $4.24 \times$.
In contrast, \ourmethod{} in \cref{fig:result-main-slowdown-rotatedadam} exhibits substantially improved robustness to delay with only $1.27 \times$ slowdown.

Next, we investigate the scalability of \ourmethod{} by jointly increasing the number of Transformer blocks and the number of pipeline stages $P$, assigning one block to each stage.
This setup reflects the practical necessity of using deeper pipelines to accommodate the memory footprint of larger models.
\cref{fig:result-main-blockscaling-lradamw,fig:result-main-blockscaling-nadamw} reveal a critical failure in baseline methods;
increasing model size leads to higher training loss, directly contradicting standard scaling laws \citep{kaplan2020scaling,hoffmann2022training}.
In contrast, \ourmethod{} in \cref{fig:result-main-blockscaling-rotatedadam-twosided-cov} successfully restores scalability, achieving performance improvements as model size grows.

Finally, we evaluate \ourmethod{} on a 1B parameter model to verify its efficacy at a larger scale.
Specifically, we set $P=24$ and increase the number of embedding dimensions to scale the model.
\cref{fig:result-main-1b} demonstrates the performance gap between \ourmethod{} and the baselines widens for larger models.
For 1B model, \ourmethod{} achieves the same training loss with $76.8\%$ fewer iterations than the best-performing baseline, surpassing the $62.4\%$ reduction in smaller models.
This advantage further increases at 3B scale, where \ourmethod{} reduces iterations by $81.7\%$ (see \cref{fig:result-appendix-3b} in \cref{app:additional-exp-results}).
The results show that the advantages of our approach scale positively with model size.

Overall, our results underscore that the performance degradation in asynchronous training is not an unavoidable cost of staleness but rather a byproduct of optimizer sensitivity to delay.
\Ourmethod{} effectively eliminates this bottleneck, unlocking a new regime for bubble-free execution at scale.

\begin{figure}[t]
    \centering
    \begin{subfigure}{0.32\linewidth}
        \includegraphics[width=\linewidth]{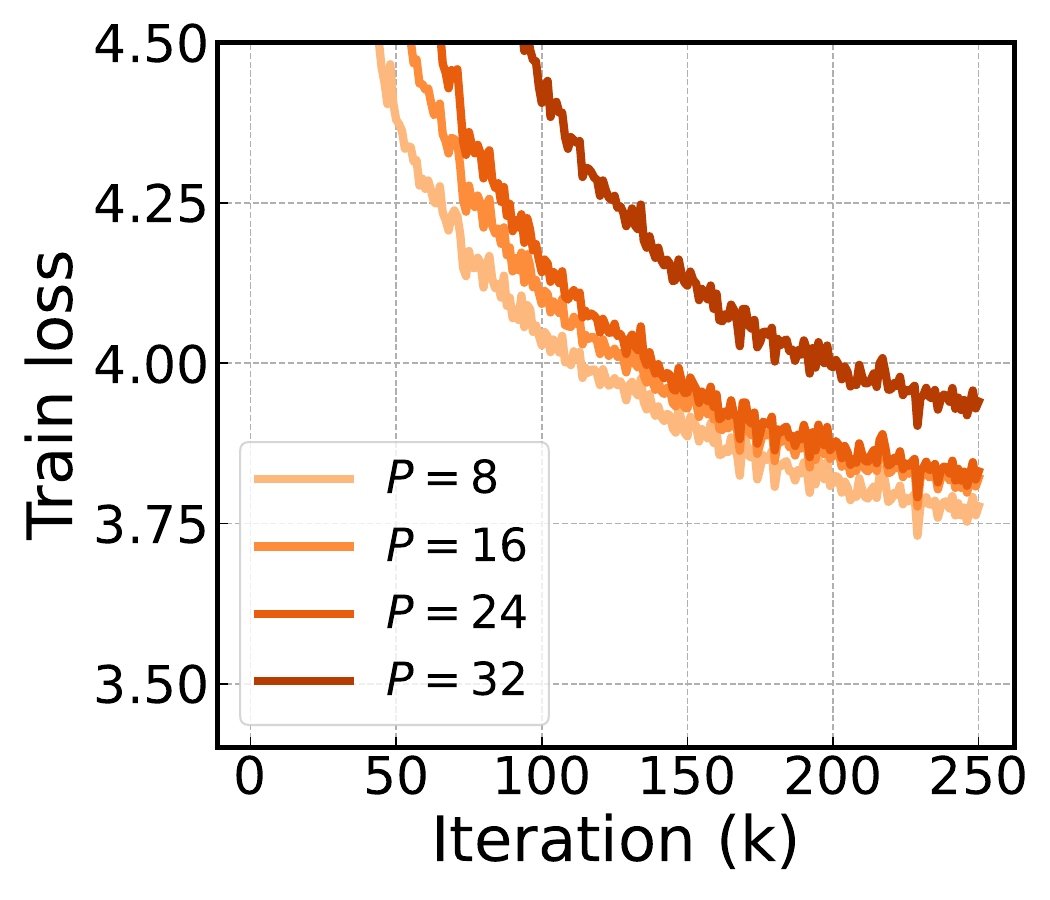}
        \caption{PipeDream-LR}
        \label{fig:result-main-blockscaling-lradamw}
    \end{subfigure}
    \begin{subfigure}{0.32\linewidth}
        \includegraphics[width=\linewidth]{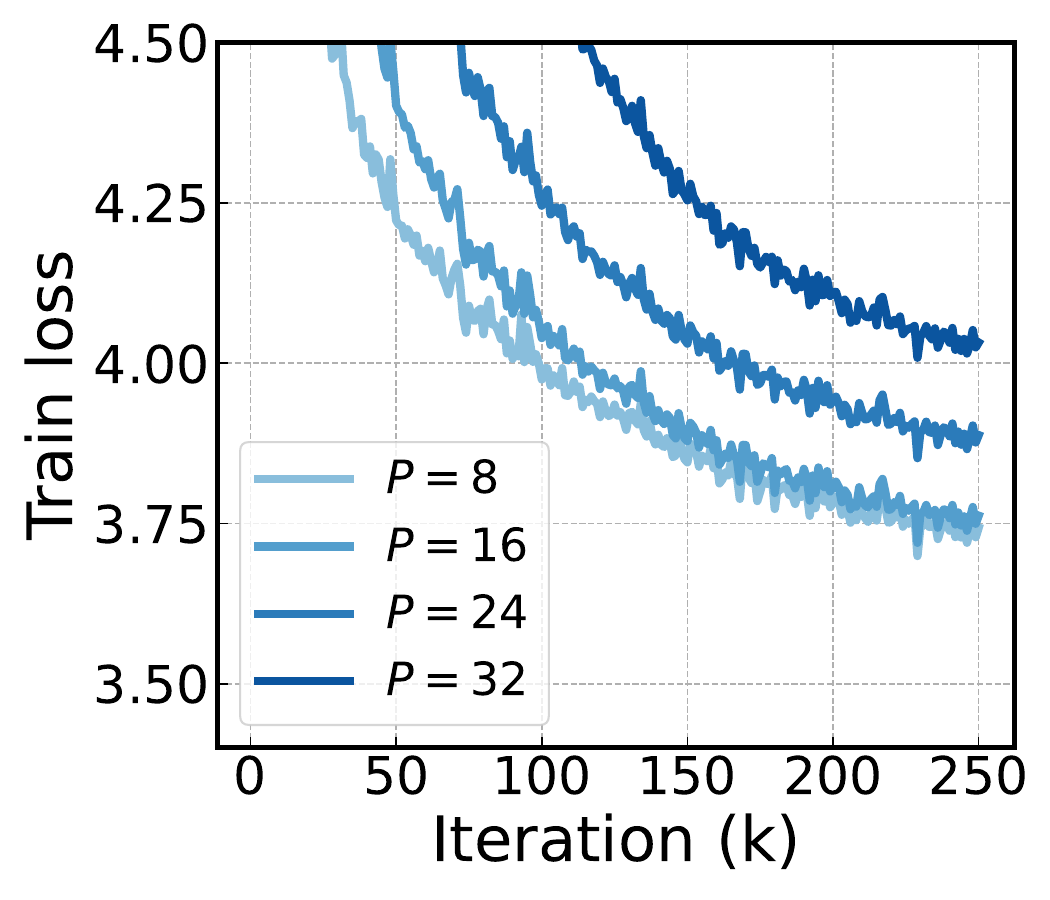}
        \caption{Nesterov}
        \label{fig:result-main-blockscaling-nadamw}
    \end{subfigure}
    \begin{subfigure}{0.32\linewidth}
        \includegraphics[width=\linewidth]{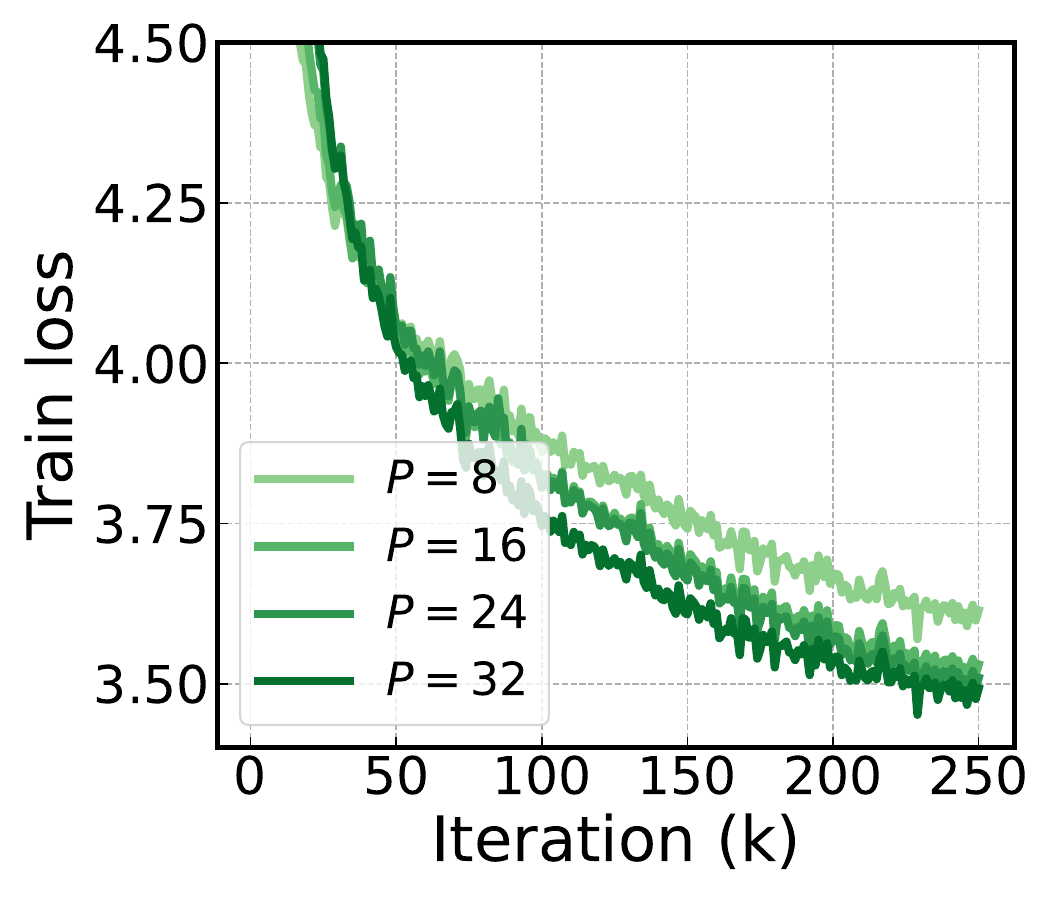}
        \caption{\Ourmethod{}}
        \label{fig:result-main-blockscaling-rotatedadam-twosided-cov}
    \end{subfigure}    
    \caption{
    Performance of different methods when increasing $P$ by scaling the number of blocks.
    While scaling the model leads to increased loss for baselines (a,b), it decreases the loss for \ourmethod{} (c).
    More results are presented in \cref{fig:result-appendix-blockscaling} of \cref{app:additional-exp-results}.
    }
    \label{fig:result-main-blockscaling}
\end{figure}

\begin{figure}[t]
    \centering
    \begin{subfigure}{0.4\linewidth}
        \includegraphics[width=\linewidth]{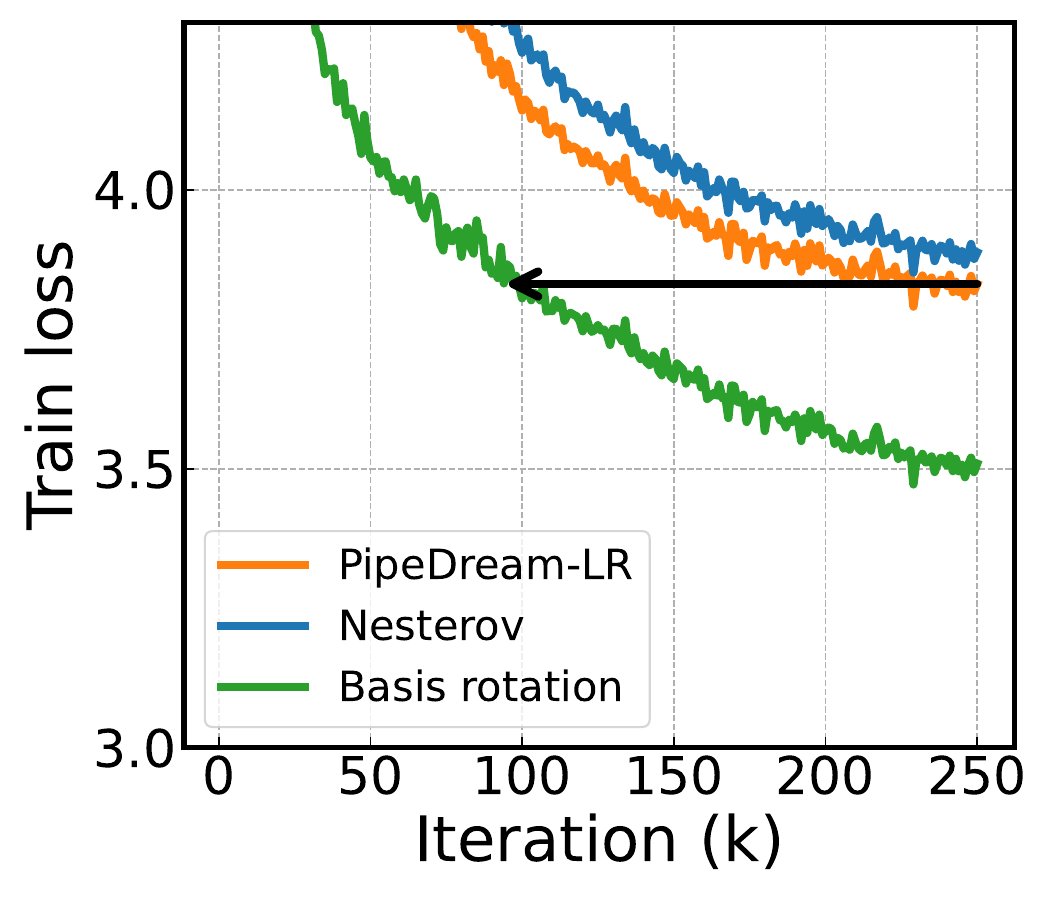}
        \caption{$0.1$B model}
        \label{fig:result-main-24layers-small}
    \end{subfigure}
    \hspace{2em}
    \begin{subfigure}{0.4\linewidth}
        \includegraphics[width=\linewidth]{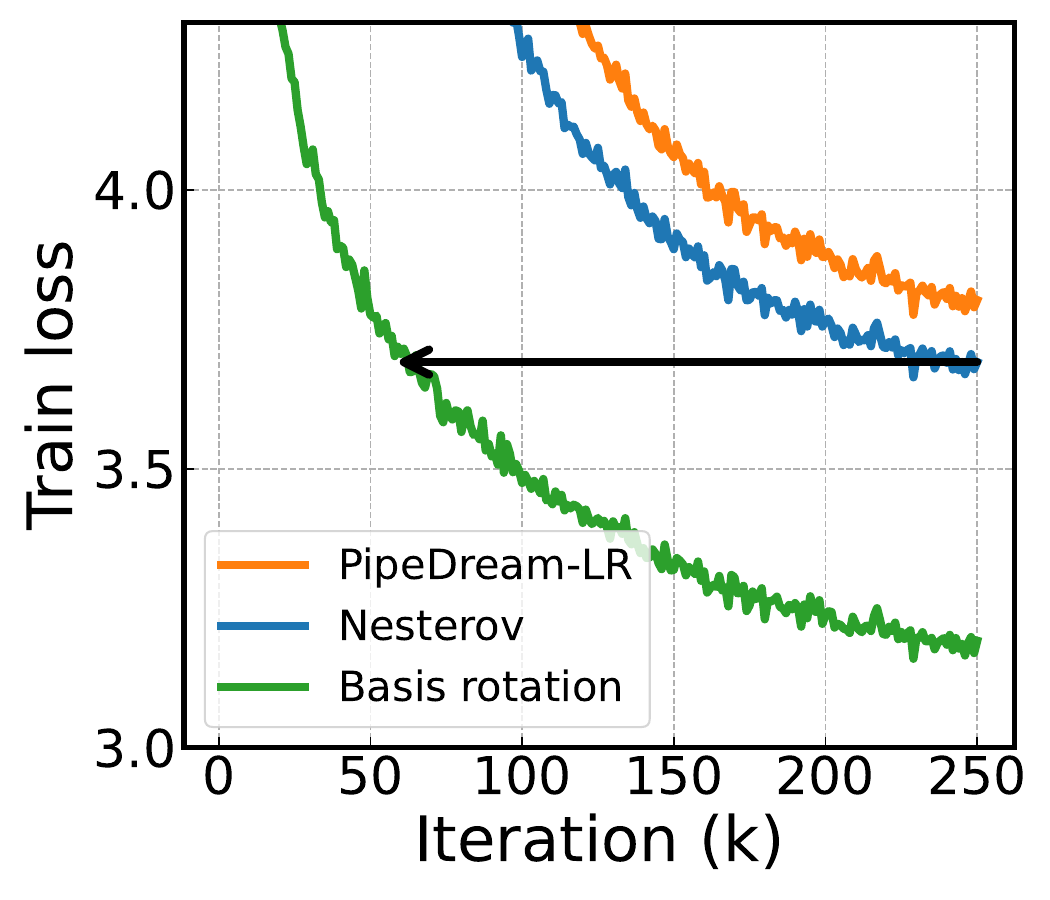}
        \caption{$1$B model}
        \label{fig:result-main-24layers-large}
    \end{subfigure}
    \caption{
    Performance of different methods at $P=24$ for (a) $0.1$B and (b) $1$B models.
    The gap between \ourmethod{} and baselines increases further for larger models.
    }
    \label{fig:result-main-1b}
    \vspace{-1em}
\end{figure}

\subsection{More Results}

\label{sec:exp-more}

\paragraph{Basis Approximation Sensitivity}

To evaluate how the precision of the eigenbasis approximation affects robustness to delay, we compare different \estimation{} strategies in \cref{fig:result-rotation}.
As seen in \cref{tab:result-rotation-slowdown}, \ourmethod{} exhibits higher robustness with a closer approximation to the Hessian eigenbasis.
Specifically, \covariance{} results in a smaller slowdown than \mean{} while \twosided{} outperforms \onesided{}.
The results confirm that \misalign{} is the primary factor of the performance degradation under delay and a more accurate \solution{} effectively mitigates the issue.
We also note that even the least accurate strategy (\mean{} / \onesided{}) outperforms the best-performing baseline by a large margin (see \cref{fig:reuslt-rotation-stage32}).
The results suggest that our approach is highly effective even in resource-constrained settings where more sophisticated estimation strategies may be computationally expensive.

\paragraph{Computational Efficiency}

To address potential concerns regarding the computational cost of the basis update, we compare its wall-clock efficiency against baselines.
\cref{fig:result-main-efficiency-runtime} demonstrates that \ourmethod{} remains superior in terms of GPU hours, reaching the same training loss with $54.3\%$ less amount of time than the most competitive baselines.
Furthermore, this overhead can be reduced by adjusting the basis update frequency with a marginal performance degradation.
As shown in \cref{fig:result-main-efficiency-frequency}, \ourmethod{} remains significantly more efficient than the baselines even at an update frequency of $100$ iterations.

\paragraph{Stage-aware \OURMETHOD{}}
Beyond uniform update frequencies, we investigate a stage-aware strategy (detailed in \cref{app:additional-exp-results}) that allocates the computational budget for subspace updates in proportion to the gradient delay at each stage. 
\cref{fig:result-main-layerwise_freqeuncy} shows this achieves a 29.2\% speedup in convergence compared to the uniform-frequency baseline, while preserving the same total computational budget. 
This empirical gain aligns with the theoretical insight from \cref{eq:effective-delay}: the effective delay $\tau'$ is dominated by the misalignment magnitude ($\Smn_i^2$) at the earliest pipeline stages, where the per-stage delay $(K-k)$ is most severe.
Suppressing it thus reduces $\tau'$ and accelerates convergence.
An ablation with inversely-ordered frequency allocation in \cref{fig:reverse-stage-wise-frequency} of \cref{app:additional-exp-results} confirms this: the reversed strategy degrades performance relative to the uniform baseline.

\begin{figure}[!t]
    \centering
    
    \refstepcounter{figure}
    \setcounter{table}{\value{figure}}
    \addtocounter{table}{-1}
    \begin{minipage}[b]{0.55\linewidth}
        \centering
        \captionsetup{type=table}
\begin{subtable}[b]{\linewidth}
    \centering
    \resizebox{\linewidth}{!}{
\begin{tabular}{llcc}
    \toprule
    \multirow{2}{*}{Method} & \multicolumn{2}{c}{\texttt{Estimation}} & \multirow{2}{*}{Slowdown} \\ 
    & $\mathcal{S}$ & $\mathcal{G}$ & \\ 
    \midrule
    PipeDream-LR & --- & --- & $4.24 \times$ \\ 
    \midrule
    \multirow{4}{*}{Basis rotation} & \multirow{2}{*}{$1^{\text{st}}$} & Uni & $2.55 \times$ \\ 
                                     &  & Bi  & $1.77 \times$ \\ 
                                     &\multirow{2}{*}{$2^{\text{nd}}$} & Uni & $1.66 \times$ \\ 
                                     & & Bi  & $1.27 \times$ \\ 
    \bottomrule
\end{tabular}
    }
    \vspace{12.5pt} 
    \caption{Convergence slowdown}
    \vspace{-6pt}
    \label{tab:result-rotation-slowdown}
\end{subtable}
    \end{minipage}
    \setcounter{subfigure}{1}
    \begin{subfigure}[b]{0.4\linewidth}
        \centering
        \includegraphics[width=\linewidth]{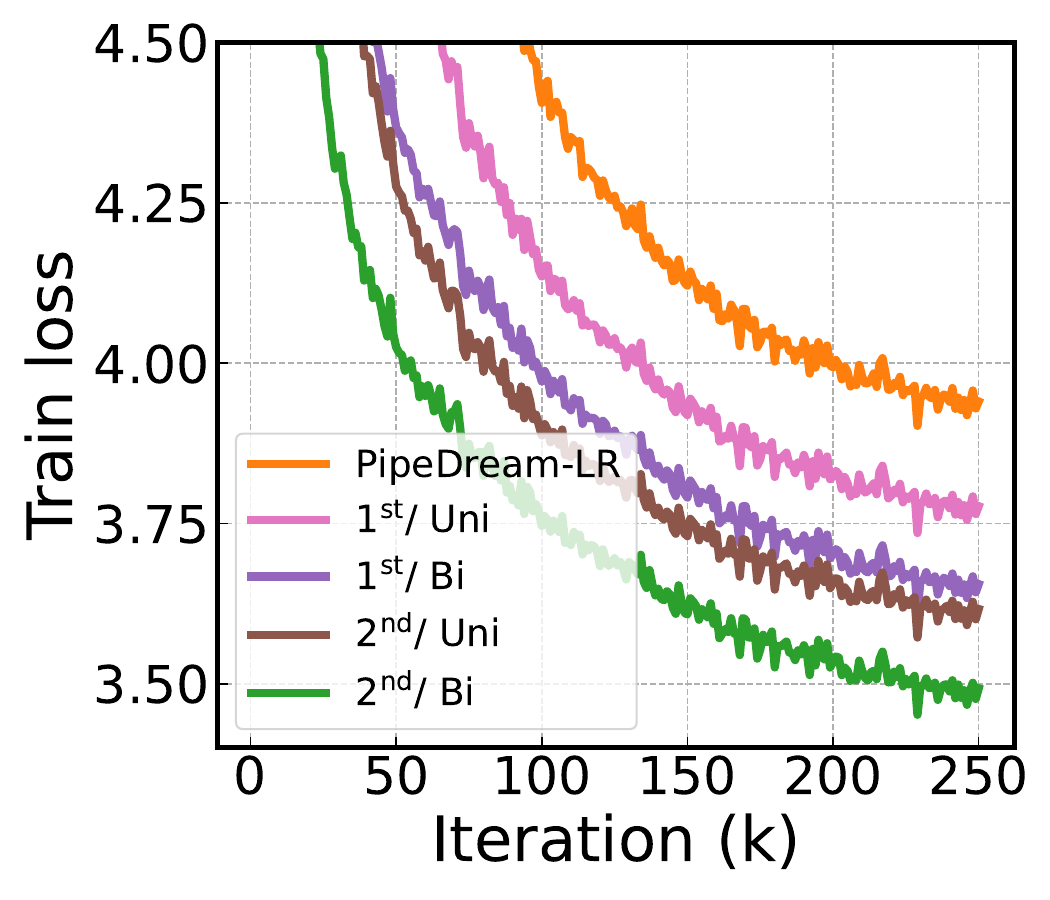}
        \caption{Training loss}
        \label{fig:reuslt-rotation-stage32}
    \end{subfigure}
    \caption{
    Comparison of different \estimation{} strategies.
    (a) High-fidelity estimation leads to a smaller slowdown under delay.
    (b) High-fidelity estimation achieves faster convergence for $P=32$.
    Even the least accurate estimation (\mean{}/\onesided{}) outperforms the best-performing baseline.
    More results are presented in \cref{fig:result-appendix-optrotated} of \cref{app:additional-exp-results}.
    }
    \label{fig:result-rotation}
\end{figure}

\begin{figure}[!t]
    \centering
    \begin{subfigure}{0.32\linewidth}
        \includegraphics[width=\linewidth]{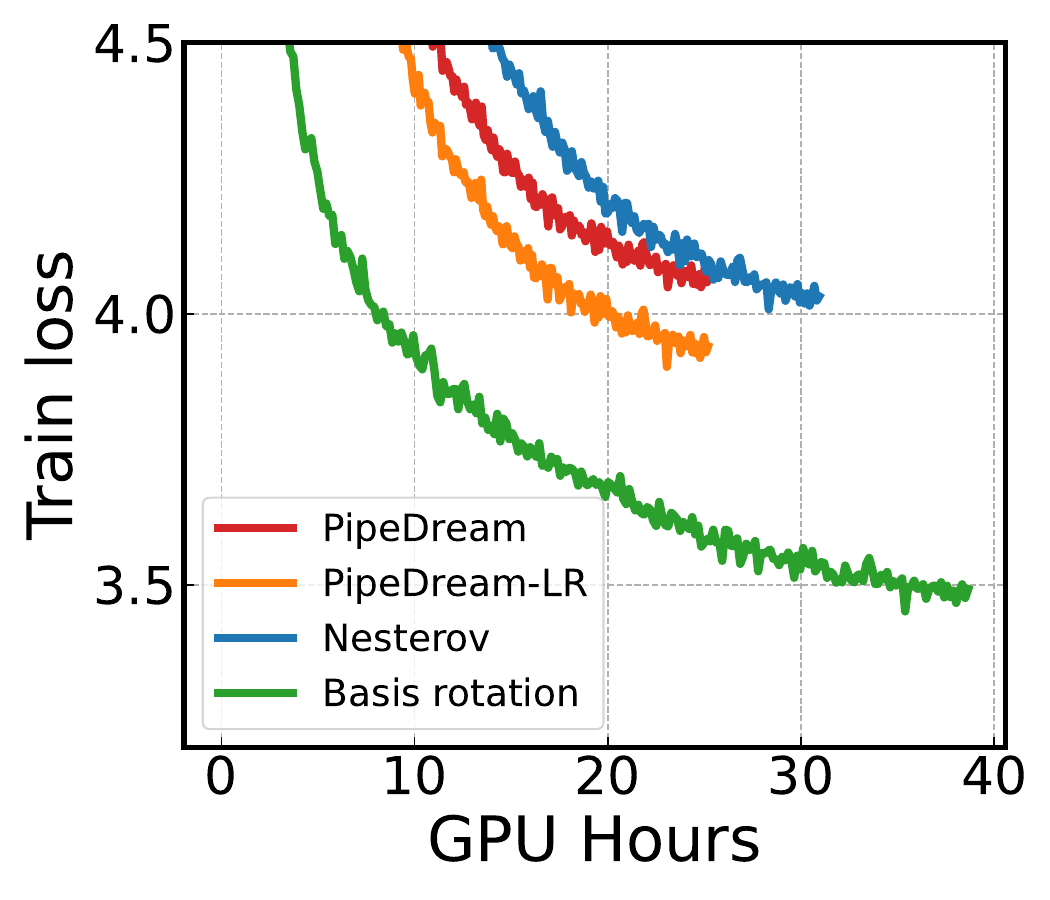}
        \caption{GPU Hours}
        \label{fig:result-main-efficiency-runtime}
    \end{subfigure}
    \begin{subfigure}{0.32\linewidth}
        \includegraphics[width=\linewidth]{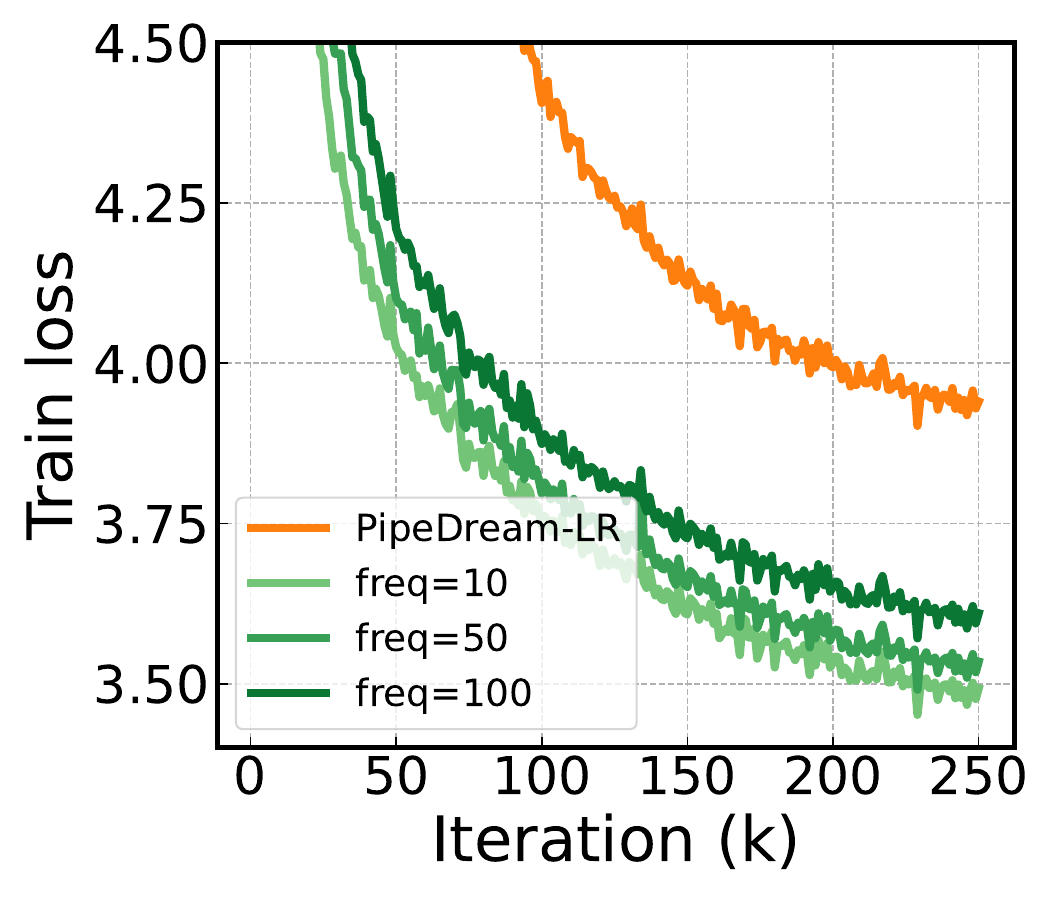}
        \caption{Update frequency}
        \label{fig:result-main-efficiency-frequency}
    \end{subfigure}    
    \begin{subfigure}{0.32\linewidth}
        \includegraphics[width=\linewidth]{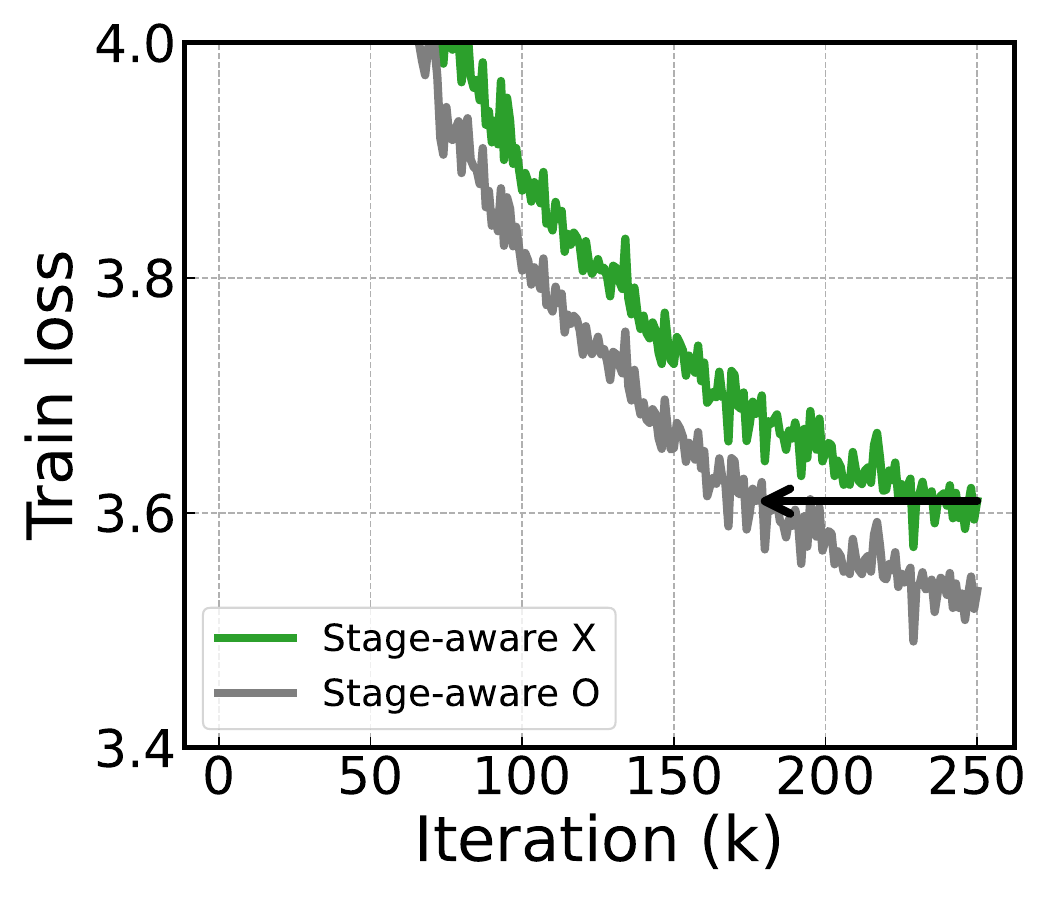}
        \caption{Stage-aware BR}
        \label{fig:result-main-layerwise_freqeuncy}
    \end{subfigure}   
    \caption{
    Efficiency of \ourmethod{}.
    (a) \Ourmethod{} outperforms other methods by a large margin in terms of GPU hours.
    (b) \Ourmethod{} shows little performance degradation with infrequent basis updates.
    (c) Stage-aware \ourmethod{} achieves superior convergence compared to the uniform update frequency under the same total computational budget.
    }
    \label{fig:result-main-efficiency}
\end{figure}

\begin{figure}[t]
    \centering
    \begin{subfigure}{0.32\linewidth}
        \includegraphics[width=\linewidth]{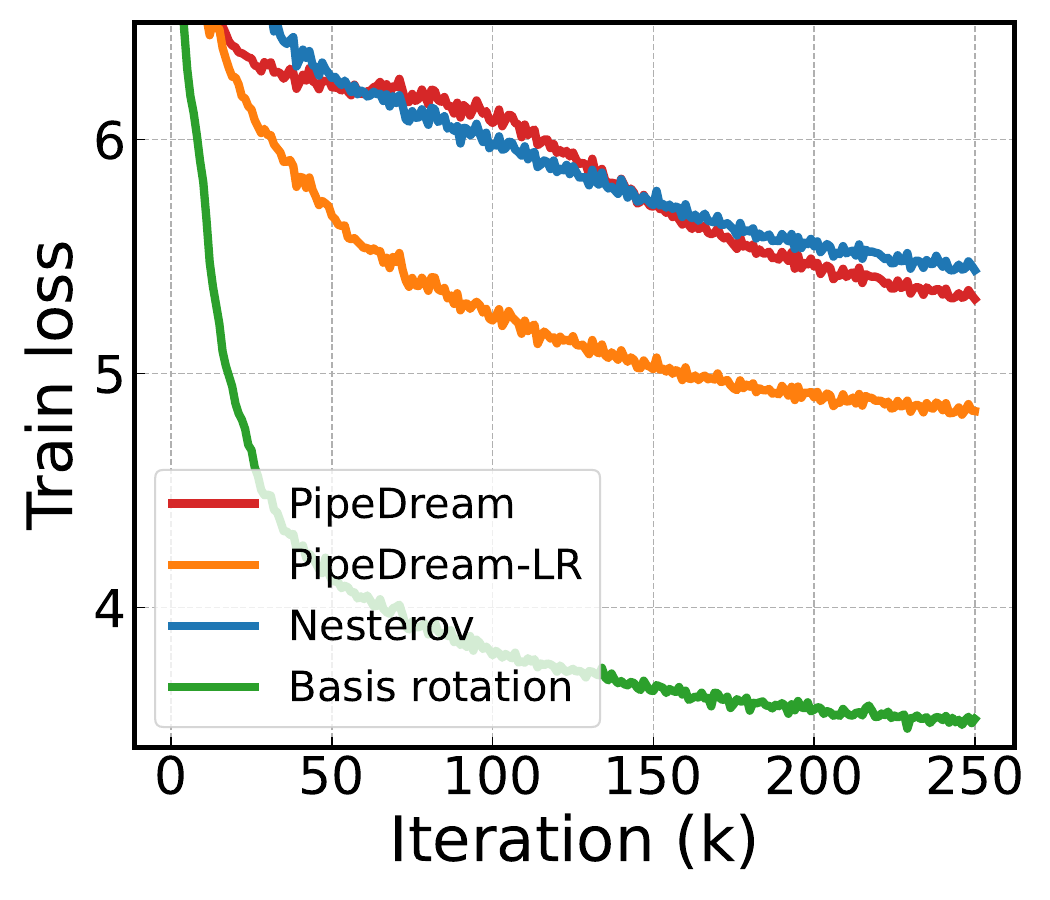}
        \caption{W/o stashing}
        \label{fig:result-main-wows}
    \end{subfigure}    
    \begin{subfigure}{0.32\linewidth}
        \includegraphics[width=\linewidth]{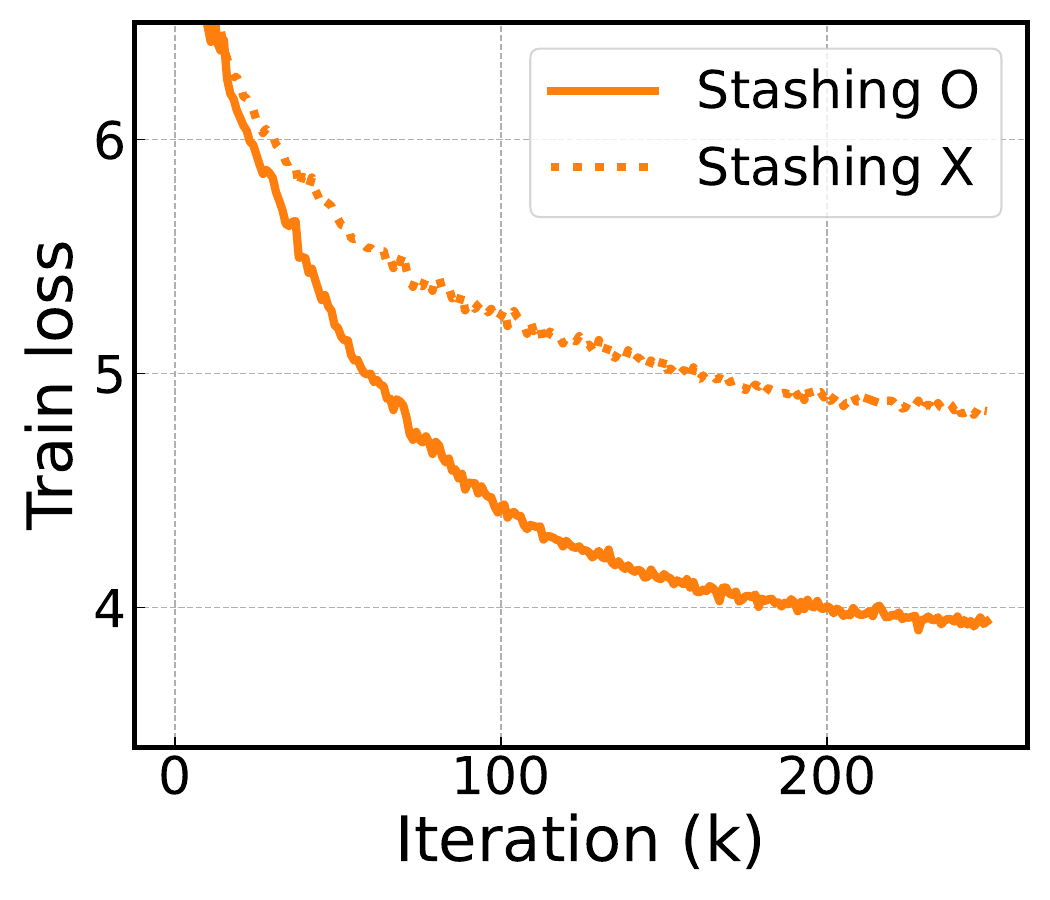}
        \caption{PipeDream-LR}
        \label{fig:result-main-weightstashing-lradamw}
    \end{subfigure}    
    \begin{subfigure}{0.32\linewidth}
        \includegraphics[width=\linewidth]{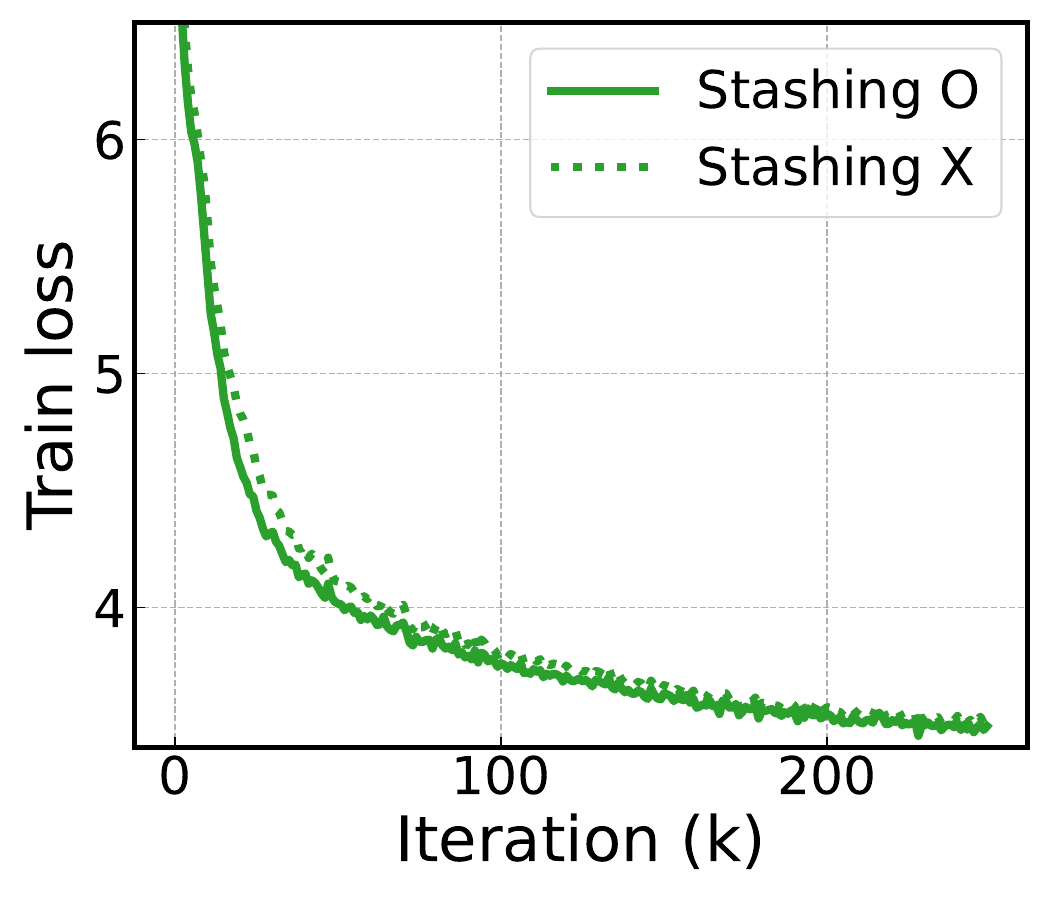}
        \caption{\Ourmethod{}}
        \label{fig:result-main-weightstashing-basisrotation}
    \end{subfigure}    
    \caption{
    Performance of different methods without weight stashing for $P=32$.
    (a) \Ourmethod{} outperforms baselines with a large margin.
    (b,c) While baselines exhibit severe degradation without weight stashing, \ourmethod{} remains robust.
    }
    \label{fig:result-weightstashing}
\end{figure}

\paragraph{Robustness without Weight Stashing}
The weight stashing employed throughout our main experiments ensures correct backpropagation by storing the weights used for the forward pass.
However, it incurs memory overhead that scales linearly with the number of pipeline stages, making it prohibitive in memory-constrained environments.
To evaluate whether the benefits of \ourmethod{} persist in such resource-limited settings, we omit weight stashing, resulting in incorrect gradient computation \citep{gaunt2017ampnet,huo2018decoupled}.
\cref{fig:result-main-wows} shows that \ourmethod{} maintains a significant performance advantage over baselines: it remains robust without weight stashing (\cref{fig:result-main-weightstashing-basisrotation}), whereas the best-performing baseline exhibit severe degradation (\cref{fig:result-main-weightstashing-lradamw}).
Additional experiments with PipeMare-style weight prediction \citep{yang2021pipemare}---which partially addresses the weight discrepancy problem---likewise show \ourmethod{} outperforming baselines (\cref{fig:result-appendix-pipemare} in \cref{app:additional-exp-results}), confirming that our approach is significantly more resilient to incorrect backpropagation.

\begin{figure}[t]
    \centering
    \begin{subfigure}{0.46\linewidth}
        \includegraphics[width=\linewidth]{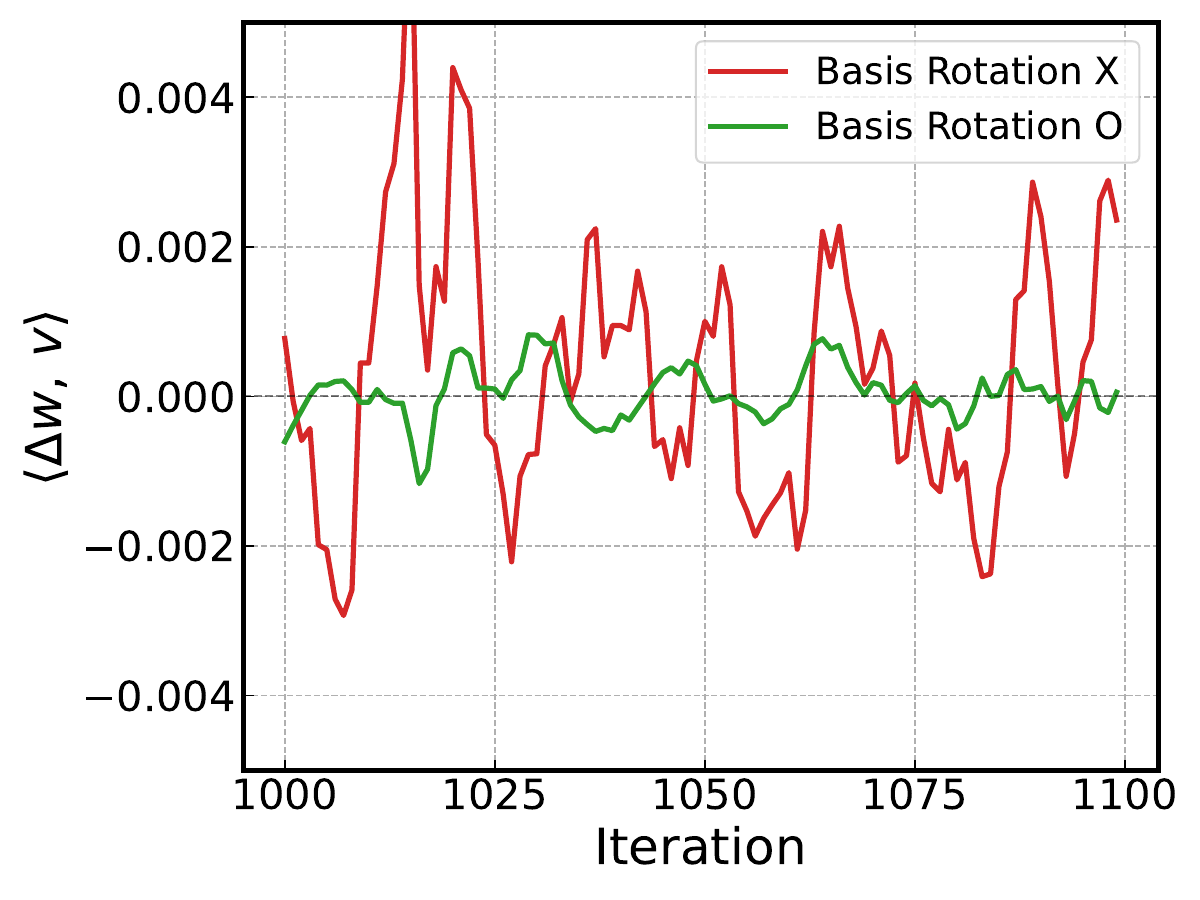}
        \caption{Dominant eigenvector}
        \label{fig:result-main-oscillation-dominant}
    \end{subfigure}    
    \begin{subfigure}{0.46\linewidth}
        \includegraphics[width=\linewidth]{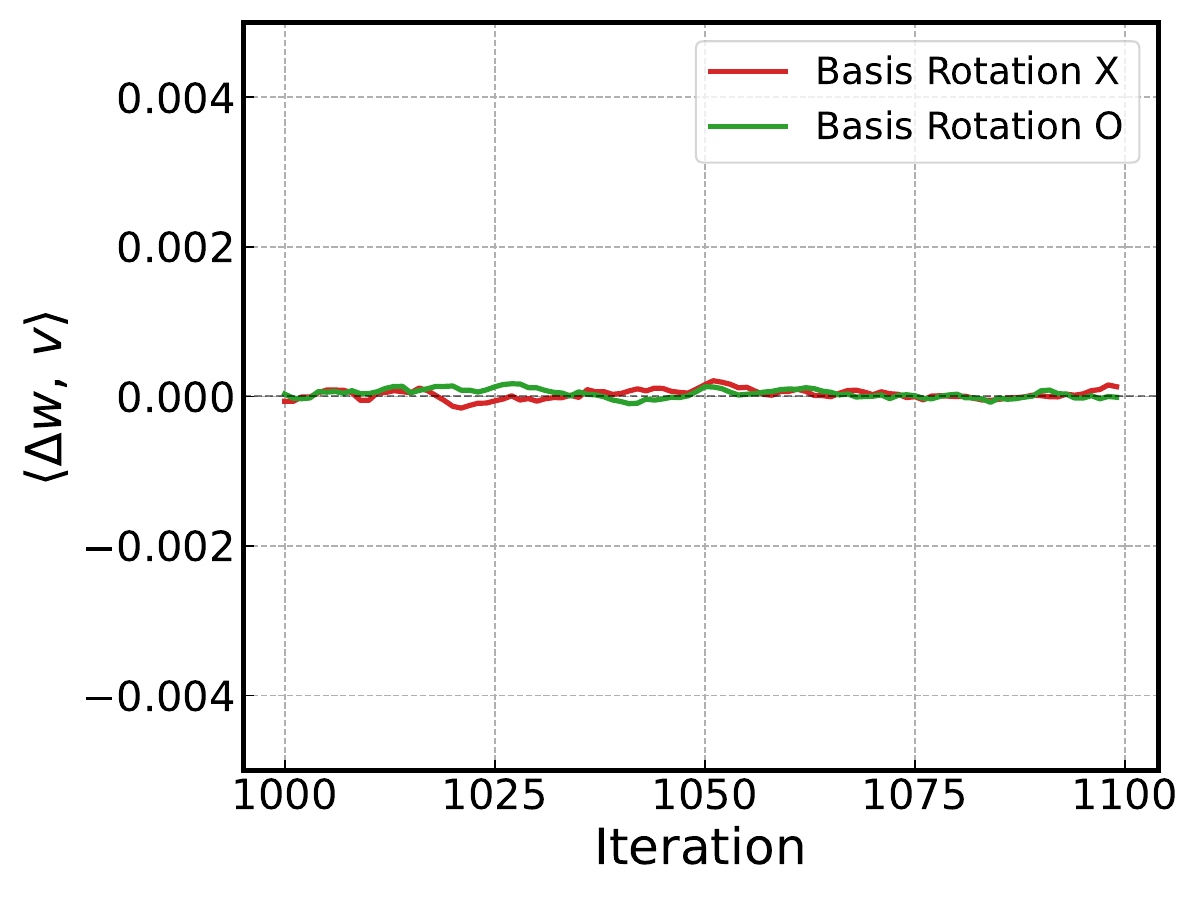}
        \caption{Nondominant eigenvector}
        \label{fig:result-main-oscillation-nondominant}
    \end{subfigure}   
    \caption{
    Oscillation of parameter updates along the (a) dominant and (b) non-dominant eigenvectors of the Hessian. 
    Standard training without basis rotation exhibits severe oscillations along the dominant direction. 
    Applying basis rotation effectively dampens these oscillations. Updates along the non-dominant direction remain stable in both settings.
    }
    \label{fig:result-main-oscillation}
\end{figure}

\paragraph{Empirical Validation of Basis Alignment}
To confirm whether our core arguments made in \cref{sec:analysis} hold in practical LLM pre-training, we empirically validate that basis alignment reduces update oscillations along dominant eigenvectors and decreases the Hessian $(1,1)$-norm.
We estimate the dominant and non-dominant eigenvectors midway through training and track parameter updates along them. 
As shown in \cref{fig:result-main-oscillation}, standard training exhibits severe oscillations along the dominant eigenvector, whereas basis rotation substantially dampens them; non-dominant directions remain stable in both cases.
We further measure the Hessian $(1,1)$-norm -- our theoretical proxy for basis misalignment -- via trace estimation with random Cauchy vectors \citep{xie2024adam}.
Basis rotation reduces the normalized Hessian $(1,1)$-norm (per parameters) from $0.5436$ to $0.1228$.
This confirms that our framework realigns the optimization space, neutralizing the conditions that make adaptive optimizers vulnerable to gradient delay (see \cref{app:sub:empirical-exp-details} for detailed experimental settings).

\paragraph{Additional Results}
We further compare \ourmethod{} against preconditioned optimizers (SOAP \citep{vyas2024soap}, Muon \citep{jordan2024muon}, Scion \citep{pethick2025training}) and verify generalization to Mixture-of-Experts (MoE) architectures; see \cref{app:additional-exp-results} for details.

%% file: tex/6_related_work.tex
\section{Related Works}

\paragraph{Pipeline Parallelism}

GPipe \citep{huang2019gpipe} introduces synchronous pipeline parallelism but its efficiency is limited by pipeline bubbles.
While several works propose advanced pipeline scheduling algorithms to improve device utilization, they still suffer from synchronization bottlenecks \citep{fan2021dapple,li2021chimera}.
Asynchronous pipeline parallelism significantly maximizes throughput by entirely removing the synchronization step but exhibits performance degradation due to delayed gradients \citep{narayanan2019pipedream,narayanan2021memory}.
Prior works address the issue with learning rate scheduling \citep{yang2021pipemare,zhuang2021fully}, 
Nesterov momentum \citep{ajanthan2025nesterov}, and future weight prediction \citep{chen2018efficient,guan2019xpipe}.
However, these methods fail to effectively address delayed gradients under large-scale settings.

\paragraph{Adam for LLM Training}

Adam \citep{kingma2014adam} is widely used for training LLMs due to its superior performance over SGD.
Recent works explain this success from various aspects including directional smoothness \citep{zhang2020gradientclipping,pan2023toward}, heterogeneity of the Hessian spectrum in Transformers \citep{zhang2024transformers}, and heavy-tailed noise of language data \citep{kunstner2024heavy}. 
Meanwhile, Adam loses its advantage over SGD when the Hessian is not near-diagonal, or equivalently, when the Hessian eigenbasis is not aligned to the standard basis \citep{wang2020adasgdbridginggapsgd, zhang2024adam}.
This is because Adam is not equivariant under rotation and its coordinate-wise adaptivity fails to work effectively for the basis-misaligned loss landscapes \citep{xie2024adam,zhang2024understanding}.
We demonstrate that this problem exacerbates the impact of delay.

\paragraph{Hessian Approximation}

Hessian is often approximated via Gauss-Newton matrix \citep{schraudolph2002fast,botev2017practical}, Fisher \citep{amari1998natural,martens2020new}, or empirical Fisher \citep{roux2007topmoumoute, singh2020woodfisher,frantar2021m}, and their Kronecker-factored variants \citep{martens2015optimizing,gupta2018shampoo,vyas2024soap}.
While Gauss-Newton and Fisher matrices serve as a robust proxy for the Hessian \citep{martens2020new,grosse2021adaptive}, they require multiple backpropagation which leads to complex scheduling in asynchronous pipeline parallelism.
Recent works show that Kronecker-factored empirical Fisher can serve as a practical surrogate for the Gauss-Newton matrix \citep{morwani2024new}.
Based on these works, we utilize the eigenvectors of the Kronecker-factored empirical Fisher to perform basis rotation, ensuring a principled yet computationally tractable transformation.

%% file: tex/7_conclusion.tex
\section{Conclusion}

In this work, we show that asynchronous pipeline parallelism is fundamentally challenged by gradient staleness: its convergence properties and model performance degrade critically as the delay grows with pipeline depth or model size.
We identify that the root of this degradation lies in basis misalignment, which prevents a coordinate-wise adaptive optimizer from effectively navigating the loss landscape and leads to amplified oscillations under delay.
To address this, we propose \solution{}, a framework that rotates the optimization space to realign the Hessian eigenbasis with the standard coordinate basis.
Our extensive experiments show that \solution{} effectively neutralizes the impact of delay, significantly accelerating convergence in practical settings, and restoring scalable asynchronous training.

%% file: tex/acknowledgements.tex
\section*{Acknowledgements}

We thank Dongyeop Lee, Jinseok Chung, and Jihun Kim for helpful discussions during the early development of this project.
This work was supported by the Institute of Information \& Communications Technology Planning \& Evaluation (IITP) grant funded by the Korean government (MSIT) (RS-2019-II191906, Artificial Intelligence Graduate School Program (POSTECH), RS-2026-25507427, Development of Efficient Architectures and Training Techniques for High-Performance Lightweight AI Models), the National Research Foundation of Korea (NRF) grant funded by the Korea government (MSIT) (RS-2023-00210466), and Korea Basic Science Institute (National research Facilities and Equipment Center) grant funded by the Korea government(MSIT) (RS-2026-25500419).
Hyunji Jung was supported by the Korea Student Aid Foundation (KOSAF).
Sungbin Shin was supported by Kwanjeong Educational Foundation Scholarship.

%% file: tex/impact_statement.tex
\section*{Impact Statement}

This paper presents work aimed at advancing the field of machine learning by proposing basis rotation to mitigate delay in asynchronous pipeline parallelism.
While the continued development of machine learning models has broad societal implications, there are no specific ethical concerns or negative consequences of this research that we feel must be highlighted here.

%% file: tex/appendix.tex
\crefalias{section}{appendix} 
\crefalias{subsection}{appendix} 

\input{tex/901_stage_analysis}

\input{tex/910_linear_traj_analysis}

\input{tex/911_adam_rotation_equivalence}

\input{tex/902_experiment_details}

\input{tex/903_convergence_analysis}

\input{tex/904_proof_approximation_comparison}

\input{tex/906_connection_recent_optimizers}

\input{tex/912_memory_overhead}

\input{tex/907_more_experiments}

%% file: tex/901_stage_analysis.tex
\section{Analysis for the Number of Stages for Pipeline Parallelism}

\label{app:stage-analysis}

This analysis analyzes the number of stages for pipeline parallelism when training LLMs of varying sizes. 
We calculate the minimum number of pipeline stages required to fit different model architectures onto standard GPU devices with different memory capacities.
By modeling the memory footprint of a single Transformer block including parameters, optimizer states, and activation, we derive a formula to determine how many blocks can fit on a single device, and consequently, how many devices (stages) are needed for the full model.

\paragraph{Setting}
We assume a standard training recipe commonly used for dense LLMs. 
Specifically, we assume mixed precision training with AdamW optimizer for a standard Transformer architecture.
We assume that gradient checkpointing is applied at the beginning of each Transformer block while weight stashes are offloaded to CPU to save memory.
We only employ pipeline parallelism and do not use context, tensor, or fully sharded data parallelism.
This is because these strategies require frequent collective communications and can be impractical for low-bandwidth settings \citep{rajbhandari2020zero,shoeybi2019megatron}.

\paragraph{Notation}

We denote the embedding dimension, the number of attention heads, and the sequence length as $h$, $a$, and $s$ respectively.
We also denote batch size as $b$.
$W$ represents the number of parameters in a single Transformer block.
$L$ represents the total number of Transformer blocks.

\paragraph{Memory required for a single block}
To determine the number of stages, we first calculate the memory required for a single Transformer block.
The memory consumption for one Transformer block consists of parameters, gradients, optimizer states, and activations. 
Assuming mixed-precision training, we need the following memory: (1) $2W$ bytes for half-precision parameters, (2) $2W$ bytes for half-precision gradients, (3) $4W + 4W + 4W = 12W$ bytes for full-precision optimizer states (each corresponding to full-precision master weights, first momentum, and second momentum), (4) $\approx 34 sbh + 5bas^2$ bytes for activations. 
The memory for activation is the result from \citet{korthikanti2023reducing}.
Summing these components, the total memory required for a single block, $M_{\text{block}}$, is
\begin{equation}
    M_{\text{block}} = 16W + 34sbh + 5bas^2 \quad \text{(bytes)}
\end{equation}
Consequently, the total memory for $N$ blocks is $N \times M_{\text{block}}$.

\paragraph{Calculating Required Stages} 
Let $m$ be the memory capacity of a single device (in bytes).
We calculate $N_{\text{max}}$, the maximum number of Transformer blocks that can fit on a single device, using the inequality:
\begin{equation}
    N M_{\text{block}} \leq m,
\end{equation}
which leads to $N_{\text{max}} = \left\lfloor \frac{m}{16W + 34sbh + 5bas^2} \right\rfloor$.
Based on $N_{\text{max}}$, the minimum number of pipeline stages $P$ required to host the full model (which has $L$ total blocks) is determined as follows: (1) when $N_{\text{max}} \geq 1$, at least one block fits on a device and $P = \left\lceil \frac{L}{N_{\text{max}}} \right\rceil$\footnote{This is when we ignore the memory needed for loading embedding layers and language model head. The number of stages may further increase for large embedding layers.} , and (2) when $N_{\text{max}} = 0$, a single device cannot hold even one Transformer block and $P \geq 2 L$ (since we require at least two stages to allocate a single block).

We calculate the required number of stages ($P$) for various LLaMA models \citep{touvron2023llama,dubey2024llama} across different GPU hardware.
Specifically, we assume a batch size of $1$ for simplicity, \ie, $b=1$.
We also assume a sequence length of $4096$, \ie, $s=4096$, which are typically used for initial-phase training of recent LLMs \citep{touvron2023llama2,team2024qwen2,gunter2024apple,abdin2024phi}.
We note that recent LLMs typically use much larger values of $b$ and $s$  so the number of requires stages may increase further.

\begin{table}[H]
\centering
\setcounter{table}{0}
\caption{
Number of required stages ($P$) for different LLaMA models with sequence length $s=4096$ and batch size $b=1$.
Values marked with * indicate that a single Transformer block cannot fit on one device ($N_{max}=0$).}
\label{tab:stages_s4096}
\resizebox{\textwidth}{!}{
\begin{tabular}{lccccccccc}
\toprule
Model & $h$ & $a$ & $W$ & $L$ & RTX3070 & RTX3080 & RTX3090 & A6000 & A100 \\
 & & & & & ($8$GB) & ($16$GB) & ($24$GB) & ($48$GB) & ($80$GB) \\
\midrule
Llama $3.2$ 1B & $2048$ & $32$ & $\approx 67M$ & $16$ & $16$ & $6$ & $4$ & $2$ & $1$ \\
Llama $3.2$ 3B & $3072$ & $24$ & $\approx 113M$ & $28$ & $28$ & $10$ & $6$ & $3$ & $2$ \\
LLaMA $1$-7B & $4096$ & $32$ & $\approx 202M$ & $32$ & $32$ & $16$ & $11$ & $5$ & $3$ \\
LLaMA $1$-13B & $5120$ & $40$ & $\approx 317M$ & $40$ & $\ge 80^*$ & $40$ & $20$ & $8$ & $5$ \\
LLaMA $1$-33B & $6656$ & $52$ & $\approx 535M$ & $60$ & $\ge 120^*$ & $60$ & $60$ & $20$ & $12$ \\
LLaMA $1$-65B & $8192$ & $64$ & $\approx 810M$ & $80$ & $\ge 160^*$ & $\ge 160^*$ & $80$ & $40$ & $20$ \\
Llama $3.1$ 405B & $16384$ & $128$ & $\approx 3.19B$ & $126$ & $\ge 512^*$ & $\ge 512^*$ & $\ge 512^*$ & $\ge 512^*$ & $126$ \\
\bottomrule

\end{tabular}
}
\end{table}

The results are presented in \cref{tab:stages_s4096}.
For devices with small memory (\eg, RTX3070), we need at least $80$ stages even for $13$B models.
For larger models, the number of stages easily increases to hundreds.
Even with a device with larger memory (\eg, A6000), we need tens of stages for moderate-size models (\eg, $33$B, $65$B). 
This analysis confirms that the number of stages in pipeline parallelism can easily reach tens or hundreds for the case of LLMs.

%% file: tex/910_linear_traj_analysis.tex
\section{\texorpdfstring{Stability of the Linearized Trajectory under Delay (\cref{subsec:mechanism})}{Stability of the Linearized Trajectory under Delay (Section 2.1)}}\label{app:linear_traj_analysis}

Before delving into the main argument, we first define the standard and delayed trajectories. The actual trajectory without delay is updated by
\begin{equation}
    x_{t} = x_{t-1} + u_t = x_{t-1} - \eta \frac{m_t}{\sqrt{v_t+\epsilon}},
\end{equation}
where $m_t = \beta_1 m_{t-1} + (1-\beta_1) g_t$ and $v_t = \beta_2 v_{t-1} + (1-\beta_2) g_t^2$, and $g_t = \nabla f(x_{t-1};\xi_t)$ denotes the stochastic gradient evaluated at the current iterate $x_{t-1}$ with batch $\xi_t$.
The trajectory generated by the delayed gradients obeys the update rule
\begin{equation}
    \tilde{x}_{t} = \tilde{x}_{t-1} + \tilde{u}_t = \tilde{x}_{t-1} - \eta \frac{\tilde{m}_t}{\sqrt{\tilde{v}_t+\epsilon}},
\end{equation}
with the moments following the identical recursion
\begin{equation}
    \tilde{m}_t = \beta_1 \tilde{m}_{t-1} + (1-\beta_1) \tilde{g}_t,
    \qquad
    \tilde{v}_t = \beta_2 \tilde{v}_{t-1} + (1-\beta_2) \tilde{g}^2_t .
\end{equation}
The only difference from the non-delayed trajectory lies in which gradient is consumed.
Under a delay of $\tau$, the gradient applied at update step $t$ was produced $\tau$ updates earlier, i.e., evaluated at the iterate $\tilde{x}_{t-1-\tau}$:
\begin{equation}
    \tilde{g}_t =
    \begin{cases}
        \nabla f(\tilde{x}_{0};\, \xi_t), & t \le \tau, \\[3pt]
        \nabla f(\tilde{x}_{t-1-\tau};\, \xi_t), & t > \tau.
    \end{cases}
    \label{eq:delayed_grad}
\end{equation}

Now, we can derive the following statement: \emph{When the gradient signal dominates the noise ($\nabla f(x;\xi)\approx\nabla f(x)$), the linearized non-delayed trajectory ($u_t\approx u$ for all $1\le t\le s$) makes the delayed trajectory closely approximate the original ($\tilde{x}_t\approx x_t$ for all $0\le t\le s$).}

We justify the statement by strong induction on $t$.

\begin{proof}
\textbf{Base case ($t=0$):} $\tilde{x}_0 = x_0$ by initialization.

\textbf{Inductive step for $t\le\tau$:} For every $j\le t$, $\tilde{x}_0 = x_0$, hence by signal-dominance,
\begin{equation}
\tilde{g}_j = \nabla f(x_0;\xi_j) \approx \nabla f(x_0),
\quad
\tilde{m}_j \approx (1-\beta_1)\nabla f(x_0),
\quad
\tilde{v}_j \approx (1-\beta_2)\nabla f(x_0)^{2}.
\end{equation}
Using the inductive hypothesis and the local consistency assumption ($u_t \approx u_1$):
\begin{align*}
    \tilde{x}_{t} &= \tilde{x}_{t-1} - \eta \frac{\tilde{m}_t}{\sqrt{\tilde{v}_t+\epsilon}} \\
    &\approx \tilde{x}_{t-1} - \eta \frac{(1-\beta_1) g_1}{\sqrt{(1-\beta_2) g_1^2 + \epsilon}} \\
    &= \tilde{x}_{t-1} + u_1 \\
    &\approx x_{t-1} + u_t = x_{t}.
\end{align*}

\textbf{Inductive step for $t>\tau$:} By the inductive hypothesis, $\tilde{x}_{t-1-\tau}\approx x_{t-1-\tau}$. Applying the signal-dominance assumption twice together with continuity:
\begin{equation}
\tilde{g}_t = \nabla f(\tilde{x}_{t-1-\tau};\xi_t)
\approx \nabla f(\tilde{x}_{t-1-\tau})
\approx \nabla f(x_{t-1-\tau})
\approx \nabla f(x_{t-1-\tau};\xi_{t-\tau}) = g_{t-\tau}.
\end{equation}
The cross-batch step $\xi_t$ vs.\ $\xi_{t-\tau}$ is bridged through the deterministic intermediate $\nabla f(x_{t-1-\tau})$, so no condition on the noise law is invoked. Consequently, the moving averages also approximate their un-delayed counterparts at step $t-\tau$, yielding $\tilde{m}_t \approx m_{t-\tau}$ and $\tilde{v}_t \approx v_{t-\tau}$. Therefore,
\begin{align*}
    \tilde{x}_{t} &= \tilde{x}_{t-1} - \eta \frac{\beta_1 \tilde{m}_{t-1} + (1-\beta_1) \tilde{g}_t}{\sqrt{\beta_2 \tilde{v}_{t-1} + (1-\beta_2) \tilde{g}^2_t + \epsilon}} \\
    &\approx x_{t-1} - \eta \frac{\beta_1 m_{t-\tau-1} + (1-\beta_1) g_{t-\tau}}{\sqrt{\beta_2 v_{t-\tau-1} + (1-\beta_2) g^2_{t-\tau} + \epsilon}} \\
    &= x_{t-1} + u_{t-\tau}.
\end{align*}
Applying the local consistency assumption again ($u_{t-\tau} \approx u_t$), we conclude:
\begin{equation*}
    \tilde{x}_{t} \approx x_{t-1} + u_t = x_{t}.
\end{equation*}
Thus, by strong induction, the delayed trajectory $\tilde{x}_t$ follows the original trajectory $x_t$ closely along the optimization path.
\end{proof}

%% file: tex/911_adam_rotation_equivalence.tex
\section{\texorpdfstring{Derivation of Adam's Equivalence under Fixed Rotation (\cref{subsec:solution})}{Derivation of Adam's Equivalence under Fixed Rotation (Section 3.1)}}\label{app:adam_rotation_equi}
In this section, we provide a detailed derivation to show that applying the standard Adam update in the basis-aligned coordinate system $\tilde{w} = \gU^\top w$ mathematically recovers the basis rotation update rule in the original space (\cref{section4:eq:rotated2}).
Let the objective function in the rotated parameter space be defined as $\tilde{f}(\tilde{w}) \triangleq f(\gU \tilde{w})$. By the multivariable chain rule, the gradient with respect to the rotated parameters $\tilde{w}$ is given by:
\begin{align*}
    \nabla \tilde{f}(\tilde{w}_{t-1}) = \gU^\top \nabla f(\gU \tilde{w}_{t-1}) = \gU^\top \nabla f(w_{t-1}).
\end{align*}
Denoting $\tilde{g}_s = \nabla \tilde{f}(\tilde{w}_{s-1}) = \gU^\top \nabla f(w_{s-1})$ for each step $s$, the standard Adam update rule applied to the parameters $\tilde{w}$ in this rotated space yields:
\begin{align*}
    \tilde{w}_{t} = \tilde{w}_{t-1} - \eta_t \frac{\tilde{m}_t}{\sqrt{\tilde{v}_t + \epsilon}},
\end{align*}
where the first and second moment estimates ($\tilde{m}_t$ and $\tilde{v}_t$) are computed using the gradients in the rotated space:
\begin{align*}
    \tilde{m}_t &= (1-\beta_1) \sum_{s=1}^{t} \beta_1^{t-s} \tilde{g}_s = (1-\beta_1) \sum_{s=1}^{t} \beta_1^{t-s} \gU^\top \nabla f(w_{s-1}), \\
    \tilde{v}_t &= (1-\beta_2) \sum_{s=1}^{t} \beta_2^{t-s} \tilde{g}_s^2 = (1-\beta_2) \sum_{s=1}^{t} \beta_2^{t-s} (\gU^\top \nabla f(w_{s-1}))^2.
\end{align*}
To express this update trajectory in the original parameter space, we multiply both sides of the update equation by the orthogonal rotation matrix $\gU$. Using the identity $\gU \tilde{w} = \gU \gU^\top w = w$, we obtain:
\begin{align*}
    \gU \tilde{w}_{t} &= \gU \tilde{w}_{t-1} - \eta_t \gU \frac{\tilde{m}_t}{\sqrt{\tilde{v}_t + \epsilon}}, \\
    w_{t} &= w_{t-1} - \eta_t \gU \frac{(1-\beta_1) \sum_{s=1}^{t} \beta_1^{t-s} \gU^\top \nabla f(w_{s-1})}{\sqrt{(1-\beta_2) \sum_{s=1}^{t} \beta_2^{t-s} (\gU^\top \nabla f(w_{s-1}))^2 + \epsilon}}.
\end{align*}
This strictly recovers \cref{section4:eq:rotated2}, confirming that Adam in the rotated space is equivalent to applying the rotation matrix $\gU$ to the gradients, computing the coordinate-wise adaptive steps, and projecting the update back to the original space.

%% file: tex/902_experiment_details.tex
\section{Experimental Details}
\label{app:exp-details}

\subsection{\texorpdfstring{\cref{sec:analysis}}{Section 2}}
\label{app:sub:toy-exp-details}
\paragraph{\cref{subsec:mechanism}: \cref{fig:quadratic}}
To compare Adam's coordinate-wise adaptivity with SGD, we employed AdaSGD \citep{wang2020adasgdbridginggapsgd}, which scales a uniform learning rate across all parameters by taking the exponential moving average of the average second momentum.
This setup enabled us to use an identical learning rate for both AdaSGD and Adam, facilitating a direct comparison of how gradient delay interacts with coordinate-wise adaptivity.
We set the learning rate to $1.0$ and $\beta_1 = 0$ in order to isolate the adversarial update directions induced by delayed gradients without the confounding effect of momentum.
For $\beta_2$, we used $0.1$ which is smaller than the standard Adam setting, to more clearly visualize oscillatory behavior.
Nevertheless, our empirical analysis in \cref{subsec:analysis-empirical} confirms that this alignment-dependent instability persists even with standard hyperparameters in practical training scenarios.
Finally, we defined the convergence criterion as reaching a loss of $15.0$ for each optimizer, tested with and without a delay of $\tau=2$.

\paragraph{\cref{subsec:analysis-empirical}: \cref{fig:toy-example}}
The spiral loss landscape was defined as $f(r, \theta) = r^2 + (20 \sin(4 r - \theta) + 1)^2$, where $r$ and $\theta$ represent the radius and angle in polar coordinates.
We set the learning rate to $0.1$ and $\beta_1 = 0$ to focus on the effect of delayed gradients without momentum. 
We used $\beta_2 = 0.9$, matching standard Adam settings.
Even with such a large $\beta_2$, Adam exhibited low oscillation when the Hessian eigenbasis is aligned with the standard coordinate basis, whereas severe oscillations emerge in the misaligned region.
For \cref{fig:delay_conv_speed}, we introduced a delay of $\tau = 1$ at a randomly chosen iteration during training without delay.
We then measured the number of iterations required to traverse an angular displacement of $3^\circ$ with and without delay, and report their ratio as the slowdown metric.

\subsection{\texorpdfstring{\cref{sec:experiments}}{Section 4}}
\label{app:sub:LLM-exp-details}

Our model is based on nanoGPT \citep{Karpathy2022}.
The model has an embedding dimension of $384$, sequence length of $512$, $6$ attention heads, and $32$ Transformer blocks.
The model has learnable positional embedding and untied language model head.
When the number of stages is $P$, we allocate $32/P$ Transformer blocks to each stage. 
The first and the last stage also hold embedding and language model head respectively.
This model has approximately $96$ million parameters.

The model is trained in bf16 mixed precision.
We use the batch size of $8$, which is the maximum without incurring out-of-memory error in our setting, and do not use gradient accumulation since it changes the level of delay \citep{ajanthan2025nesterov}.
For each method and stage, we search the learning rate among $\{0.0001, 0.0003, 0.001\}$.
$\beta_1$ is set to $0.99$ for Nesterov following the original paper \citep{ajanthan2025nesterov} and set to $0.9$ for others.
The learning rate is scheduled using a linear warmup for the first $1.2\%$ of iterations, followed by a cosine decay for the remainder of the training process.
We use $\beta_2$ of $0.999$, gradient clipping value of $1.0$, and weight decay of $0.01$.
Unless specified otherwise, we set the basis update frequency to $10$ for \ourmethod{}.
We only perform rotation to the MLP and attention layers excluding embedding layers, language model head, bias parameters, and layer normalization parameters.
We use at most eight RTX 3090 GPUs for the experiments.
When the number of pipeline stages $P$ exceeds the number of available physical GPUs, we configure multiple virtual stages to share a single GPU; for example, a $32$-stage experiment on eight RTX 3090 GPUs assigns four virtual stages per GPU.
This configuration preserves stage depth and gradient delay exactly, since the pipeline scheduling and communication order between forward, backward, and update operations are unchanged.
While absolute hardware utilization and per-iteration wall-clock time may be affected by context-switching overhead, this overhead applies uniformly across all baselines and \ourmethod{}, so the comparative results---including the loss vs. GPU hour comparison in \cref{fig:result-main-efficiency-runtime}---remain valid.

Our $1$B model has an embedding dimension of $1728$, sequence length of $512$, $27$ attention heads, and $24$ Transformer blocks.
The model has approximately $1.03$ billion parameters.
We use the learning rate of $0.0001$.
We use at most six A100 GPUs for these experiments.

\subsection{Empirical Validation in \cref{fig:result-main-oscillation}}

\label{app:sub:empirical-exp-details}

To observe the optimization trajectory and measure oscillation along specific eigenvector directions in \cref{fig:result-main-oscillation}, we estimate the dominant eigenvector of the Hessian midway through the training process using power iteration.
For the non-dominant direction, we sample a random vector and explicitly orthogonalize it against the estimated dominant eigenvector. 
We then track the projection of the parameter updates along these two isolated directions for a span of $100$ iterations. 
For the Hessian $(1,1)$-norm estimation, we follow the methodology of \citet{xie2024adam}, utilizing Hessian-vector products with random Cauchy vectors to approximate the norm without explicitly constructing the full Hessian matrix. 
Specifically, this estimation is rigorously computed over $8,000$ batches using $200$ sampled random Cauchy vectors.
The resulting values are then normalized by dividing them by the total number of parameters in the model to provide a standardized metric for comparison.

%% file: tex/903_convergence_analysis.tex
\section{\texorpdfstring{Proof of \cref{thm:convergence}}{Proof of Theorem 2.3}}
\label{app:convergence-proof}

In this section, we present the proof of \cref{thm:convergence}.
In the following, we use $H_i$ instead of $\Smn_i$ defined in \cref{assumption:smoothness}.

We first present several lemmas necessary for the proof.
The lemmas are taken from \citet{xie2024adam} for the case of Adam. 
\begin{lemma}\label[lemma]{lemma:secondorderterm}
(Lemma D.3. of \citet{xie2024adam})
    For any twice differentiable loss which is $H$-coordinate-wisely-smooth w.r.t. $\ell_\infty$ norm, we have for any $x$ and $\Delta \in \R^d$,
    \begin{align*}
        | \Delta^\top \nabla^2 f(x) \Delta | \leq \sum_{i=1}^d H_i \Delta_i^2.
    \end{align*}
\end{lemma}

\begin{lemma} \label[lemma]{lemma:accumulated-squared-update}
    (Lemma 3.12. of \citet{xie2024adam})
    For any $0 < \beta_2 < 1$, for any scalar sequences $\{v_t\}_{t=0}^T$ and $\{g_t\}_{t=1}^T$ satisfying $v_0 \geq 0, v_1 > 0,$ and $ v_t - \beta_2 v_{t-1} \geq (1 - \beta_2) g_t^2$ for $t \geq 1$, the following holds:
    \begin{align*}
        \sum_{t=1}^T \frac{g_t^2}{v_t} \leq T + \frac{\beta_2}{1-\beta_2} \ln \frac{v_T}{v_0}.
    \end{align*}
\end{lemma}

\begin{lemma}\label[lemma]{lemma:denominator}
(Lemma 3.13. of \citet{xie2024adam})
    Under \cref{assumption:bounded-stochasticity}, for any $i \in [d]$, the following holds:
    \begin{align*}
        \E \sum_{t=1}^T \frac{g_{t,i} \bar{g}_{t,i}}{\sqrt{v_{t,i} + \epsilon}} \geq \frac{1}{2} \E \sum_{t=1}^T \frac{\bar{g}_{t,i}^2}{\sqrt{\tilde{v}_{t,i} + \epsilon}} - \sqrt{1 - \beta_2} T \sigma_i - \frac{\sigma_i \beta_2}{\sqrt{1-\beta_2}} \E  \left[\ln 
    \frac{v_{T,i} + \epsilon}{v_{0,i} + \epsilon}\right],
    \end{align*}
    where $\tilde{v}$ is defined as $\tilde{v}_{t,i} = (1-\beta_2) (\bar{g}_{t,i}^2 + \sigma_i^2) + \beta_2 v_{t-1,i}$ and $g_t$ and $\bar{g}_t$ are stochastic the full-batch gradient respectively.
\end{lemma}

\begin{lemma}  \label[lemma]{lemma:first-order-approx}
(Lemma 3.14. of \citet{xie2024adam})
Under \cref{assumption:bounded-stochasticity}, for any $i \in [d]$, the following holds:
\begin{align*}
    \sum_{t=\frac{T}{2} + 1}^T \E \left[\sqrt{\tilde{v}_{t,i} + \epsilon}\right] \leq \frac{2\beta_2^{\frac{T}{4}}}{1-\beta_2} \sqrt{v_{0,i}}
 + \frac{T}{2} \sigma_i + \frac{T}{2} \sqrt{\epsilon} + 2 \sum_{t=1}^T \E \left[\frac{\bar{g}_{t,i}^2}{\sqrt{\tilde{v}_{t,i} + \epsilon}} \right]
 \end{align*}
\end{lemma}

\begin{lemma} \label[lemma]{lemma:c1}
    (Lemma D.1. of \citet{xie2024adam})
    With \cref{assumption:bounded-stochasticity,assumption:smoothness}, the following holds for any $T$:
    \begin{align*}
        \ln \frac{\E \max_{i \in [d]} v_{T,i} + \epsilon}{v_0 + \epsilon} \leq 2 \ln \left(1 + \frac{\sum_{i=1}^d \sigma_i^2 + \| \nabla f(x_0) \|_\infty^2 + \sum_{i=1}^d H_i^2 \eta^2 T (T + \frac{1}{1 - \beta_2})}{v_0 + \epsilon} + \ln 32 \right)
    \end{align*}
\end{lemma}

Now, we are ready to prove \cref{thm:convergence}.
The key step for analyzing the convergence rate is to address the (B) term in \cref{eq:lowerbound}.
\begin{proof}

First, by the Taylor approximation, we have
\begin{align*}
f({x_t}) = f({x_{t-1}}) + \nabla f({x_{t-1}}) ({x_t - x_{t-1}}) + \frac{1}{2}({x_t - x_{t-1}})^\top \nabla^2 f(x) ({x_t - x_{t-1}})
\end{align*}
for some $x$. By \cref{lemma:secondorderterm},
\begin{align*}
f({x_t}) - f({x_{t-1}}) &= \nabla f({x_{t-1}}) ({x_t - x_{t-1}}) + \frac{1}{2} \sum_{i=1}^d H_i ({x_{t,i} - x_{t-1,i}})^2 \\
&= - \eta \sum_{i=1}^d \frac{g_{t-\tau,i} \bar{g}_{t,i}}{\sqrt{v_{t-\tau,i} + \epsilon}} + \frac{1}{2} \eta^2 \sum_{i=1}^d H_i \frac{g_{t-\tau,i}^2}{v_{t-\tau,i} + \epsilon} \\
&= \underbrace{- \eta \sum_{i=1}^d \frac{\bar{g}_{t-\tau,i} g_{t-\tau,i} }{\sqrt{v_{t-\tau,i} + \epsilon}}}_{(A)} 
\underbrace{+ \eta \sum_{i=1}^d \frac{(\bar{g}_{t-\tau,i} - \bar{g}_{t,i}) g_{t-\tau,i} }{\sqrt{v_{t-\tau,i} + \epsilon}}}_{(B)} 
\underbrace{+ \frac{1}{2} \eta^2 \sum_{i=1}^d H_i \frac{g_{t-\tau,i}^2}{v_{t-\tau,i} + \epsilon}}_{(C)}. \tag{1} \label{eq:lowerbound}
\end{align*}

Note that we have added and subtracted $- \eta \sum_{i=1}^d \frac{\bar{g}_{t-\tau,i} g_{t-\tau,i} }{\sqrt{v_{t-\tau,i} + \epsilon}}$ so that we can use \cref{lemma:denominator}.
Since (A) and (C) can be upper bounded with \cref{lemma:denominator} and \cref{lemma:accumulated-squared-update} respectively, the remaining step is to upper bound (B).

By Cauchy-Schwarz inequality, we have
\begin{align*}
    \eta \sum_{i=1}^d \frac{(\bar{g}_{t-\tau,i} - \bar{g}_{t,i} ) g_{t-\tau,i} }{\sqrt{v_{t-\tau,i} + \epsilon}}  &\leq \eta \underbrace{\sqrt{\sum_{i=1}^d (\bar{g}_{t-\tau,i} - \bar{g}_{t,i} )^2}}_{A_t} \cdot \underbrace{\sqrt{\sum_{i=1}^d \frac{g_{t-\tau,i}^2}{v_{t-\tau,i} + \epsilon}}}_{B_t}.
\end{align*}

Summing over $t$ and applying Cauchy-Schwarz inequality leads to
\begin{align*}
    \eta \sum_{t=1}^{T+\tau} A_t B_t \leq \eta \sqrt{\sum_{t=1}^{T+\tau} A_t^2} \sqrt{\sum_{t=1}^{T+\tau} B_t^2} .
\end{align*}

We can upper bound $\sqrt{\sum_{t=1}^{T+\tau} A_t^2}$ as follows.
\begin{align*}
\sqrt{\sum_{t=1}^{T+\tau} A_t^2} &= \sqrt{\sum_{t=1}^{T+\tau} \sum_{i=1}^d (\bar{g}_{t,i} - \bar{g}_{t-\tau,i})^2} \\
& \leq \sqrt{\sum_{t=1}^{T+\tau} \sum_{i=1}^d H_i^2 \| {x_{t-1} - x_{t-\tau-1}} \|_\infty^2} \\
& \leq \sqrt{\sum_{t=1}^{T+\tau} \sum_{i=1}^d H_i^2 \| {x_{t-1} - x_{t-\tau-1}} \|_2^2} \\
& = \sqrt{\sum_{t=1}^{T+\tau} \sum_{i=1}^d H_i^2 \left \| \sum_{k=1}^\tau \eta \frac{g_{t-k-\tau}}{\sqrt{v_{t-k-\tau} + \epsilon}} \right \|_2^2} \\
& \leq \sqrt{\sum_{t=1}^{T+\tau} \sum_{i=1}^d H_i^2 \tau \eta^2 \sum_{k=1}^\tau \left\| \frac{g_{t-k-\tau}}{\sqrt{v_{t-k-\tau} + \epsilon}} \right\|_2^2 } \\
& = \sqrt{\tau \eta^2 \sum_{i=1}^d H_i^2 \sum_{t=1}^{T+\tau} \sum_{k=1}^\tau \left \| \frac{g_{t-k-\tau}}{\sqrt{v_{t-k-\tau} + \epsilon}} \right \|_2^2} \\
& \leq \sqrt{\tau \eta^2 \sum_{i=1}^d H_i^2 \tau \sum_{t=1}^{T+\tau} \left \| \frac{g_{t-\tau}}{\sqrt{v_{t-\tau} + \epsilon}} \right \|_2^2} \\
& = \tau \eta \sqrt{\sum_{i=1}^d H_i^2} \cdot \sqrt{\sum_{t=1}^{T+\tau} \left \| \frac{g_{t-\tau}}{\sqrt{v_{t-\tau} + \epsilon}} \right \|_2^2} \\
&= \tau \eta \sqrt{\sum_{i=1}^d H_i^2} \cdot \sqrt{ \sum_{t=1}^T \sum_{i=1}^d \frac{g_{t,i}^2}{v_{t,i} + \epsilon}}
\end{align*}

We have used \cref{assumption:smoothness} for the first inequality.
The second inequality holds because $\ell_\infty$ norm of a vector is always smaller than or equal to its $\ell_2$ norm.
We have used Jensen's inequality for the third inequality.
The final inequality holds because each term $\frac{g_{t-\tau}}{\sqrt{v_{t-\tau} + \epsilon}}$ is summed at most $\tau$ times.

Thus, we have
\begin{align*}
    \eta \sum_{t=1}^T A_t B_t & \leq \eta \sqrt{\sum_{t=1}^T A_t^2} \sqrt{\sum_{t=1}^T B_t^2} \\
    & \leq \tau \eta^2 \sqrt{\sum_{i=1}^d H_i^2} \cdot \sum_{t=1}^T \sum_{i=1}^d \frac{g_{t-\tau,i}^2}{v_{t-\tau,i} + \epsilon}. \tag{2} \label{eq:crossterm-final}
\end{align*}

Next, applying \cref{lemma:denominator} to (A) of \cref{eq:lowerbound} gives 
\begin{align*}
    \sum_{t=1}^{T+\tau} \E (A) &=  -\eta \sum_{t=1}^{T+\tau} \sum_{i=1}^d \E \frac{\bar{g}_{t-\tau,i} g_{t-\tau,i} }{\sqrt{v_{t-\tau,i} + \epsilon}} \\
    & = - \eta \sum_{i=1}^d \sum_{t=1}^{T} \E \frac{\bar{g}_{t,i} g_{t,i} }{\sqrt{v_{t,i} + \epsilon}} \\
    & \leq -\eta \sum_{i=1}^d  \frac{1}{2} \E \sum_{t=1}^{T} \frac{\bar{g}_{t,i}^2}{\sqrt{\tilde{v}_{t,i}^2 + \epsilon}} + \eta \sum_{i=1}^d \sqrt{1-\beta_2} T \sigma_i + \eta \sum_{i=1}^d \frac{\sigma_i \beta_2}{\sqrt{1-\beta_2}} \E \left [\ln \frac{v_{T,i} + \epsilon}{v_{0,i} + \epsilon} \right].
\end{align*}



Then, by summing over $t$ from  $t=0$ to $T_\tau$, dividing by $T$, applying \cref{lemma:accumulated-squared-update} to (C) and \cref{eq:crossterm-final}, and rearranging the terms, we get the following.

\begin{align*}
    l &= \frac{1}{T} \E \left[\sum_{i=1}^d \sum_{t=1}^{T} \frac{\bar{g}_{t,i}^2}{\sqrt{\tilde{v}_{t,i} + \epsilon}}\right]  \\ & \leq
    \frac{2}{\eta T} \E [f(x_0) - f(x_{T+\tau})] + \frac{\eta}{T} \E \left [\sum_{i=1}^d H_i \left(T + \frac{\beta_2}{1-\beta_2}\ln \frac{v_{T,i} + \epsilon}{v_{0,i} + \epsilon}\right) \right] \\
    & + \frac{2}{T} \sum_{i=1}^d \sigma_i \sqrt{1-\beta_2} \left(T   + \frac{\beta_2}{1-\beta_2} \E \ln \frac{v_{T,i} + \epsilon}{v_{0,i} + \epsilon}\right) \\
    & + \frac{2}{T} \tau \eta \sqrt{\sum_{i=1}^d H_i^2} \E \sum_{i=1}^d \left(T + \frac{\beta_2}{1-\beta_2} \ln \frac{v_{T,i} + \epsilon}{v_{0,i}+ \epsilon}\right) \\
    & \leq \frac{2}{\eta T} \E [f(x_0) - \min_xf(x)] + \eta \sum_{i=1}^d H_i + 2 \sqrt{1-\beta_2} \sum_{i=1}^d \sigma_i + 2 \tau \eta d \sqrt{\sum_{i=1}^d H_i^2} \\ 
    &+ \frac{\beta_2}{T (1-\beta_2)} \left(\eta \sum_{i=1}^d H_i + 2 \sqrt{1-\beta_2} \sum_{i=1}^d \sigma_i + 2\tau \eta d \sqrt{\sum_{i=1}^d H_i^2}\right) \max_{i} \E \ln \frac{v_{T,i} + \epsilon}{v_{0,i} + \epsilon} \\
    & \leq \frac{2}{\eta T} \E [f(x_0) - \min_xf(x)] + \eta \sum_{i=1}^d H_i + 2 \sqrt{1-\beta_2} \sum_{i=1}^d \sigma_i + 2 \tau \eta d \sqrt{\sum_{i=1}^d H_i^2} \\
    &+ \frac{\beta_2}{T (1-\beta_2)} \left(\eta \sum_{i=1}^d H_i + 2 \sqrt{1-\beta_2} \sum_{i=1}^d \sigma_i + 2\tau \eta d \sqrt{\sum_{i=1}^d H_i^2}\right) \ln \frac{\E \max_i v_{T,i} + \epsilon}{v_0 + \epsilon} \\ 
    & \leq \frac{2}{\eta T} \E [f(x_0) - \min_xf(x)] + \eta (1 + 2 \tau d) \sum_{i=1}^d H_i + 2 \sqrt{1-\beta_2} \sum_{i=1}^d \sigma_i \\
    &+ \leq \frac{\beta_2}{T (1-\beta_2)} \left(\eta (1 + 2\tau d) \sum_{i=1}^d H_i + 2 \sqrt{1-\beta_2} \sum_{i=1}^d \sigma_i\right) \ln \frac{\E \max_i v_{T,i} + \epsilon}{v_0 + \epsilon}. 
\end{align*}

We have used $\sqrt{\sum_{i=1}^d H_i^2} \leq \sum_{i=1}^d H_i$ for the final inequality.

From \cref{lemma:c1}, we define 
\begin{align*}
    E = \frac{2}{\eta T }\E [f(x_0) - \min_xf(x)] + \eta (1 + 2\tau d) \sum_{i=1}^d H_i + 2 \sqrt{1-\beta_2} \sum_{i=1}^d \sigma_i \\ + \frac{\beta_2}{T (1-\beta_2)} \left(\eta (1 + 2\tau d) \sum_{i=1}^d H_i + 2 \sqrt{1-\beta_2} \sum_{i=1}^d \sigma_i \right) F 
\end{align*}
where 
\begin{align*}
    F = 2 \ln \left(1 + \frac{\sum_{i=1}^d \sigma_i^2 + \| \nabla f(x_0) \|_\infty^2 + \sum_{i=1}^d H_i^2 \eta^2 T \left (T + \frac{1}{1 - \beta_2}\right) }{v_0 + \epsilon}\right) + \ln 32.
\end{align*}

Then,
\begin{align*}
    \frac{1}{T}\E \left [\sum_{i=1}^d \sum_{t=1}^T \frac{\bar{g}_{t,i}^2}{\sqrt{\tilde{v}_{t,i} + \epsilon}}\right] \leq E.
\end{align*}

By \cref{lemma:first-order-approx} and Cauchy-Schwarz inequality,

\begin{align*}
    \frac{2}{T} \E \sum_{t=\frac{T}{2} + 1}^T \sum_{i=1}^d | \bar{g}_{t,i} | & \leq \left ( \frac{2}{T} \sum_{t=\frac{T}{2} +1}^T \sum_{i=1}^d \frac{ \bar{g}_{t,i} ^2}{\sqrt{\tilde{v}_{t,i} + \epsilon}} \right)^{1/2} \left( \frac{2}{T} \E \sum_{t=\frac{T}{2} + 1}^T \sum_{i=1}^d \sqrt{\tilde{v}_{t,i} + \epsilon} \right)^{1/2}  \\
    &\leq \sqrt{2E} \left(4E + \frac{4\beta_2^{T/4}}{T (1-\beta_2)} d \sqrt{v_0} + \sum_{i=1}^d \sigma_i + d  \sqrt{\epsilon}\right)^{1/2} \\
    &\leq 2\sqrt{2}E + \sqrt{2}\sqrt{E}\sqrt{\frac{4\beta_2^{T/4}}{T (1-\beta_2)} d \sqrt{v_0} + \sum_{i=1}^d \sigma_i + d \sqrt{\epsilon}}.
\end{align*}

Finally, the following holds:
\begin{align*}
\min_{\frac{T}{2} < t \leq T} \E [\| \nabla f(x_t) \|_1] &\leq \frac{1}{T} \sum_{t=\frac{T}{2} + 1}^T \sum_{i=1}^d | \bar{g}_{t,i} | \\
&\leq \mathcal{O}\left(E + \sqrt{E} \sqrt{\frac{\beta_2^{T/4}}{T (1-\beta_2)} d \sqrt{v_0} + \sum_{i=1}^d \sigma_i + d \sqrt{\epsilon}}\right)
\end{align*}

with 
\begin{align*}
    E = \frac{2}{\eta T} \E [f(x_0) - f(x_{T+\tau})] + (1 + \frac{\beta_2 F}{T (1 - \beta_2)}) \left(\eta \sum_{i=1}^d H_i + 2 \sqrt{1-\beta_2} \sum_{i=1}^d \sigma_i + 2 \tau \eta \sqrt{\sum_{i=1}^d H_i^2}\right)
\end{align*}
and 
\begin{align*}
    F = 2 \ln \left(1 + \frac{\sum_{i=1}^d \sigma_i^2 + \| \nabla f(x_0) \|_\infty^2 + \sum_{i=1}^d H_i^2 \eta^2 T\left (T + \frac{1}{1 - \beta_2}\right) }{v_0 + \epsilon}\right) + \ln 32.
\end{align*}

When we assume 
$v_0 + \epsilon > \left(\sum_{i=1}^d \sigma_i^2 + \| \nabla f(x_0) \|_\infty^2 + \sum_{i=1}^d H_i^2 \eta^2\right) / \text{poly} (T)$ and $\frac{1}{1-\beta_2} = \text{poly}(T)$, then $F = \mathcal{O}(\log T)$.
Terms involving $\sum_{i=1}^d \sigma_i$ have a lower bound $\Theta \left(\sum_{i=1}^d \sigma_i \left(\frac{\log T}{T}\right)^{1/2}\right)$ with $1 - \beta_2 = \Theta \left(\frac{\log T}{T}\right)$.
Terms involving $\eta$ has a lower bound $\Theta\left(\sqrt{\frac{(f(x_0) - \min_x f(x))  (1 + \tau)\sum_{i=1}^d H_i}{T}}\right)$ reached by $\eta = \Theta\left(\sqrt{\frac{f(x_0) - \min_x f(x)}{T (1 + \tau d)\sum_{i=1}^d H_i}}\right)$.

For $R \triangleq (f(x_0 ) - \min_x f(x)) \sum_{i=1}^d H_i$, $\beta_1=0$, $1-\beta_2 = \Theta\left(\frac{\log T}{T}\right)$, $\eta = \Theta\left (\sqrt{\frac{f(x_0) - \min_x f(x)}{T (1 + \tau d )\sum_{i=1}^d H_i}}\right)$, $v_0 + \epsilon > \left(\sum_{i=1}^d \sigma_i^2 + \| \nabla f(x_0) \|_\infty^2 + \sum_{i=1}^d H_i^2 \eta^2\right) / \poly (T)$, and $\frac{1}{1-\beta_2} = \poly(T)$, 
we have
\begin{align*}
    \min_{\frac{T}{2} < t \leq T} \E \| \nabla f(x_t) \|_1 = \mathcal{O} \left(\sqrt{\frac{(1 + d \tau) R}{T}} + \sqrt{\sum_{i=1}^d \sigma_i} \left(\frac{(1 + d \tau) R}{T}\right)^{1/4} + \sum_{i=1}^d \sigma_i \left (\frac{\log T}{T}\right)^{1/4} + \delta_T\right) 
\end{align*}

with $\delta_T = \sqrt{\frac{dv_0}{T (1 - \beta_2)}} \exp \left(- \frac{T (1 - \beta_2)}{8}\right) \left[\left(\frac{(1 + d \tau)R}{T}\right)^{1/4} + \sqrt{\sum_{i=1}^d \sigma_i }\left(\frac{\log T}{T}\right)^{1/4}\right]$ .
\end{proof}

\subsection{Extension to Stage-Dependent Delay}
\label{app:convergence-stage-delay}

In this section, we extend the convergence analysis of \cref{thm:convergence} to accommodate stage-dependent delays, which more accurately reflects the structural mechanics of asynchronous pipeline parallelism. 

Let $S_1, \dots, S_K$ denote the partition of parameter coordinates across $K$ pipeline stages. For a coordinate $i \in S_k$, the gradient delay is determined by its stage distance, i.e., $\tau_i = K - k$. The update rule for coordinate $i$ is thus given by:
\begin{align*}
    \update{x_{t,i}} = \update{x_{t-1,i}} - \eta \frac{g_{t_i,i}}{\sqrt{v_{t_i,i} + \epsilon}}
\end{align*}
where $t_i = t - \tau_i$.

\begin{theorem}
\label{thm:stage-dependent-convergence}
Under the same assumptions as \cref{thm:convergence}, the convergence rate of asynchronous Adam with stage-dependent delays $\{\tau_i\}$ is as follows:
\begin{align*}
    \min_{\frac{T}{2} < t \leq T} & \E \| \nabla f(w_t) \|_1 = \mathcal{O} & \left (\sqrt{\frac{(1 + d \tau') \Delta_0 \Smn}{T}}  +\sqrt{\sum_{i=1}^d \sigma_i} \left(\frac{(1 + d \tau') \Delta_0 \Smn}{T}\right)^{1/4} + \sum_{i=1}^d \sigma_i \left(\frac{\log T}{T}\right)^{1/4}\right),
\end{align*}
where the effective delay $\tau'$ is defined as:
\begin{equation}
    \tau' \triangleq \sqrt{\frac{\sum_{i=1}^d \Smn_i^2 \tau_i^2}{\sum_{i=1}^d \Smn_i^2}} = \sqrt{\frac{\sum_{k=1}^K (K-k)^2 \sum_{i \in \mathcal{S}_k} \Smn_i^2}{\sum_{i=1}^d \Smn_i^2}}.
\end{equation}
\end{theorem}

\begin{proof}

We follow the structure of the constant delay proof in \cref{app:convergence-proof}.
\Cref{eq:lowerbound} is modified by replacing the constant $\tau$ with the coordinate-specific $\tau_i$. The bounds for terms (A) and (C) follow the exact same steps using \cref{lemma:denominator} and \cref{lemma:accumulated-squared-update}.

For term (B), we redefine $A_t$ and $B_t$ by replacing $t-\tau$ with $t_i$:
\begin{align*}
    A_t = \sqrt{\sum_{i=1}^d (\bar{g}_{t_i,i} - \bar{g}_{t,i})^2}, \quad B_t = \sqrt{\sum_{i=1}^d \frac{g_{t_i,i}^2}{v_{t_i,i} + \epsilon}}.
\end{align*}

Applying the Cauchy-Schwarz inequality over the extended time horizon up to $T+K$ yields:
\begin{align*}
    \eta \sum_{t=1}^{T+K} A_t B_t \leq \eta \sqrt{\sum_{t=1}^{T+K} A_t^2} \sqrt{\sum_{t=1}^{T+K} B_t^2}.
\end{align*}

Next, we upper bound $\sum_{t=1}^{T+K} A_t^2$. Using \cref{assumption:smoothness}, the property that the $\ell_\infty$ norm is bounded by the $\ell_2$ norm, and Jensen's inequality, we have:
\begin{align*}
\sum_{t=1}^{T+K} A_t^2 &= \sum_{t=1}^{T+K} \sum_{i=1}^d (\bar{g}_{t,i} - \bar{g}_{t_i,i})^2 \\
&\leq \sum_{t=1}^{T+K} \sum_{i=1}^d H_i^2 \| \update{x_{t-1} - x_{t_i-1}} \|_2^2 \\
&= \sum_{t=1}^{T+K} \sum_{i=1}^d H_i^2 \left\| \sum_{k=1}^{\tau_i} \eta \frac{g_{t-k_j,j}}{\sqrt{v_{t-k_j,j} + \epsilon}} \right\|_2^2 \\
&\leq \sum_{t=1}^{T+K} \sum_{i=1}^d H_i^2 \tau_i \eta^2 \sum_{k=1}^{\tau_i} \left\| \frac{g_{t-k_j,j}}{\sqrt{v_{t-k_j,j} + \epsilon}} \right\|_2^2 \\
&\leq \eta^2 \sum_{i=1}^d H_i^2 \tau_i^2 \sum_{t=1}^{T+K} \sum_{j=1}^d \frac{g_{t_j,j}^2}{v_{t_j,j} + \epsilon}
\end{align*}
The final inequality holds because each term $\frac{g_{t_j,j}}{\sqrt{v_{t_j,j} + \epsilon}}$ within the sliding window of size $\tau_i$ is summed at most $\tau_i$ times. 

Therefore, taking the square root, we obtain:
\begin{align*}
    \sqrt{\sum_{t=1}^{T+K} A_t^2} \leq \eta \sqrt{\sum_{i=1}^d H_i^2 \tau_i^2} \cdot \sqrt{\sum_{t=1}^{T+K} \sum_{j=1}^d \frac{g_{t_j,j}^2}{v_{t_j,j} + \epsilon}}
\end{align*}
and consequently:
\begin{align*}
    \eta \sum_{t=1}^{T+K} A_t B_t \leq \eta^2 \sqrt{\sum_{i=1}^d H_i^2 \tau_i^2} \cdot \sum_{t=1}^{T+K} \sum_{j=1}^d \frac{g_{t_j,j}^2}{v_{t_j,j} + \epsilon}.
\end{align*}

Comparing this to the constant delay case in \cref{eq:crossterm-final}, the term $\tau \sqrt{\sum_{i=1}^d H_i^2}$ is seamlessly replaced by $\sqrt{\sum_{i=1}^d H_i^2 \tau_i^2}$. We can define the effective delay $\tau'$ as:
\begin{align*}
    \tau' \triangleq \sqrt{\frac{\sum_{i=1}^d H_i^2 \tau_i^2}{\sum_{i=1}^d H_i^2}} = \sqrt{\frac{\sum_{k=1}^K (K-k)^2 \sum_{i \in S_k} H_i^2}{\sum_{i=1}^d H_i^2}}.
\end{align*}
This formulation allows us to rewrite $\sqrt{\sum_{i=1}^d H_i^2 \tau_i^2} = \tau' \sqrt{\sum_{i=1}^d H_i^2}$.

Following the remaining algebraic rearrangements as in \cref{app:convergence-proof}, and preserving the dimensionality dependence $d$ (which inherently arises from bounding the maximum over $d$ dimensions of the logarithmic variance terms), we arrive at the corresponding convergence bound:
\begin{align*}
    \min_{\frac{T}{2} < t \leq T} \E \| \nabla f(x_t) \|_1 = \mathcal{O} \left(\sqrt{\frac{(1 + d \tau') R}{T}} + \sqrt{\sum_{i=1}^d \sigma_i} \left(\frac{(1 + d \tau') R}{T}\right)^{1/4} + \sum_{i=1}^d \sigma_i \left (\frac{\log T}{T}\right)^{1/4} + \delta_T\right) 
\end{align*}
where $R \triangleq (f(x_0) - \min_x f(x)) \sum_{i=1}^d H_i$.

This result demonstrates that modeling stage-dependent delays actively tightens the constant factor ($\tau' \leq K$, the maximum delay) while preserving the fundamental multiplicative interaction between the delay and the basis misalignment (represented by the $H_i$ terms). More importantly, the explicit formulation of $\tau'$ provides a rigorous theoretical validation for stage-aware basis rotation: by actively dampening the misalignment magnitude ($H_i^2$) in earlier pipeline stages where the delay $(K-k)$ is most severe, the effective delay $\tau'$ is directly minimized, yielding an accelerated convergence bound.

\end{proof}

\input{tex/913_convergence_analysis_moment}

%% file: tex/913_convergence_analysis_moment.tex
\subsection{Extension to \texorpdfstring{$\beta_1 > 0$}{beta1 > 0}}
\label{app:convergence-momentum}

In this section, we extend \cref{thm:convergence} to the heavy-ball momentum
setting ($\beta_1 > 0$). The asynchronous Adam update with momentum is
\begin{align}
m_{t,i} &= \beta_1\, m_{t-1,i} + (1-\beta_1)\, g_{t,i}, \label{eq:mom_update}\\
v_{t,i} &= \beta_2\, v_{t-1,i} + (1-\beta_2)\, g_{t,i}^2,\\
x_{t,i} &= x_{t-1,i} - \eta\, \frac{m_{t-\tau,i}}{\sqrt{v_{t-\tau,i} + \epsilon}}.
\end{align}
Equivalently, $m_s = (1-\beta_1)\sum_{k=0}^{s-1}\beta_1^{k}g_{s-k}$.
We extend the auxiliary variance shorthand to incorporate a momentum index:
for $k \in \{0,1,\dots,s-1\}$,
\begin{align}\label{eq:tildev_def_mom}
\tilde{v}_{s,k+1,i}\;\triangleq\; \beta_2^{k+1}\, v_{s-k-1,i}
\;+\; (1-\beta_2)\,\E_{s-k-1}\!\Bigg[\sum_{j=s-k}^{s} \beta_2^{s-j}\,g_{j,i}^2\Bigg],
\end{align}
where $\E_{s-k-1}[\cdot]$ denotes conditional expectation w.r.t.\
$f_1, \dots, f_{s-k-1}$. By construction, $\tilde{v}_{s,k+1,i}$ is
$\sigma(f_1,\dots,f_{s-k-1})$-measurable and therefore independent of
$g_{s-k,i}, \dots, g_{s,i}$. It reduces to $\tilde{v}_{s,i}$ of
\cref{lemma:denominator} when $k=0$.

For the heavy-ball extension we additionally adopt the standard
coordinate-wise almost-sure gradient bound used in
\citet{dfossez2022a}.

\begin{assumption}
\label[assumption]{assumption:bounded-gradient}
    (Coordinate-wise bounded gradient)
    There exists $G > 0$ such that
    $\| \nabla f(w; \xi) \|_\infty \leq G$ almost surely for all $w \in \R^d$,
    where $\xi$ denotes the stochasticity of data.
\end{assumption}

\begin{theorem}[Convergence with momentum]\label{thm:convergence_mom}
Under \cref{assumption:bounded-stochasticity,assumption:smoothness,assumption:bounded-gradient},
for asynchronous Adam with $0 \le \beta_1 < \beta_2 < 1$ where $\beta_1$
is a fixed constant independent of $T$, uniform delay $\tau$, and
hyperparameter choices $1 - \beta_2 = \Theta(\log T / T)$ and
$\eta = \Theta\!\big(\sqrt{\Delta_0/(T\,d\,(1+\tau)\sum_i H_i)}\big)$
with $\Delta_0 \triangleq f(x_0) - \min_x f(x)$,
the following holds for all $T \ge T_0= 4/(1-\beta_1)$,
\[
\min_{T/2 < t \le T}\, \E\|\nabla f(x_t)\|_1
\;=\; \mathcal{O}\!\left(
\sqrt{\frac{d\,(1+\tau)\,R}{T}}
\;+\; \sqrt{\textstyle\sum_i \sigma_i}\,\Big(\tfrac{d\,(1+\tau)\,R}{T}\Big)^{\!1/4}
\;+\; \sqrt{G\,d\textstyle\sum_i \sigma_i}\,\Big(\tfrac{\log T}{T}\Big)^{\!1/4}
\right),
\]
where $R \triangleq \Delta_0 \sum_{i=1}^d H_i$.
\end{theorem}

\subsubsection{Auxiliary Lemmas}

The proof requires three additional lemmas. The first is a momentum-aware counterpart of \cref{lemma:denominator}.

\begin{lemma}[Hybrid descent lemma with momentum]\label{lemma:hybrid-descent}
Under \cref{assumption:bounded-stochasticity,assumption:smoothness,assumption:bounded-gradient},
for any $i \in [d]$ and $0 \le \beta_1 < \beta_2 \le 1$,
\begin{align*}
\E \sum_{s=1}^{T} \frac{\bar{g}_{s,i}\, m_{s,i}}{\sqrt{v_{s,i}+\epsilon}}
\;\ge\;& \frac{1-\beta_1}{4}\,\E \sum_{s=1}^{T} \sum_{k=0}^{s-1}\beta_1^{k}\,
\frac{\bar{g}_{s-k,i}^{2}}{\sqrt{\tilde{v}_{s,k+1,i}+\epsilon}} \\
& -\; \frac{\beta_1\, \eta\, H_i}{(1-\beta_1)^{2}}\,
\E \sum_{w=1}^{T}\Big\|\frac{m_w}{\sqrt{v_w+\epsilon}}\Big\|_2^{2} \\
& -\; \frac{\sqrt{1-\beta_1}\,\big[4G\sqrt{1-\beta_2} \;+\; H_i\,\eta\big]}{(1-\beta_1/\beta_2)^{3/2}}\,
\E\!\sum_{u=1}^{T}\frac{g_{u,i}^{2}}{v_{u,i}+\epsilon}.
\end{align*}
\end{lemma}

\begin{proof}
Fix coordinate $i$ and drop it from the notation; write
$a_s := \sqrt{v_s+\epsilon}$ and $b_s^{(k)} := \sqrt{\tilde v_{s,k+1}+\epsilon}$.
Using $m_s = (1-\beta_1)\sum_{k=0}^{s-1}\beta_1^{k} g_{s-k}$ and the
elementary identity $\bar g_s g_{s-k} = \bar g_{s-k}g_{s-k} +
(\bar g_s-\bar g_{s-k})g_{s-k}$,
\begin{align}
\frac{\bar g_s m_s}{a_s}
\;=\;& (1-\beta_1)\sum_{k=0}^{s-1}\beta_1^{k}
\underbrace{\frac{\bar g_{s-k}\,g_{s-k}}{a_s}}_{(\mathcal A_k)}
\;+\; (1-\beta_1)\sum_{k=0}^{s-1}\beta_1^{k}
\underbrace{\frac{(\bar g_s-\bar g_{s-k})\,g_{s-k}}{a_s}}_{(\mathcal B_k)}.
\label{eq:mom_decomp}
\end{align}
We lower-bound the conditional expectation of $\mathcal A_k$ in
Step~A and upper-bound $|\mathcal B_k|$ in Step~B.

\medskip
\noindent\textbf{Step A: Bounding $\mathcal A_k$.}
Define the partial sums
\begin{align}\label{eq:delta_r_def}
\delta_s^{2} \;:=\; (1-\beta_2)\sum_{j=s-k}^{s}\beta_2^{s-j}g_{j}^{2},
\qquad
r_s^{2} \;:=\; (1-\beta_2)\sum_{j=s-k}^{s}\beta_2^{s-j}\E_{s-k-1}[g_{j}^{2}],
\end{align}
so that $a_s^{2} = \beta_2^{k+1}v_{s-k-1} + \delta_s^{2} + \epsilon$ and
$\big(b_s^{(k)}\big)^{2} = \beta_2^{k+1}v_{s-k-1} + r_s^{2} + \epsilon$, hence
$a_s^{2} - \big(b_s^{(k)}\big)^{2} = \delta_s^{2} - r_s^{2}$, $a_s \ge \delta_s$,
and $b_s^{(k)} \ge r_s$. Note that $r_s, \beta_2^{k+1}v_{s-k-1}, \tilde v_{s,k+1}$,
and $b_s^{(k)}$ are all $\sigma(f_1,\dots,f_{s-k-1})$-measurable.
By the tower property,
\begin{align}
\E\!\left[\frac{\bar g_{s-k}\,g_{s-k}}{b_s^{(k)}}\right]
\;=\; \E\!\left[\frac{\bar g_{s-k}^{\,2}}{b_s^{(k)}}\right],
\quad\text{so}\quad
\E\!\left[\frac{\bar g_{s-k}\,g_{s-k}}{a_s}\right]
\;=\; \E\!\left[\frac{\bar g_{s-k}^{\,2}}{b_s^{(k)}}\right]
+ \E[C_{s,k}],
\label{eq:A_identity}
\end{align}
where $C_{s,k} \triangleq \bar g_{s-k}\,g_{s-k}\!\left(1/a_s - 1/b_s^{(k)}\right)$. 
Using $|x-y| = |x^{2}-y^{2}|/(x+y)$ and
$a_s + b_s^{(k)} \ge \max(a_s, b_s^{(k)})$,
\begin{align}
\left|\frac{1}{a_s} - \frac{1}{b_s^{(k)}}\right|
\;=\; \frac{|a_s^{2}-(b_s^{(k)})^{2}|}{a_s\,b_s^{(k)}\,(a_s+b_s^{(k)})}
\;\le\; \frac{r_s^{2}}{a_s\,(b_s^{(k)})^{2}} + \frac{\delta_s^{2}}{a_s^{2}\,b_s^{(k)}},
\label{eq:multistep_decomp}
\end{align}
where the final inequality uses $|\delta_s^{2}-r_s^{2}|\le \delta_s^{2}+r_s^{2}$.
Hence $|C_{s,k}| \le \kappa_{s,k} + \rho_{s,k}$ with
\begin{align*}
\kappa_{s,k} \;=\; \frac{|\bar g_{s-k}\,g_{s-k}|\,r_s^{2}}{a_s\,(b_s^{(k)})^{2}},
\qquad
\rho_{s,k} \;=\; \frac{|\bar g_{s-k}\,g_{s-k}|\,\delta_s^{2}}{a_s^{2}\,b_s^{(k)}}.
\end{align*}

To bound $\kappa_{s,k}$, apply 
$|xy|\le \tfrac{\lambda}{2}x^{2}+\tfrac{1}{2\lambda}y^{2}$ with
$\lambda = \tfrac{1}{2}\sqrt{1-\beta_1}\,b_s^{(k)}$,
$x = |\bar g_{s-k}|/b_s^{(k)}$, and
$y = |g_{s-k}|\,r_s^{2}/(a_s\,b_s^{(k)})$:
\begin{align}\label{eq:kappa_bound_raw}
\kappa_{s,k}
\;\le\; \frac{\sqrt{1-\beta_1}\,\bar g_{s-k}^{\,2}}{4\,b_s^{(k)}}
\;+\; \frac{1}{\sqrt{1-\beta_1}}\cdot\frac{g_{s-k}^{2}\,r_s^{4}}
{a_s^{2}\,(b_s^{(k)})^{3}}.
\end{align}
Take $\E_{s-k-1}[\cdot]$. Since $r_s, b_s^{(k)}, \bar g_{s-k}$ are
$\sigma(f_1,\dots,f_{s-k-1})$-measurable, and using
$(b_s^{(k)})^{2}\ge r_s^{2}$ to bound $r_s^{4}/(b_s^{(k)})^{3}\le r_s^{2}/b_s^{(k)}$:
\begin{align}\label{eq:kappa_bound_cond}
\E_{s-k-1}[\kappa_{s,k}]
\;\le\; \frac{\sqrt{1-\beta_1}\,\bar g_{s-k}^{\,2}}{4\,b_s^{(k)}}
\;+\; \frac{r_s^{2}}{\sqrt{1-\beta_1}\,b_s^{(k)}}\,
\E_{s-k-1}\!\!\left[\frac{g_{s-k}^{2}}{a_s^{2}}\right].
\end{align}

To bound $\rho_{s,k}$, apply the same inequality with
$\lambda = \tfrac{1}{2}\sqrt{1-\beta_1}\,b_s^{(k)}/r_s^{2}$,
$x = |\bar g_{s-k}|\,\delta_s/b_s^{(k)}$, and
$y = |\delta_s g_{s-k}|/a_s^{2}$:
\begin{align*}
\rho_{s,k}
\;\le\; \frac{\sqrt{1-\beta_1}\,\bar g_{s-k}^{\,2}\,\delta_s^{2}}{4\,r_s^{2}\,b_s^{(k)}}
\;+\; \frac{r_s^{2}\,\delta_s^{2}\,g_{s-k}^{2}}{\sqrt{1-\beta_1}\,b_s^{(k)}\,a_s^{4}}.
\end{align*}
Taking $\E_{s-k-1}[\cdot]$, using $\E_{s-k-1}[\delta_s^{2}]=r_s^{2}$ on the
first summand and $\delta_s^{2}\le a_s^{2}$ on the second:
\begin{align}\label{eq:rho_bound_cond}
\E_{s-k-1}[\rho_{s,k}]
\;\le\; \frac{\sqrt{1-\beta_1}\,\bar g_{s-k}^{\,2}}{4\,b_s^{(k)}}
\;+\; \frac{r_s^{2}}{\sqrt{1-\beta_1}\,b_s^{(k)}}\,
\E_{s-k-1}\!\!\left[\frac{g_{s-k}^{2}}{a_s^{2}}\right].
\end{align}

Adding \cref{eq:kappa_bound_cond,eq:rho_bound_cond} and using
$r_s^{2}/b_s^{(k)}\le r_s$ together with
$r_s \le G\sqrt{(k+1)(1-\beta_2)}$ (from \cref{assumption:bounded-gradient}
and the geometric sum
$\sum_{j=0}^{k}\beta_2^{j}\le k+1$):
\begin{align}\label{eq:Cterm_bound}
\E_{s-k-1}[|C_{s,k}|]
&\le \frac{\sqrt{1-\beta_1}\,\bar g_{s-k}^{\,2}}{2\,b_s^{(k)}}
\;+\; \frac{2\,G\,\sqrt{(k+1)(1-\beta_2)}}{\sqrt{1-\beta_1}}\,
\E_{s-k-1}\!\!\left[\frac{g_{s-k}^{2}}{a_s^{2}}\right]\\
&\le \frac{\sqrt{1-\beta_1}}{2}\,\E\!\!\left[\frac{\bar g_{s-k}^{\,2}}{b_s^{(k)}}\right]
\;+\; \frac{2\,G\,\sqrt{(k+1)(1-\beta_2)}}{\sqrt{1-\beta_1}\,\beta_2^{k}}\,
\E[\frac{g_{s-k}^{2}}{v_{s-k}+\epsilon}],
\end{align}
where we use $a_s^{2} \ge \beta_2^{k}(v_{s-k}+\epsilon)$ for the second inequality. Finally, we have

\begin{align}
(1-\beta_1)\sum_{s,k}\beta_1^{k}\E\left[\frac{\bar g_{s-k}\,g_{s-k}}{a_s}\right] 
&\ge (1-\beta_1)\sum_{s,k}\beta_1^{k} \left(\E\!\left[\frac{\bar g_{s-k}^{\,2}}{b_s^{(k)}}\right] - \E[|C_{s,k}|]\right)\\
&\ge (1-\beta_1)\sum_{s,k}\beta_1^{k} \left(\Big(1-\tfrac{\sqrt{1-\beta_1}}{2}\Big)\,\E\!\left[\frac{\bar g_{s-k}^{\,2}}{b_s^{(k)}}\right]
\;-\; \frac{2\,G\,\sqrt{(k+1)(1-\beta_2)}}{\sqrt{1-\beta_1}\,\beta_2^{k}}\,\E[\frac{g_{s-k}^{2}}{v_{s-k}+\epsilon}]\right)\\
&\;\ge\; \frac{1-\beta_1}{2}\!\sum_{s,k}\!\beta_1^{k}\,
\E\!\!\left[\frac{\bar g_{s-k}^{\,2}}{b_s^{(k)}}\right]
\;-\; \frac{4\,G\,\sqrt{(1-\beta_1)(1-\beta_2)}}{(1-\beta_1/\beta_2)^{3/2}}\,
\E\sum_u \frac{g_{u}^{2}}{v_{u}+\epsilon},
\label{eq:boundingA}
\end{align}
where we use $1-\sqrt{1-\beta_1}/2 \ge 1/2$ and \citet[Lemma~A.3]{dfossez2022a},
$\sum_{k\ge 0}(\beta_1/\beta_2)^{k}\sqrt{k+1}\le 2/(1-\beta_1/\beta_2)^{3/2}$ in the last inequality.

\medskip
\noindent\textbf{Step B: Bounding $\mathcal B_k$.}
Apply $|xy|\le \tfrac{\lambda}{2}x^{2}+\tfrac{1}{2\lambda}y^{2}$ with
$\lambda = \sqrt{1-\beta_1}/(H_i\eta\sqrt{k+1})$, $x = |\bar g_s - \bar g_{s-k}|$,
$y = |g_{s-k}|/a_s$:
\begin{equation}\label{eq:Bterm_amgm}
\sum_{s=1}^{T} \sum_{k=0}^{s-1}\beta_1^{k} |\mathcal B_k|
\;\le\; \sum_{s=1}^{T} \sum_{k=0}^{s-1}\beta_1^{k} \frac{\sqrt{1-\beta_1}}{2 H_i\eta\sqrt{k+1}}\,(\bar g_{s,i} - \bar g_{s-k,i})^{2}
+ \sum_{s=1}^{T}\sum_{k=0}^{s-1}\beta_1^{k}\frac{H_i\eta\sqrt{k+1}}{2\sqrt{1-\beta_1}}\cdot\frac{g_{s-k,i}^{2}}{a_s^{2}}.
\end{equation}
The second term can be bound by
\begin{align}
\sum_{s=1}^{T}\sum_{k=0}^{s-1}\beta_1^{k}\frac{H_i\eta\sqrt{k+1}}{2\sqrt{1-\beta_1}}\cdot\frac{g_{s-k,i}^{2}}{a_s^{2}} 
&\le \sum_{s=1}^{T}\sum_{k=0}^{s-1}\beta_1^{k} \frac{H_i\,\eta\,\sqrt{k+1}}{2\sqrt{1-\beta_1}\,\beta_2^{k}}\cdot \frac{g_{s-k,i}^{2}}{v_{s-k,i}+\epsilon}\\
&\le \frac{H_i\,\eta}{2 \sqrt{1-\beta_1}}
\sum_{u=1}^{T}\frac{g_{u,i}^{2}}{v_{u,i}+\epsilon}\,
\sum_{k\ge 0}\!\Big(\tfrac{\beta_1}{\beta_2}\Big)^{\!k}\!\sqrt{k+1}\\
&\le \frac{H_i\,\eta}{(1-\beta_1/\beta_2)^{3/2} \sqrt{1-\beta_1}}\,
\sum_{u=1}^{T}\frac{g_{u,i}^{2}}{v_{u,i}+\epsilon},
\end{align}
where we use $a_s^{2}\ge \beta_2^{k}(v_{s-k,i}+\epsilon)$ in the first inequality and \citet[Lemma~A.3]{dfossez2022a} in the last. Taking expectation,
\begin{equation}\label{eq:B2_form}
\E\,(1-\beta_1) \sum_{s=1}^{T}\sum_{k=0}^{s-1}\beta_1^{k}\frac{H_i\eta\sqrt{k+1}}{2\sqrt{1-\beta_1}}\cdot\frac{g_{s-k,i}^{2}}{a_s^{2}} \le
\frac{H_i\,\eta\,\sqrt{1-\beta_1}}{(1-\beta_1/\beta_2)^{3/2}}\,
\E\!\sum_{u=1}^{T}\frac{g_{u,i}^{2}}{v_{u,i}+\epsilon}.
\end{equation}

\medskip

By \cref{assumption:smoothness} and Cauchy Schwarz,
\begin{equation}\label{eq:grad_diff_displacement}
(\bar g_{s,i} - \bar g_{s-k,i})^{2} \le H_i^2 \|x_{s-1}-x_{s-k-1}\|_\infty^2 \le H_i^2 \eta^2 \Big(\sum_{l=1}^{k}\|\tfrac{m_{s-l}}{\sqrt{v_{s-l}+\epsilon}}\|_\infty\Big)^{2}\le H_i^2 \eta^2 k \sum_{l=1}^{k}\|\tfrac{m_{s-l}}{\sqrt{v_{s-l}+\epsilon}}\|_\infty^2.
\end{equation}
Substituting \cref{eq:grad_diff_displacement} into the first term of \cref{eq:Bterm_amgm} gives
\begin{align}
(1-\beta_1)\sum_{s=1}^{T} \sum_{k=0}^{s-1}\beta_1^{k} \frac{\sqrt{1-\beta_1}}{2 H_i\eta\sqrt{k+1}}\,(\bar g_{s,i} - \bar g_{s-k,i})^{2} 
&\le \frac{H_i\,(1-\beta_1)^{3/2}\,\eta}{2}
\sum_{s,k}\!\beta_1^{k}\,\frac{k}{\sqrt{k+1}}\,\sum_{l=1}^{k}\|\frac{m_{s-l}}{\sqrt{v_{s-l}+\epsilon}}\|_\infty^{2} \\
&\le \frac{H_i\,(1-\beta_1)^{3/2}\,\eta}{2} \sum_{u=1}^{T}\|\frac{m_{u}}{\sqrt{v_{u}+\epsilon}}\|_\infty^{2} \sum_{k\ge 1}\!\beta_1^{k}\,\frac{k^{2}}{\sqrt{k+1}}\\
&\le 
\frac{H_i\,(1-\beta_1)^{3/2}\,\eta}{2}\cdot \frac{2\,\beta_1}{(1-\beta_1)^{5/2}}\,
\sum_{u=1}^{T}\|\frac{m_{u}}{\sqrt{v_{u}+\epsilon}}\|_\infty^{2}\\
&= \frac{\beta_1\,H_i\,\eta}{1-\beta_1} \sum_{u=1}^{T}\|\frac{m_{u}}{\sqrt{v_{u}+\epsilon}}\|_\infty^{2},
\end{align}
where we apply $k^{2}/\sqrt{k+1}\le k^{3/2}$ for $k\ge 1$ and
\citet[Lemma~A.4]{dfossez2022a} in the last inequality.

Finally, from the fact that $\|\cdot\|_\infty^{2}\le\|\cdot\|_2^{2}$, we have
\begin{equation}\label{eq:B1_l2form}
(1-\beta_1)\sum_{s=1}^{T} \sum_{k=0}^{s-1}\beta_1^{k} \frac{\sqrt{1-\beta_1}}{2 H_i\eta\sqrt{k+1}}\,(\bar g_{s,i} - \bar g_{s-k,i})^{2} \le  \frac{\beta_1\,H_i\,\eta}{1-\beta_1} \sum_{u=1}^{T}\|\tfrac{m_{u}}{\sqrt{v_{u}+\epsilon}}\|_2^{2}.
\end{equation}

Combining \cref{eq:boundingA,eq:B2_form,eq:B1_l2form} completes the proof.
\end{proof}

The second lemma is a momentum-aware counterpart of \cref{lemma:accumulated-squared-update}.

\begin{lemma}[Momentum variance bound]\label{lemma:mom-variance}
For any $i \in [d]$ and $0 \le \beta_1 < \beta_2 \le 1$,
\[
\sum_{s=1}^{T} \frac{m_{s,i}^{2}}{v_{s,i}+\epsilon}
\;\le\; \frac{1-\beta_1}{1-\beta_1/\beta_2}\,
\sum_{s=1}^{T} \frac{g_{s,i}^{2}}{v_{s,i}+\epsilon}
\;\le\; \frac{1-\beta_1}{1-\beta_1/\beta_2}\!
\left(T + \frac{\beta_2}{1-\beta_2}\,\ln\frac{v_{T,i}+\epsilon}{v_{0,i}+\epsilon}\right).
\]
\end{lemma}

\begin{proof}
By Cauchy--Schwarz applied to $m_{s,i} = (1-\beta_1)\sum_{u=1}^{s}\beta_1^{s-u}g_{u,i}$,
\[
m_{s,i}^{2}
\;\le\; (1-\beta_1)^{2}\!\left(\sum_{u=1}^{s}\beta_1^{s-u}\right)\!
\left(\sum_{u=1}^{s}\beta_1^{s-u}g_{u,i}^{2}\right)
\;\le\; (1-\beta_1)\sum_{u=1}^{s}\beta_1^{s-u}g_{u,i}^{2}.
\]
For $u \le s$, the recursion for $v$ gives $v_{s,i} \ge \beta_2^{s-u}v_{u,i}$,
hence $v_{s,i}+\epsilon \ge \beta_2^{s-u}(v_{u,i}+\epsilon)$. Therefore
\[
\sum_{s=1}^{T} \frac{m_{s,i}^{2}}{v_{s,i}+\epsilon}
\;\le\; (1-\beta_1)\sum_{u=1}^{T}\frac{g_{u,i}^{2}}{v_{u,i}+\epsilon}
\sum_{s\ge u}\!\left(\tfrac{\beta_1}{\beta_2}\right)^{s-u}
\;=\; \frac{1-\beta_1}{1-\beta_1/\beta_2}\sum_{u=1}^{T}\frac{g_{u,i}^{2}}{v_{u,i}+\epsilon}.
\]
\end{proof}
The third lemma extends \cref{lemma:first-order-approx} to handle the
$\beta_1^{k}$-weighted sum over the momentum index.
\begin{lemma}[Momentum-weighted analog of \cref{lemma:first-order-approx}]
\label{lemma:mom-first-order-approx}
Under \cref{assumption:smoothness}, for any $i \in [d]$ and
$0 \le \beta_1 < \beta_2 < 1$,
\[
\sum_{s=T/2+1}^{T} \sum_{k=0}^{s-1} \beta_1^{k}\,
\E\!\Big[\sqrt{\tilde{v}_{s,k+1,i}+\epsilon}\Big]
\;\le\; \frac{1}{1-\beta_1}\!\left[
\frac{2\beta_2^{T/4}}{1-\beta_2}\sqrt{v_{0,i}}
+ \tfrac{T}{2}\sigma_i + \tfrac{T}{2}\sqrt{\epsilon}
+ 2 \sum_{s=1}^{T} \E\!\left[\frac{\bar{g}_{s,i}^{2}}{\sqrt{\tilde{v}_{s,1,i}+\epsilon}}\right]\right].
\]
\end{lemma}

\begin{proof}
Fix $k \in \{0,\dots,s-1\}$. By definition of $\tilde{v}_{s,k+1,i}$ in
\cref{eq:tildev_def_mom}, the proof of \cref{lemma:first-order-approx}
(\citet[Lemma~3.16]{xie2024adam}) applies verbatim to the inner sum
$\sum_{s=T/2+1}^{T}\E[\sqrt{\tilde{v}_{s,k+1,i}+\epsilon}]$ at each fixed $k$,
yielding
\begin{equation}\label{eq:mom_E4_inner}
\sum_{s=T/2+1}^{T}\!\E\!\Big[\sqrt{\tilde{v}_{s,k+1,i}+\epsilon}\Big]
\;\le\; \frac{2\beta_2^{T/4}}{1-\beta_2}\sqrt{v_{0,i}}
+ \tfrac{T}{2}\sigma_i + \tfrac{T}{2}\sqrt{\epsilon}
+ 2 \sum_{s=1}^{T}\E\!\left[\frac{\bar{g}_{s,i}^{2}}{\sqrt{\tilde{v}_{s,k+1,i}+\epsilon}}\right].
\end{equation}
Multiplying \eqref{eq:mom_E4_inner} by $\beta_1^{k}$ and summing over $k$,
\begin{align*}
\sum_{s=T/2+1}^{T}\!\sum_{k=0}^{s-1}\!\beta_1^{k}\,
\E\!\Big[\sqrt{\tilde{v}_{s,k+1,i}+\epsilon}\Big]
\;\le\;& \frac{1}{1-\beta_1}\!\left[\frac{2\beta_2^{T/4}}{1-\beta_2}\sqrt{v_{0,i}}
+ \tfrac{T}{2}\sigma_i + \tfrac{T}{2}\sqrt{\epsilon}\right]\\
&+ 2 \sum_{s=1}^{T}\sum_{k=0}^{s-1}\beta_1^{k}\,\E\!\left[\frac{\bar{g}_{s,i}^{2}}{\sqrt{\tilde{v}_{s,k+1,i}+\epsilon}}\right].
\end{align*}
Since $\tilde{v}_{s,k+1,i}$ is non-decreasing in $k$ for fixed $s$
(its definition replaces a longer prefix of stochastic gradients by their
conditional means), $\sqrt{\tilde{v}_{s,k+1,i}+\epsilon} \ge \sqrt{\tilde{v}_{s,1,i}+\epsilon}$,
which gives $1/\sqrt{\tilde{v}_{s,k+1,i}+\epsilon} \le 1/\sqrt{\tilde{v}_{s,1,i}+\epsilon}$.
Hence the last double sum is at most
$\sum_s\E[\bar{g}_{s,i}^2/\sqrt{\tilde{v}_{s,1,i}+\epsilon}]\cdot\sum_{k\ge 0}\beta_1^k = \frac{1}{1-\beta_1}\sum_s\E[\bar{g}_{s,i}^2/\sqrt{\tilde{v}_{s,1,i}+\epsilon}]$,
giving the lemma.
\end{proof}

\subsubsection{Proof of \texorpdfstring{\cref{thm:convergence_mom}}{Theorem 4}}

\begin{proof}
\par\smallskip\noindent\textbf{(1)}\quad
Following the procedure of \cref{eq:lowerbound},
\begin{equation}\label{eq:taylor_rearranged}
\eta\,\E\!\sum_{t,i} \underbrace{\frac{\bar g_{t-\tau,i}\,m_{t-\tau,i}}{\sqrt{v_{t-\tau,i}+\epsilon}}}_{(A)_{t,i}}
\;\le\; \Delta_0 + \eta \Big|\E\!\sum_{t,i} \underbrace{\frac{(\bar g_{t-\tau,i}-\bar g_{t,i})\,m_{t-\tau,i}}{\sqrt{v_{t-\tau,i}+\epsilon}}}_{(B)_{t,i}}\Big|
+ \tfrac{\eta^{2}}{2}\,\E\!\sum_{t,i} H_i \underbrace{\frac{m_{t-\tau,i}^{2}}{v_{t-\tau,i}+\epsilon}}_{(C)_{t,i}}.
\end{equation}
Applying \cref{lemma:hybrid-descent} per coordinate and summing over $i$:
\begin{equation}\label{eq:A_lower}
\eta\,\E\!\sum_{t,i} (A)_{t,i}
\;\ge\; \tfrac{(1-\beta_1)\eta}{4}\,\mathcal S - \eta R_H - \eta R_g,
\end{equation}
where $\mathcal S \triangleq \E\sum_{s=1}^{T}\sum_{i=1}^d\sum_{k=0}^{s-1}\beta_1^{k}\,\bar g_{s-k,i}^{2}/\sqrt{\tilde v_{s,k+1,i}+\epsilon}$ and
\[
R_H \triangleq \frac{\beta_1\eta\sum_i H_i}{(1-\beta_1)^{2}}\,\E\!\sum_{w=1}^{T}\!\Big\|\tfrac{m_w}{\sqrt{v_w+\epsilon}}\Big\|_2^{2},
\quad
R_g \triangleq \frac{\sqrt{1-\beta_1}}{(1-\beta_1/\beta_2)^{3/2}}\sum_{i=1}^d\!\big[4G\sqrt{1-\beta_2}+H_i\eta\big]\,\E\!\sum_{u=1}^T\tfrac{g_{u,i}^{2}}{v_{u,i}+\epsilon}.
\]
For $(B)$, by Cauchy--Schwarz applied coordinatewise,
$|\sum_i (B)_{t,i}| \le A_t\,B_t$ where
\[
A_t \triangleq \sqrt{\textstyle\sum_i (\bar g_{t-\tau,i}-\bar g_{t,i})^{2}},
\qquad
B_t \triangleq \sqrt{\textstyle\sum_i m_{t-\tau,i}^{2}/(v_{t-\tau,i}+\epsilon)}.
\]
By smoothness,
$(\bar g_{t,i}-\bar g_{t-\tau,i})^{2} \le H_i^{2}\|x_{t-1}-x_{t-\tau-1}\|_\infty^{2} \le H_i^{2}\|x_{t-1}-x_{t-\tau-1}\|_2^{2}$,
and since the algorithm step is now
$x_{t}-x_{t-1} = -\eta\,m_{t-\tau}/\sqrt{v_{t-\tau}+\epsilon}$,
\[
\|x_{t-1}-x_{t-\tau-1}\|_2^{2}
\;=\; \eta^{2}\Big\|\sum_{k=1}^{\tau}\tfrac{m_{t-k-\tau}}{\sqrt{v_{t-k-\tau}+\epsilon}}\Big\|_2^{2}
\;\le\; \tau\eta^{2}\sum_{k=1}^{\tau}\Big\|\tfrac{m_{t-k-\tau}}{\sqrt{v_{t-k-\tau}+\epsilon}}\Big\|_2^{2}.
\]
Using $\sum_{t,k}\|m_{t-k-\tau}/\sqrt{v_{t-k-\tau}+\epsilon}\|_2^{2}\le\tau\sum_w\|m_w/\sqrt{v_w+\epsilon}\|_2^{2}$
yields $\sum_t A_t^{2}\le \tau^{2}\eta^{2}\sum_i H_i^{2}\sum_w\|m_w/\sqrt{v_w+\epsilon}\|_2^{2}$,
while $\sum_t B_t^{2}=\sum_w\|m_w/\sqrt{v_w+\epsilon}\|_2^{2}$ after re-indexing. Cauchy--Schwarz over $t$:
\[
\eta\,\Big|\E\!\sum_{t,i}(B)_{t,i}\Big|
\;\le\; \eta\,\E\sqrt{\textstyle\sum_t A_t^{2}}\sqrt{\textstyle\sum_t B_t^{2}}
\;\le\; \tau\eta^{2}\sqrt{\textstyle\sum_i H_i^{2}}\,\E\!\sum_w\!\Big\|\tfrac{m_w}{\sqrt{v_w+\epsilon}}\Big\|_2^{2}.
\]
Applying \cref{lemma:mom-variance} to convert
$\|m_w/\sqrt{v_w+\epsilon}\|_2^{2}\to\|g_w/\sqrt{v_w+\epsilon}\|_2^{2}$
at the cost of $(1-\beta_1)/(1-\beta_1/\beta_2)$:
\begin{equation}\label{eq:B_bound_mom}
\eta\,\Big|\E\!\sum_{t,i} (B)_{t,i}\Big|
\;\le\; \frac{(1-\beta_1)\,\tau\eta^{2}\sqrt{\sum_i H_i^{2}}}{1-\beta_1/\beta_2}\,
\E\!\sum_{t,i}\tfrac{g_{t,i}^{2}}{v_{t,i}+\epsilon}.
\end{equation}
For $(C)$, re-indexing $w = t-\tau$ and applying
\cref{lemma:mom-variance,lemma:accumulated-squared-update} per coordinate:
\begin{equation}\label{eq:C_repeat}
\tfrac{\eta^{2}}{2}\,\E\!\sum_{t,i} H_i\,(C)_{t,i}
\le \frac{(1-\beta_1)\eta^{2}}{2(1-\beta_1/\beta_2)}\sum_i H_i
\Big(T+\tfrac{\beta_2}{1-\beta_2}\,\E\ln\tfrac{v_{T,i}+\epsilon}{v_{0,i}+\epsilon}\Big).
\end{equation}

\par\smallskip\noindent\textbf{(2)}\quad
By \cref{lemma:accumulated-squared-update}, $\sum_w g_{w,i}^{2}/(v_{w,i}+\epsilon)\le T+\beta_2 F_i/(1-\beta_2)$ for each $i$ with $F_i \triangleq \ln((v_{T,i}+\epsilon)/(v_{0,i}+\epsilon))$, and \cref{lemma:c1} gives $\E\sum_i F_i\le dF$ with $F=\mathcal O(\log T)$. Setting $K_T \triangleq 1+\beta_2 F/(T(1-\beta_2)) = \mathcal O(1)$,
\begin{equation}\label{eq:c1_unweighted}
\E\!\sum_{w,i}\tfrac{g_{w,i}^{2}}{v_{w,i}+\epsilon}\;\le\; dT K_T,
\qquad
\E\!\sum_i H_i\sum_w\tfrac{g_{w,i}^{2}}{v_{w,i}+\epsilon}\;\le\; T K_T \sum_i H_i.
\end{equation}
Applying \eqref{eq:c1_unweighted} to the right-hand sides of \cref{eq:A_lower,eq:B_bound_mom,eq:C_repeat}:
\begin{align}
\eta R_H \;&\le\; \frac{\beta_1 d\eta^{2}TK_T\sum_i H_i}{(1-\beta_1)(1-\beta_1/\beta_2)}, \label{eq:R_H_summary}\\
\eta R_g \;&\le\; \frac{4Gd\eta T K_T\sqrt{(1-\beta_1)(1-\beta_2)}}{(1-\beta_1/\beta_2)^{3/2}}+\frac{\eta^{2}\sqrt{1-\beta_1}\,T K_T\sum_i H_i}{(1-\beta_1/\beta_2)^{3/2}},\label{eq:R_g_summary}\\
\eta\Big|\E\!\sum_{t,i}(B)_{t,i}\Big| \;&\le\; \frac{(1-\beta_1)d\tau\eta^{2}TK_T\sum_i H_i}{1-\beta_1/\beta_2}, \label{eq:B_final}\\
\tfrac{\eta^{2}}{2}\E\!\sum_{t,i} H_i(C)_{t,i} \;&\le\; \frac{(1-\beta_1)\eta^{2}TK_T\sum_i H_i}{2(1-\beta_1/\beta_2)}. \label{eq:C_final}
\end{align}
Substituting \cref{eq:R_H_summary,eq:R_g_summary,eq:B_final,eq:C_final} into \eqref{eq:taylor_rearranged}:
\begin{equation}\label{eq:E_beta1_explicit}
\frac{(1-\beta_1)\mathcal S}{T}
\;\le\;
\mathcal E_{\beta_1}
\;\triangleq\;
\tfrac{4\Delta_0}{\eta T} + K_T\!\Big[
\tfrac{c_3(1-\beta_1)\eta(1+d\tau)\sum_i H_i}{1-\beta_1/\beta_2} +\tfrac{c_4\,\beta_1 d\eta\sum_i H_i}{(1-\beta_1)(1-\beta_1/\beta_2)}
+\tfrac{c_5 Gd\sqrt{(1-\beta_1)(1-\beta_2)}}{(1-\beta_1/\beta_2)^{3/2}}
\Big],
\end{equation}
where $c_3, c_4, c_5 > 0$ are absolute constants.

\par\smallskip\noindent\textbf{(3)}\quad
Let $\mathcal E' \triangleq \tfrac{2}{T}\,\E\!\sum_{s>T/2}\!\sum_{i,k}\beta_1^{k}\sqrt{\tilde v_{s,k+1,i}+\epsilon}$
and $\bar w \triangleq \sum_{j=T/2+1}^T(1-\beta_1^{T-j+1})/(1-\beta_1)$. Re-indexing
$j=s-k$ gives $\sum_{s>T/2,k}\beta_1^k|\bar g_{s-k,i}|\ge \sum_{j>T/2}w_j|\bar g_{j,i}|$
with $w_j=(1-\beta_1^{T-j+1})/(1-\beta_1)$. By min-vs-weighted-average and
Cauchy--Schwarz over $(i,s,k)$ ($\beta_1^k=\beta_1^{k/2}\!\cdot\!\beta_1^{k/2}$):
\begin{align}
\min_{T/2<t\le T}\E\|\nabla f(x_t)\|_1
&\le \tfrac{1}{\bar w}\,\E\!\sum_i\!\sum_{s>T/2}\!\sum_k\beta_1^{k}|\bar g_{s-k,i}|
\le \tfrac{1}{\bar w}\sqrt{\mathcal S\cdot T\mathcal E'/2}\nonumber\\
&\le 2\sqrt 2\,\sqrt{(1-\beta_1)\mathcal E_{\beta_1}\,\mathcal E'}, \label{eq:min_clean}
\end{align}
where the last step uses
$\bar w = \tfrac{T}{2(1-\beta_1)}-\tfrac{\beta_1(1-\beta_1^{T/2})}{(1-\beta_1)^{2}}\ge \tfrac{T}{4(1-\beta_1)}$
(holds for $T\ge T_0=4/(1-\beta_1)$) and $(1-\beta_1)\mathcal S/T\le\mathcal E_{\beta_1}$
from \eqref{eq:E_beta1_explicit}. By \cref{lemma:mom-first-order-approx},
\begin{equation}\label{eq:E_prime_explicit}
\mathcal E' \;\le\; \mathcal O\!\big(\tfrac{\sum_i\sigma_i}{1-\beta_1} + \tfrac{\mathcal E_{\beta_1}}{(1-\beta_1)^{2}} + \tfrac{1}{T(1-\beta_1)}\big).
\end{equation}
Substituting \eqref{eq:E_prime_explicit} into \eqref{eq:min_clean} and using $\sqrt{a+b+c}\le\sqrt a+\sqrt b+\sqrt c$:
\begin{equation}\label{eq:rate_template}
\min_{T/2<t\le T}\E\|\nabla f(x_t)\|_1
\;\lesssim\;
\sqrt{\mathcal E_{\beta_1}\textstyle\sum_i\sigma_i}+\tfrac{\mathcal E_{\beta_1}}{\sqrt{1-\beta_1}}+\sqrt{\mathcal E_{\beta_1}/T}.
\end{equation}
Let $D_\beta\triangleq (1-\beta_1)(1+d\tau)+\beta_1 d/(1-\beta_1)=\Theta(d(1+\tau))$. With
$\eta=\Theta(\sqrt{\Delta_0/(TD_\beta\sum_i H_i)})$ and $1-\beta_2=\Theta(\log T/T)$,
\eqref{eq:E_beta1_explicit} balances to
$\mathcal E_{\beta_1}=\mathcal O(\sqrt{d(1+\tau)R/T}+Gd\sqrt{\log T/T})$.
Substituting into \eqref{eq:rate_template} yields the theorem rate.
\end{proof}

%% file: tex/904_proof_approximation_comparison.tex
\section{Proof of Basis Misalignment Analysis}
\label{app:approximation-comparison-proof}
In this section, we provide the proof of the Hessian approximation comparison in \cref{subsec:solution:Hessian_approximation}. Before we begin the proof, we present useful lemmas regarding the Kronecker product.

\begin{lemma}\label[lemma]{lemma:abvecG_vecBGAt}
(\citet{Thevec-permutationmatrix})
Let $A,B,C$ be matrices of appropriate dimensions. Then the following holds:
\begin{align*}
    \Vector{ABC} = (C^\top \otimes A)\Vector{B}
\end{align*}
\end{lemma}

\begin{lemma}\label[lemma]{lemma:kronecker_basic}
(\citet{Horn_Johnson_1991})
Let $A,B,A',B'$ be matrices of appropriate dimensions. Then the followings hold:
\begin{enumerate}
    \item $(A \otimes B)(A' \otimes B')=(AA')\otimes(BB')$
    \item $(A\otimes B)^\top = (A^\top\otimes B^\top)$
\end{enumerate}
\end{lemma}

\begin{lemma}\label[lemma]{lemma:kronecker_11norm}
\begin{align*}
    ||A \otimes B||_{(1,1)} = ||A||_{(1,1)}||B||_{(1,1)}
\end{align*}
\end{lemma}
\begin{proof}
$||A \otimes B||_{(1,1)} = \sum_{i,j,k,l} |A_{ik} B_{jl}| = \left(\sum_{ik} |A_{ik}|\right) \left(\sum_{jl} |B_{jl}|\right) = ||A||_{(1,1)}||B||_{(1,1)}.$
\end{proof}

\subsection{\texorpdfstring{Proof of \cref{thm:covariance_Hessian_approximation}}{Proof of Theorem 3.1}}

Let us restate the theorem here for the sake of readability.

\covarianceHessianapproximation*

In this proof, we rely on the key property of the Hessian with exact Kronecker product form.

\begin{lemma}\label{lemma:exact_covariance_Hessian}
    if $\mathbb{E}[gg^\top] = A \otimes B$ for some $A\in \R^{n\times n}$ and $B\in \R^{m\times m}$, then 
    \[
    \mathbb{E}[gg^\top] = c \cdot \mathbb{E}[G^\top G] \otimes \mathbb{E}[GG^\top] 
    \]
    for some scalar $c$.
\end{lemma}
\begin{proof}
By the definition of Kronecker product, $\forall i,j,k,l:\mathbb{E}[G_{ij} G_{kl}] = A_{jl} B_{ik}$ holds. Then each entry of the matrices $\mathbb{E}[G G^\top]$ and  $\mathbb{E}[ G^\top G]$ satisfies
\begin{align*}
    (\mathbb{E}[G G^\top])_{ik} = \sum_{j} \mathbb{E}[G_{ij} G_{kj}] = \sum_{j} A_{jj} B_{ik} = \Tr(A)\cdot B_{ik}\\
    (\mathbb{E}[G^\top G])_{jl} = \sum_{i} \mathbb{E}[G_{ij} G_{il}] = \sum_{i} A_{jl} B_{ii} = A_{jl} \cdot \Tr(B).
\end{align*}
Thus, $\mathbb{E}[G G^\top]$ and $\mathbb{E}[G^\top G]$ are scalar multiplications of $B$ and $A$, respectively.
\end{proof}

Now, we are ready to prove \cref{thm:covariance_Hessian_approximation}.

\begin{proof}
Since Hessian admits a Kroneker factorized empirical Fisher, $H=\mathbb{E}[g g^\top] = A \otimes B$ holds with some $A \in \mathbb{R}^{n\times n}$ and $B \in \mathbb{R}^{m\times m}$.
From \cref{lemma:exact_covariance_Hessian}, and the definitions of $U$ and $V$ in the theorem, we can write the eigendecompositions as:
\[
A = V \Lambda_A V^\top, \quad B = U \Lambda_B U^\top,
\]
where $\Lambda_A$ and $\Lambda_B$ are diagonal matrices.

By \cref{lemma:abvecG_vecBGAt}, rotated parameterizations corresponds to $\tilde{w} = (I \otimes U^\top ) w$ and $\tilde{w} = (V^\top \otimes U^\top) w$.

Then Hessian in the rotated space is
\begin{align}\label{eq:H1_covariance_kronecker}
    H_{U} = (I \otimes U^\top) (A \otimes B) (I \otimes U) = A \otimes (U^\top B U) = A \otimes \Lambda_B,
\end{align}
\begin{align}\label{eq:H2_covariance_kronecker}
    H_{U,V} = (V^\top \otimes U^\top) (A \otimes B) (V \otimes U) = (V^\top A V) \otimes (U^\top B U) =  \Lambda_A \otimes \Lambda_B,
\end{align}

where in the second equalities, we use \cref{lemma:kronecker_basic} and in the last equalities, we used \cref{lemma:exact_covariance_Hessian} in both (\ref{eq:H2_covariance_kronecker}) and (\ref{eq:H1_covariance_kronecker}).
For any symmetric matrix, the $(1,1)$-norm is minimized when the matrix is diagonalized by its eigenbasis. Therefore:
\[
    ||\Lambda_A||_{(1,1)} \le ||A||_{(1,1)} \quad \text{and} \quad ||\Lambda_B||_{(1,1)} \le ||B||_{(1,1)}.
\]
Thus, \cref{lemma:kronecker_11norm} concludes the theorem.
\end{proof}


\subsection{\texorpdfstring{Basis misalignment Analysis for $\gS=$ \mean{}}{Basis misalignment Analysis for S=1st}}
Here, we present a result for the $\gS=$ \mean{} strategy analogous to \cref{thm:covariance_Hessian_approximation}, but under a different assumption regarding the Hessian's structure.
Under the assumption $H=\E[g]\E[g]^\top$, we can show the following.
\begin{theorem}\label{thm:gradient_Hessian_approximation}
Let orthogonal matrices $U, V$ be $\mathbb{E}[G] = U \Sigma V^\top$. If Hessian has an exact Kronecker product form, the following inequalities hold:
\[
||H_{U,V}||_{(1,1)} \le ||H_{U}||_{(1,1)} \le ||H||_{(1,1)}.
\]
Moreover, $\|H_{U,V}\|_{(1,1)}$ attains the global minimum over all orthogonal rotations.
\end{theorem}
However, in practice, $\mathbb{E}[g]\mathbb{E}[g]^\top$ is not guaranteed to be a faithful approximation of the true Hessian. This discrepancy potentially limits the approximation fidelity and, consequently, the optimizer's robustness to gradient delay in empirical settings.

\begin{proof}
    Recall that $H=\mathbb{E}[g]\mathbb{E}[g]^\top$. The rotated Hessians are defined as:
    \begin{align}
        H_{U} &= (I \otimes U^\top) \mathbb{E}[g] \left((I \otimes U^\top) \mathbb{E}[g]\right)^\top, \label{eq:H3_covariance_kronecker} \\
        H_{U,V} &= (V^\top \otimes U^\top) \mathbb{E}[g] \left((V^\top \otimes U^\top) \mathbb{E}[g]\right)^\top. \label{eq:H4_covariance_kronecker}
    \end{align}
    For any rank-1 matrix $M = zz^\top$, $(1,1)$-norm satisfies $\|M\|_{(1,1)} = \|zz^\top\|_{(1,1)} = \|z\|_1^2$.
    By \cref{lemma:abvecG_vecBGAt}, $(I \otimes U^\top) \mathbb{E}[g] = \Vector{\Sigma V^\top}$, $(V \otimes U^\top) \mathbb{E}[g] = \Vector{\Sigma}$ holds.
    Thus, proving the theorem is equivalent to showing $\|\Sigma\|_{(1,1)} \le \|\Sigma V^\top\|_{(1,1)} \le \|U \Sigma V^\top\|_{(1,1)}$.

    Since $H$ is rank-1 and we assume it admits an exact Kronecker product structure, $\mathbb{E}[G]$ must also be rank-1. Let $\mathbb{E}[G] = \sigma_1 u_1 v_1^\top$, where $u_1$ and $v_1$ are the first columns of $U$ and $V$, respectively, and $\sigma_1$ is the singular value.
    
    Since $U$ and $V$ are orthogonal, $\|u_1\|_2 = \|v_1\|_2 = 1$. By the norm inequality $\|x\|_1 \ge \|x\|_2$, we have $\|u_1\|_1 \ge 1$ and $\|v_1\|_1 \ge 1$.
    Therefore, the following inequalities hold:
    \begin{align*}
        \|\Sigma\|_{(1,1)} &= \sigma_1 \\
        &\le \sigma_1 \|v_1\|_1 = \|\Sigma V^\top\|_{(1,1)} \\
        &\le \sigma_1 \|v_1\|_1 \|u_1\|_1 = \|u_1 \sigma_1 v_1^\top\|_{(1,1)} = \|U \Sigma V^\top\|_{(1,1)}.
    \end{align*}
    This implies $\|H_{U,V}\|_{(1,1)} \le \|H_U\|_{(1,1)} \le \|H\|_{(1,1)}$. Furthermore, the diagonal form $\|H_{U,V}\|_{(1,1)}$ attains the global minimum.
\end{proof}

\newpage

%% file: tex/906_connection_recent_optimizers.tex
\section{Connection to Recent Optimizers}
\label{app:connection_recent_optimizers}

In this section, we establish a connection between the design axes of \estimation{} and existing modern optimizers to contextualize our framework within the current literature.
While our methodology shares conceptual roots with several recent optimizers, we introduce \estimation{} as a unified abstraction designed to isolate the effects of Hessian geometry from other implementation variables, such as momentum accumulation space or the use of power iteration instead of exact Singular Value Decomposition (SVD).

Our \covariance{} strategy is closely related to SOAP \citep{vyas2024soap} and EShampoo \citep{eschenhagen2025purifying}.
However, a key technical distinction lies in the accumulation space of the optimizer states and the timing of the basis update. 
Unlike the official implementation of SOAP, which accumulates the first momentum in a rotated space and updates the eigenbasis after the parameter update step, our \covariance{} approach accumulates the first momentum in the original space and refreshes the basis before the optimization step.
While accumulation in the rotated space requires projecting the momentum back and forth between bases during updates, accumulation in the original space does not require this additional computation.
Furthermore, while EShampoo relies on SVD to compute eigenvectors, our implementation uses a single step of power iteration and QR decomposition to approximate them.

Similarly, our \mean{} strategy shares commonalities with GaLore \citep{zhao2024galore}, LDAdam \citep{robert2024ldadam} (for \onesided{}), and the SVD-rotated AdamW variant discussed in \citet{zhang2024understanding} (for \twosided{}).
The primary departure from these methods is our use of the momentum matrix $M_t$ as the source for basis estimation, whereas GaLore and AdamW-SVD typically derive their rotation matrices from the instantaneous, noisy gradient $G_t$.
Furthermore, unlike GaLore and AdamW-SVD—which accumulate momentum in a rotated space and rely on exact SVD—our approach maintains accumulation in the original space and utilizes power iteration.
Finally, because \estimation{} remains strictly full-rank, it avoids the complex error-buffer mechanisms found in LDAdam.

By standardizing the optimizer's internal mechanics, we can systematically analyze how different basis estimation strategies interact with the loss landscape and gradient delay without being confounded by the implementation details of individual optimizers.

%% file: tex/912_memory_overhead.tex

\section{Memory Overhead Analysis}
\label{app:memory_overhead}

In this section, we analyze the memory overhead of basis rotation in a
realistic training setup. As described in
Section~\ref{subsec:solution:Hessian_approximation}, our \texttt{eigenbasis-estimation}
framework spans two design axes, the approximation source
$\mathcal{S} \in \{1\text{st}, 2\text{nd}\}$ and the rotation geometry
$\mathcal{G} \in \{\text{Unilateral}, \text{Bilateral}\}$, yielding four
concrete strategies. We quantify the memory cost of each strategy both in
terms of theoretical complexity and concrete GB usage on a realistic LLM
training setup.

For a single weight matrix $W \in \mathbb{R}^{m \times n}$, basis rotation
introduces two sources of additional memory beyond the standard Adam states:
(i) the rotation matrices used to diagonalize the Hessian approximation, and
(ii) the second-moment statistics required when $\mathcal{S} = 2\text{nd}$.
For the bilateral strategy ($\mathcal{G} = \text{Bi}$), both
$U \in \mathbb{R}^{m \times m}$ and $V \in \mathbb{R}^{n \times n}$ are
stored, incurring $m^2 + n^2$ parameters; for the unilateral strategy
($\mathcal{G} = \text{Uni}$), only the rotation along the smaller dimension
is stored, reducing the cost to $\min(m,n)^2$. When $\mathcal{S} =
2\text{nd}$, the empirical Fisher factors
$L = G G^\top \in \mathbb{R}^{m \times m}$ and
$R = G^\top G \in \mathbb{R}^{n \times n}$ must be maintained as buffers,
adding the same cost as the rotation matrices. In contrast, $\mathcal{S} =
1\text{st}$ reuses the existing momentum buffer $M_t$ to estimate the
eigenbasis, introducing \emph{no} additional moment storage.

Table~\ref{tab:memory_overhead} reports both the theoretical per-matrix
overhead and concrete GB usage on Llama-3-8B, which has hidden dimension
$h = 4096$ and MLP intermediate dimension $4h_{\text{int}} = 14336$.
Attention projection matrices are square ($4096 \times 4096$), while MLP
projection matrices are rectangular ($4096 \times 14336$). All buffers are
stored in FP32 (4 bytes per element), consistent with standard
mixed-precision training where optimizer states are kept in full precision.

\begin{table}[h]
\centering
\small
\caption{Memory overhead of basis rotation strategies on Llama-3-8B
($h=4096$, $4h_{\text{int}}=14336$). Memory is reported in GB per weight
matrix, assuming FP32 storage. Attention matrices are
$4096 \times 4096$; MLP matrices are $4096 \times 14336$.}
\label{tab:memory_overhead}
\begin{tabular}{cccccc}
\toprule
$\mathcal{S}$ & $\mathcal{G}$ & Rotation & Moments &
Mem (Attn) & Mem (MLP) \\
\midrule
$2$nd & Bi  & $m^2 + n^2$       & $m^2 + n^2$ & 0.25 & 1.66 \\
$2$nd & Uni & $\min(m,n)^2$     & $\min(m,n)^2$ & 0.13 & 0.13 \\
$1$st & Bi  & $m^2 + n^2$       & --          & 0.13 & 0.83 \\
$1$st & Uni & $\min(m,n)^2$     & --          & 0.06 & 0.06 \\
\bottomrule
\end{tabular}
\end{table}

A few observations are in order. First, the bilateral variants
($\mathcal{G} = \text{Bi}$) incur substantially more memory than their
unilateral counterparts on rectangular MLP matrices, because they store
both $U \in \mathbb{R}^{m \times m}$ and $V \in \mathbb{R}^{n \times n}$
rather than a single rotation along the smaller dimension. Second,
switching from $\mathcal{S} = 2\text{nd}$ to $\mathcal{S} = 1\text{st}$
halves the overhead by eliminating the $L$ and $R$ buffers and reusing the
existing momentum matrix $M_t$ for eigenbasis estimation. Third, the most
memory-efficient strategy ($\mathcal{S} = 1\text{st},\, \mathcal{G} =
\text{Uni}$) requires only $\min(m,n)^2$ parameters per matrix---for an
MLP layer with $m = 4n$, this is approximately $7.5\%$ relative to Adam's
$4mn$ optimizer states ($m$- and $n$-shaped first and second moments
stored in FP32)---while still delivering more than a $40\%$ speedup over
the best-performing baseline (see
Figure~\ref{fig:result-rotation}).

%% file: tex/907_more_experiments.tex
\section{More Experimental Results}

\label{app:additional-exp-results}

\begin{figure*}[t]
    \centering
    \begin{subfigure}{0.2\linewidth}
        \includegraphics[width=\linewidth]{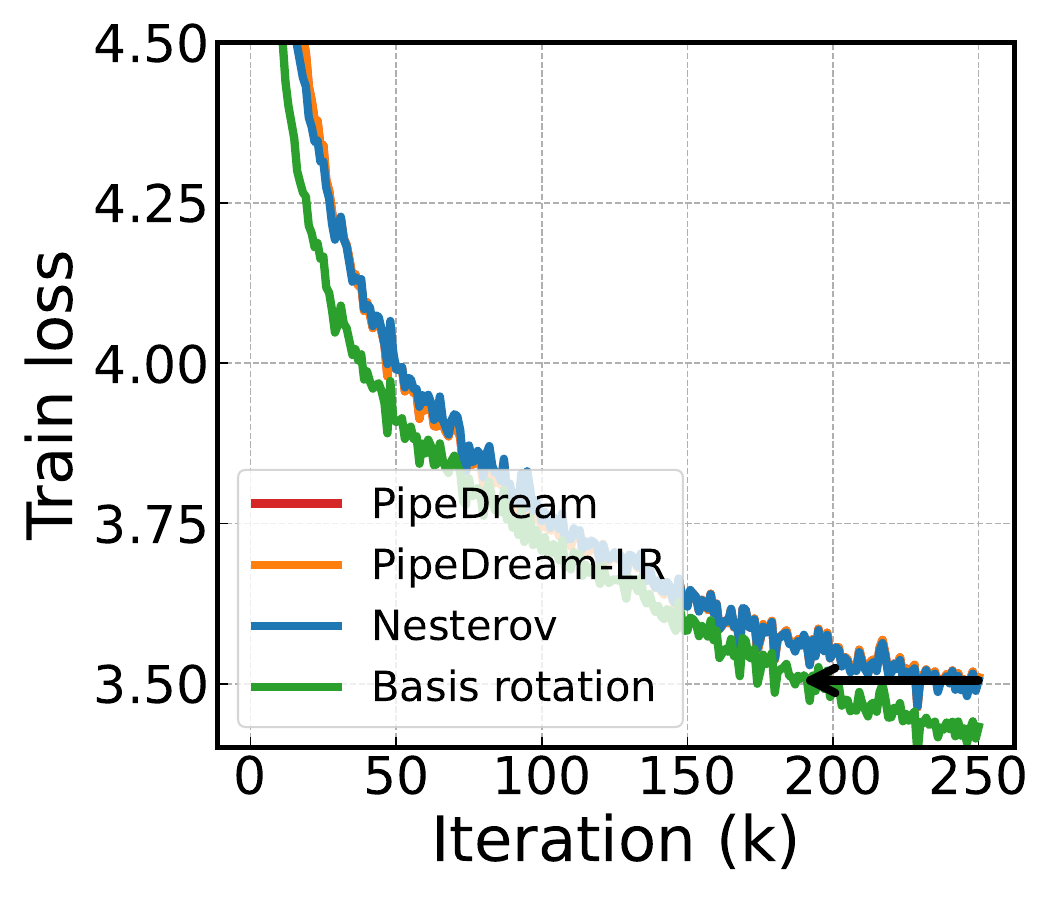}
        \caption{$P=1$}
        \label{fig:result-appendix-optmain-stage1}
    \end{subfigure}
    \hspace{2em}
    \begin{subfigure}{0.2\linewidth}
        \includegraphics[width=\linewidth]{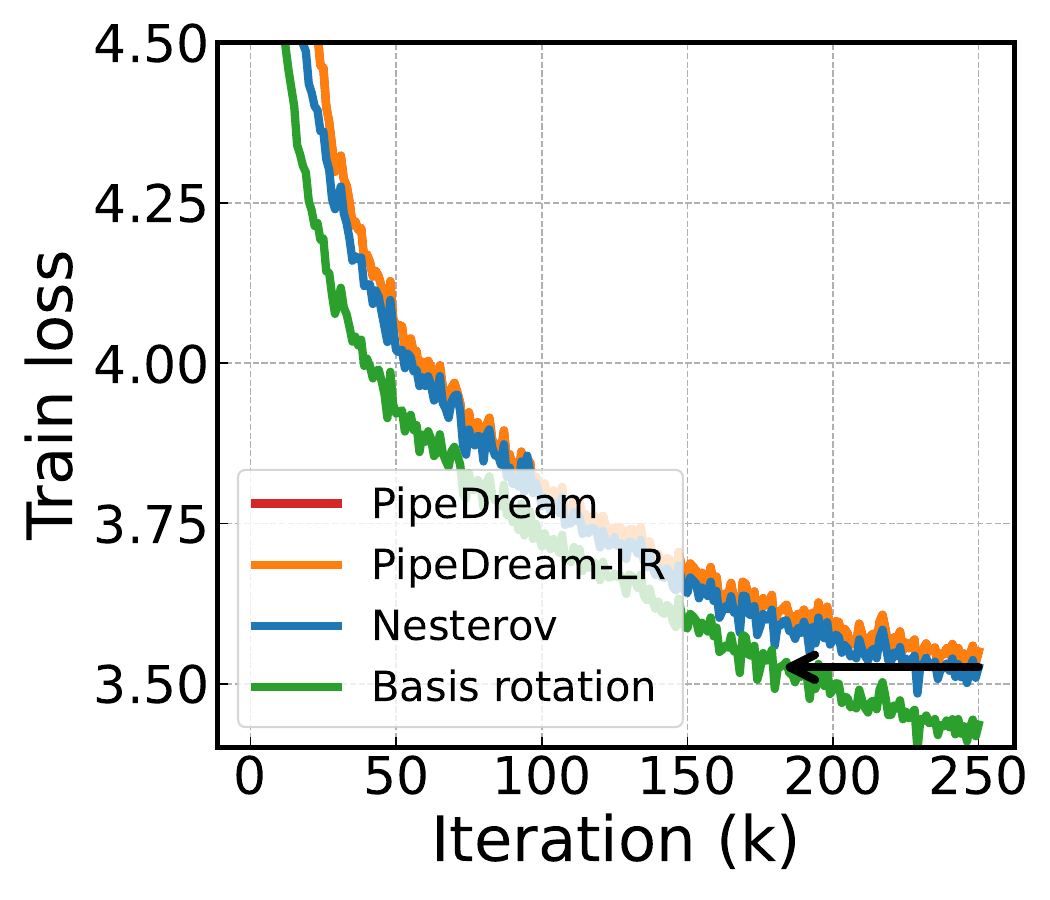}
        \caption{$P=2$}
        \label{fig:result-appendix-optmain-stage2}
    \end{subfigure}
    \hspace{2em}
    \begin{subfigure}{0.2\linewidth}
        \includegraphics[width=\linewidth]{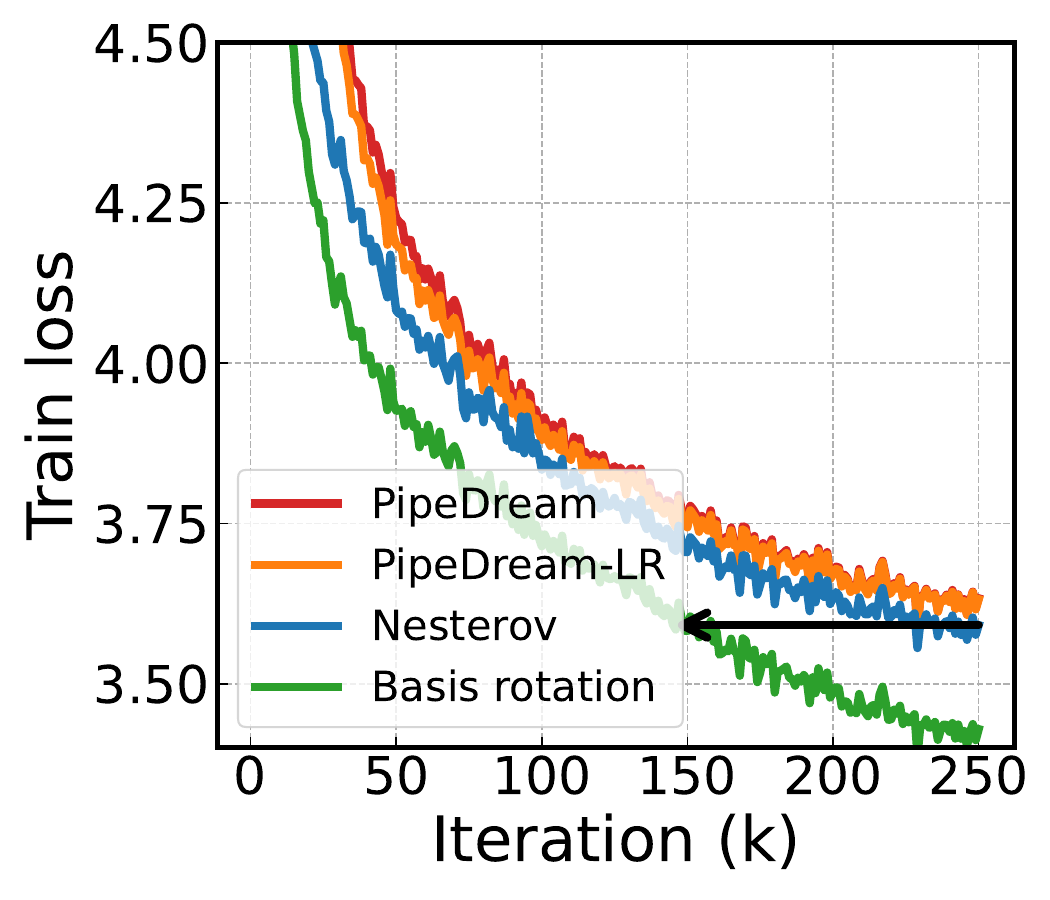}
        \caption{$P=4$}
        \label{fig:result-appendix-optmain-stage4}
    \end{subfigure} \\
    \begin{subfigure}{0.2\linewidth}
        \includegraphics[width=\linewidth]{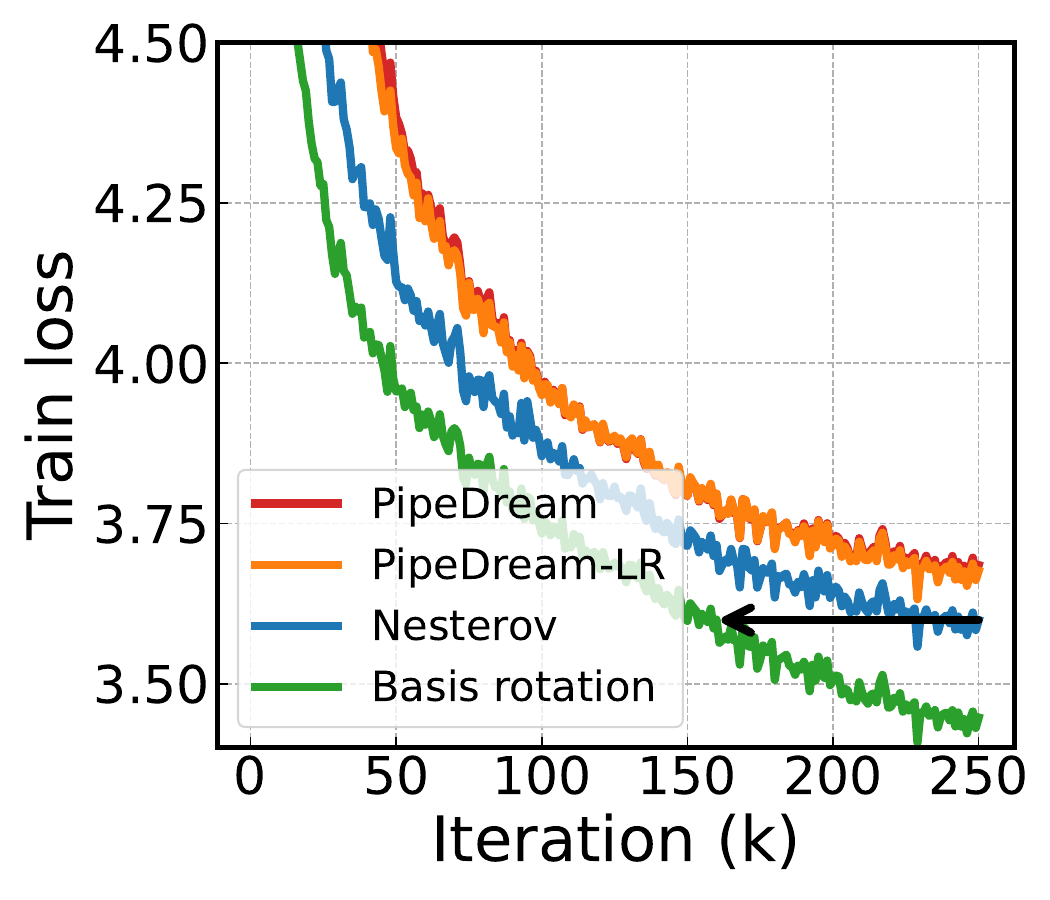}
        \caption{$P=8$}
        \label{fig:result-appendix-optmain-stage8}
    \end{subfigure}    
    \hspace{2em}
    \begin{subfigure}{0.2\linewidth}
        \includegraphics[width=\linewidth]{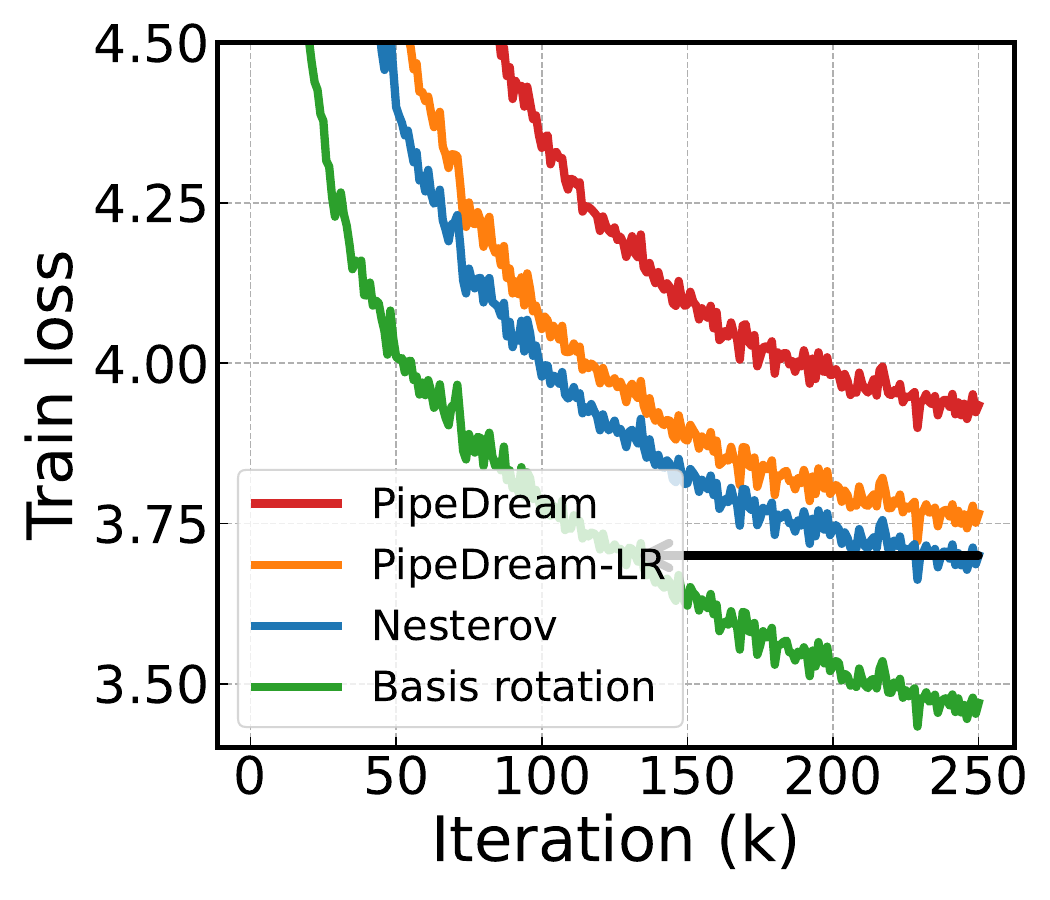}
        \caption{$P=16$}
        \label{fig:result-appendix-optmain-stage16}
    \end{subfigure}    
    \hspace{2em}
    \begin{subfigure}{0.2\linewidth}
        \includegraphics[width=\linewidth]{figure/main-exp/large-scale/optimizer/stage=32_main.pdf}
        \caption{$P=32$}
        \label{fig:result-appendix-optmain-stage32}
    \end{subfigure}    
    \caption{
    Comparison of each method for different number of stages $P$.
    The gap between \ourmethod{} and baselines increases with larger $P$.
    }
    \label{fig:result-appendix-optmain}
\end{figure*}

\begin{figure*}[t]
    \centering
    \begin{subfigure}{0.2\linewidth}
        \includegraphics[width=\linewidth]{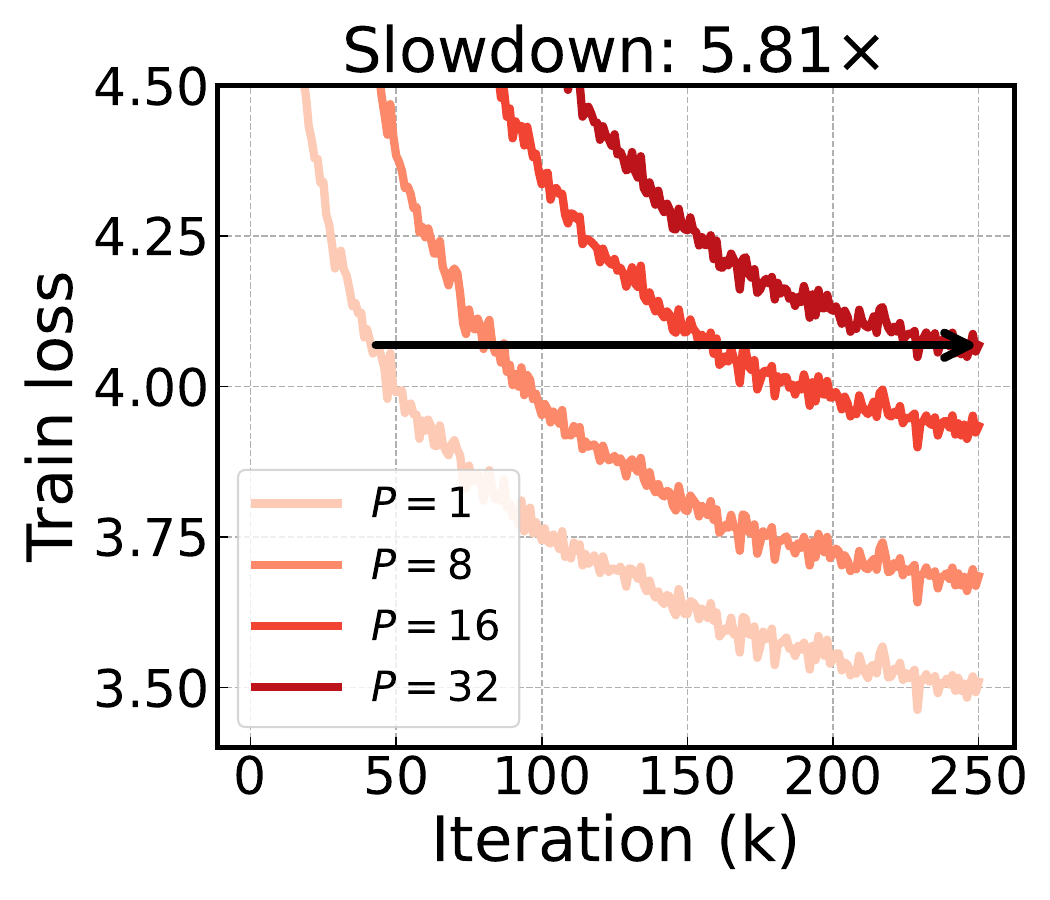}
        \caption{PipeDream}
        \label{fig:result-appendix-slowdown-adamw}
    \end{subfigure}
    \hspace{1em}
    \begin{subfigure}{0.2\linewidth}
        \includegraphics[width=\linewidth]{figure/main-exp/large-scale/slowdown/lradamw.pdf}
        \caption{PipeDream-LR}
        \label{fig:result-appendix-slowdown-lrdamw}
    \end{subfigure}
    \hspace{1em}
    \begin{subfigure}{0.2\linewidth}
        \includegraphics[width=\linewidth]{figure/main-exp/large-scale/slowdown/nadamw.pdf}
        \caption{Nesterov}
        \label{fig:rresult-appendix-slowdown-nadamw}
    \end{subfigure} \\
    \begin{subfigure}{0.2\linewidth}
        \includegraphics[width=\linewidth]{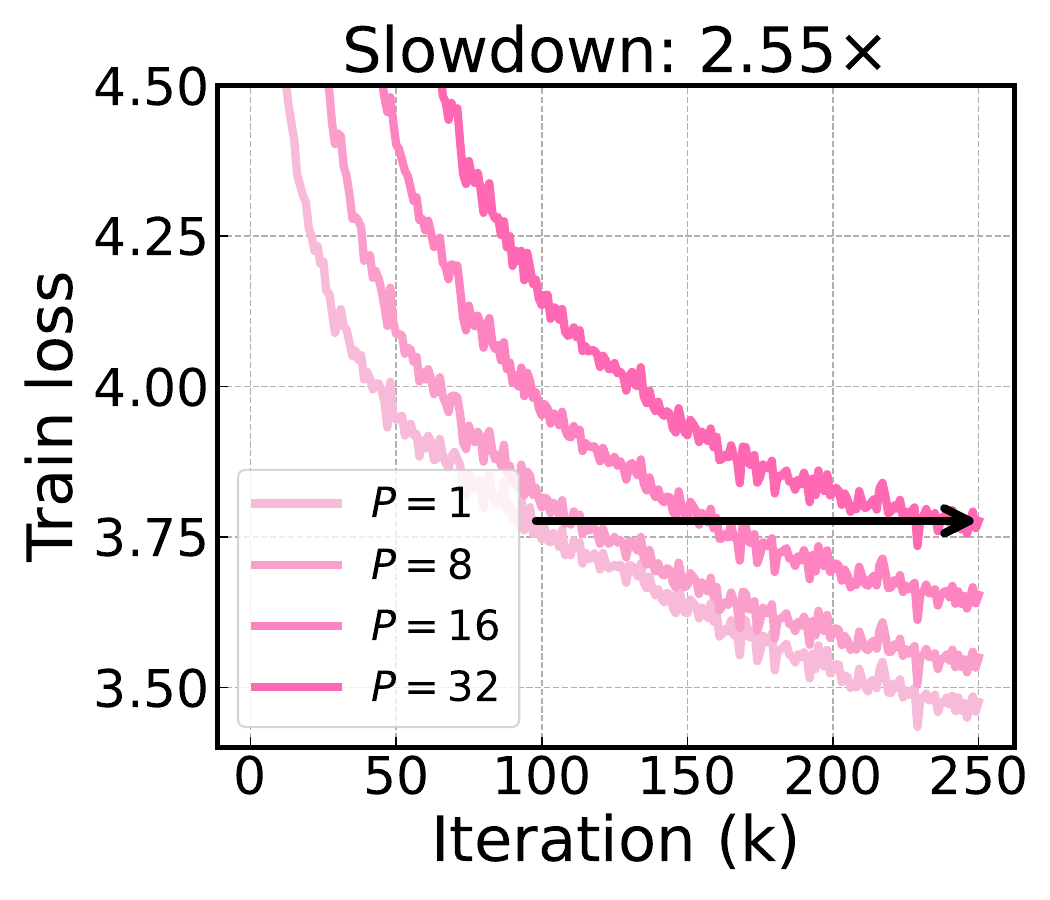}
        \caption{\mean{} / \onesided{}}
        \label{fig:result-appendix-slowdown-rotatedadam-onesided-mean}
    \end{subfigure}    
    \hspace{1em}
    \begin{subfigure}{0.2\linewidth}
        \includegraphics[width=\linewidth]{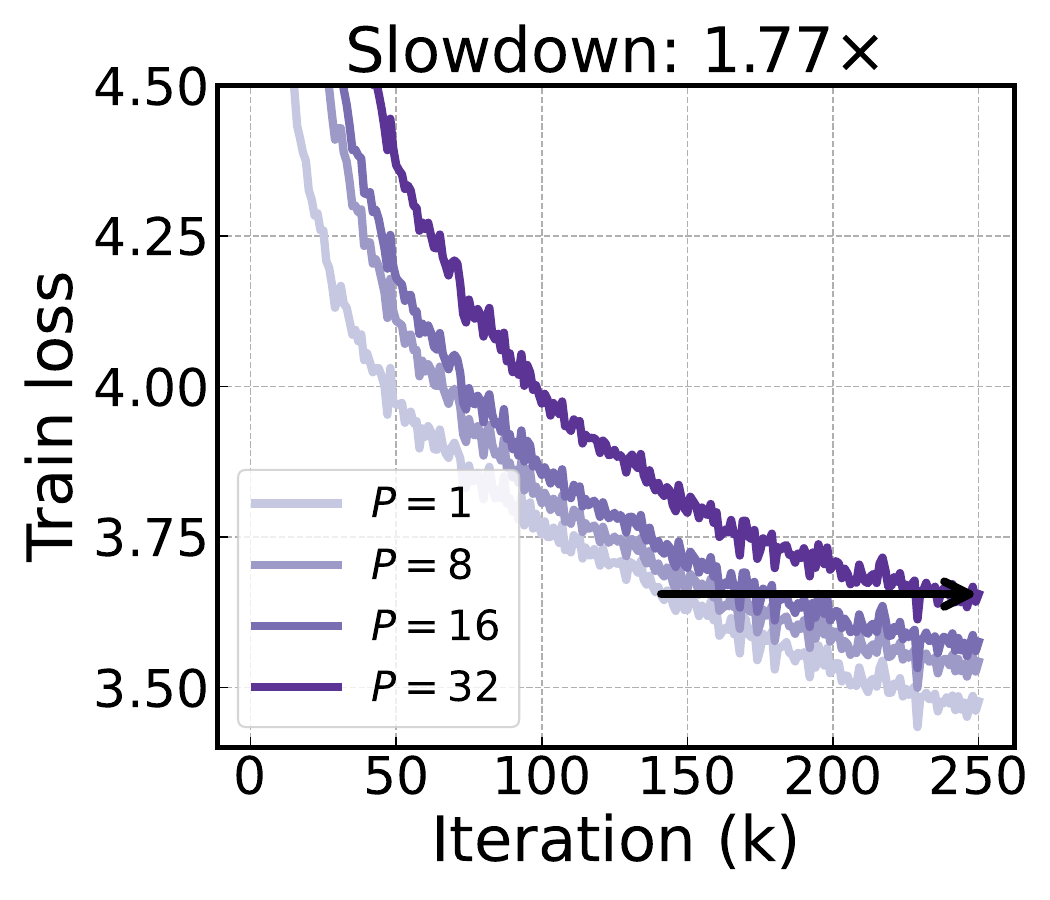}
        \caption{\mean{} / \twosided{}}
        \label{fig:result-appendix-slowdown-rotatedadam-twosided-mean}
    \end{subfigure}    
    \hspace{1em}
    \begin{subfigure}{0.2\linewidth}
        \includegraphics[width=\linewidth]{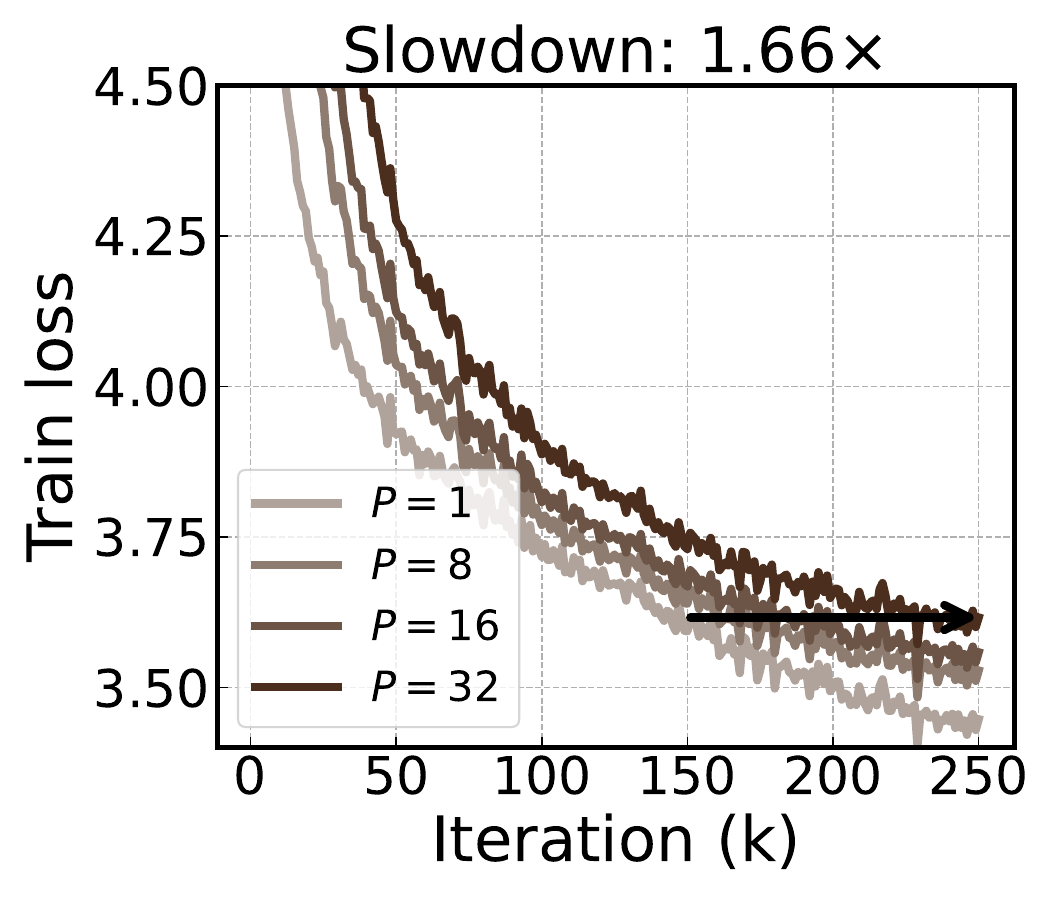}
        \caption{\covariance{} / \onesided{}}
        \label{fig:result-appendix-slowdown-rotatedadam-onesided-cov}
    \end{subfigure}  
    \hspace{1em}
    \begin{subfigure}{0.2\linewidth}
        \includegraphics[width=\linewidth]{figure/main-exp/large-scale/slowdown/rotatedadam_two_sided_covariance.pdf}
        \caption{\covariance{} / \twosided{}}
        \label{fig:result-appendix-slowdown-rotatedadam-twosided-cov}
    \end{subfigure}    
    \caption{
    Slowdown for different methods when increasing the number of stages $P$ for baselines (a-c) and \ourmethod{} (d-g).
    Here, slowdown is defined as the iteration ratio required to reach target loss for $P=32$ relative to $P=1$.
    }
    \label{fig:result-appendix-slowdown}
\end{figure*}

\begin{figure*}[t]
    \centering
    \begin{subfigure}{0.2\linewidth}
        \includegraphics[width=\linewidth]{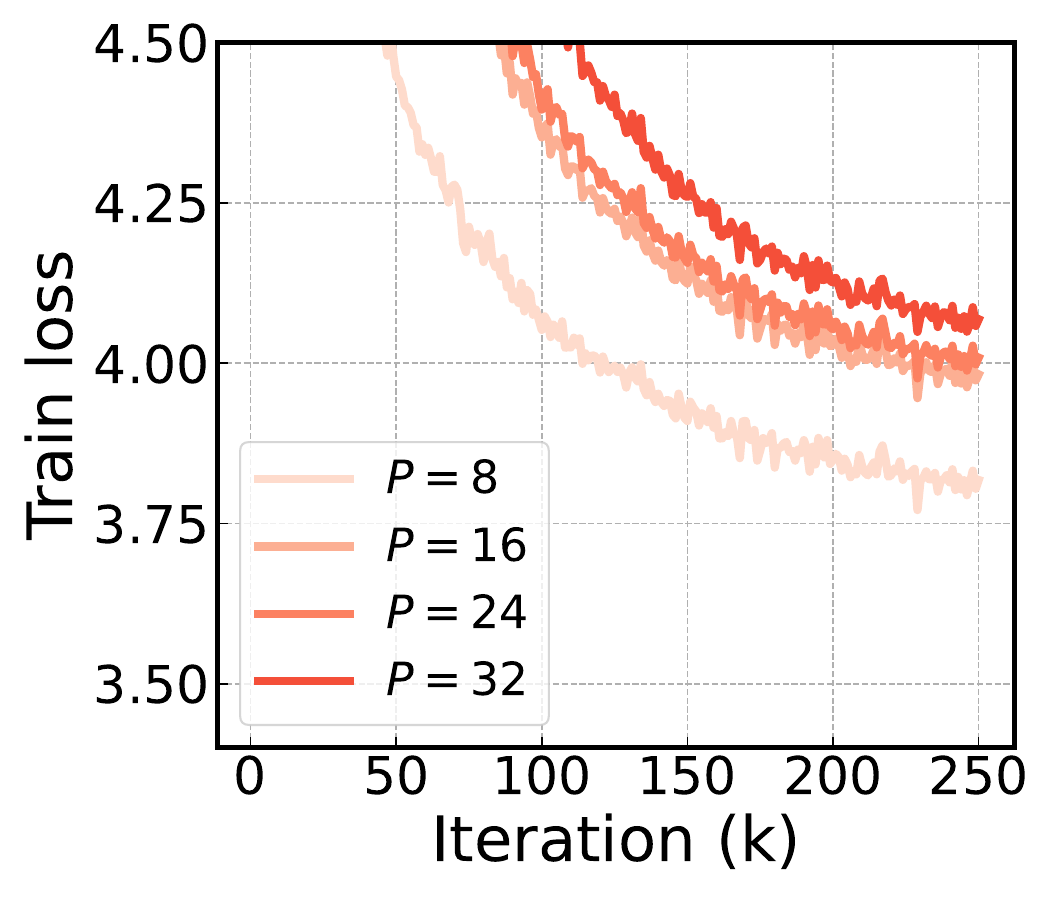}
        \caption{PipeDream}
        \label{fig:result-appendix-blockscaling-adamw}
    \end{subfigure}
    \hspace{1em}
    \begin{subfigure}{0.2\linewidth}
        \includegraphics[width=\linewidth]{figure/main-exp/large-scale/blockscaling/lradamw.pdf}
        \caption{PipeDream-LR}
        \label{fig:result-appendix-blockscaling-lrdamw}
    \end{subfigure}
    \hspace{1em}
    \begin{subfigure}{0.2\linewidth}
        \includegraphics[width=\linewidth]{figure/main-exp/large-scale/blockscaling/nadamw.pdf}
        \caption{Nesterov}
        \label{fig:rresult-appendix-blockscaling-nadamw}
    \end{subfigure} \\
    \begin{subfigure}{0.2\linewidth}
        \includegraphics[width=\linewidth]{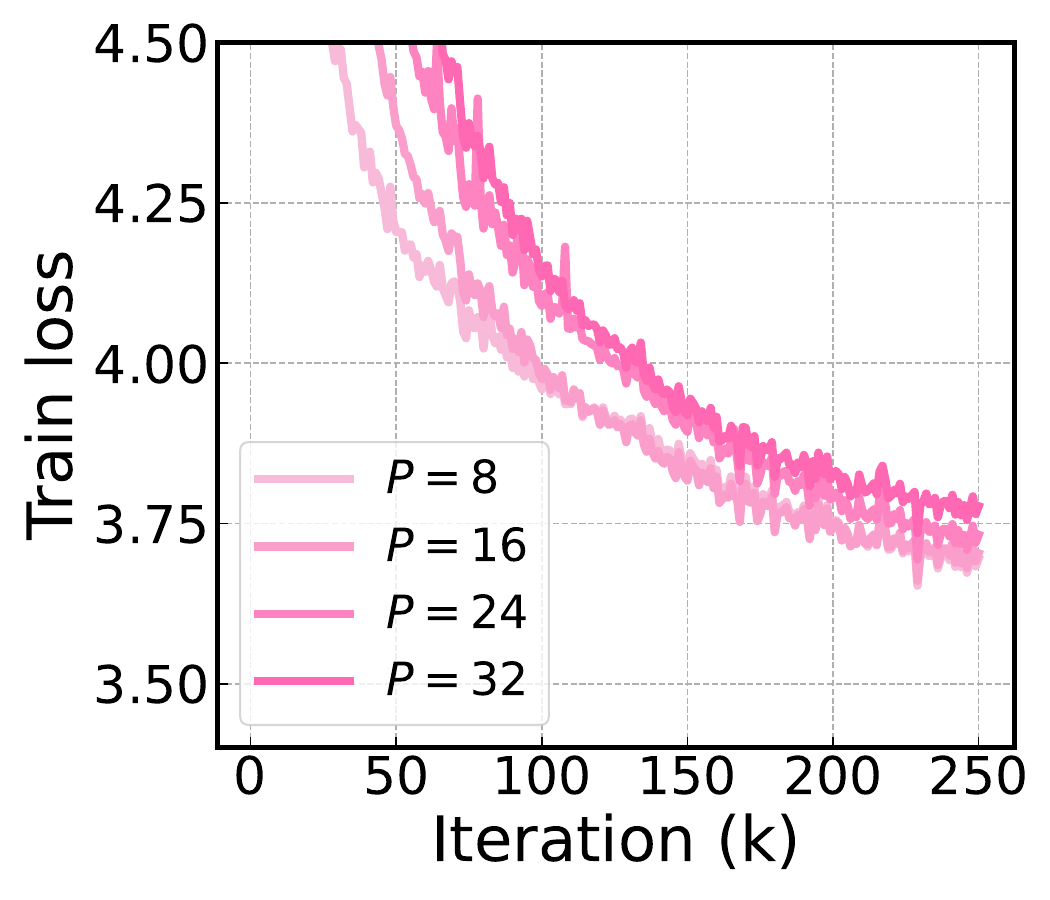}
        \caption{\mean{} / \onesided{}}
        \label{fig:result-appendix-blockscaling-rotatedadam-onesided-mean}
    \end{subfigure}    
    \hspace{1em}
    \begin{subfigure}{0.2\linewidth}
        \includegraphics[width=\linewidth]{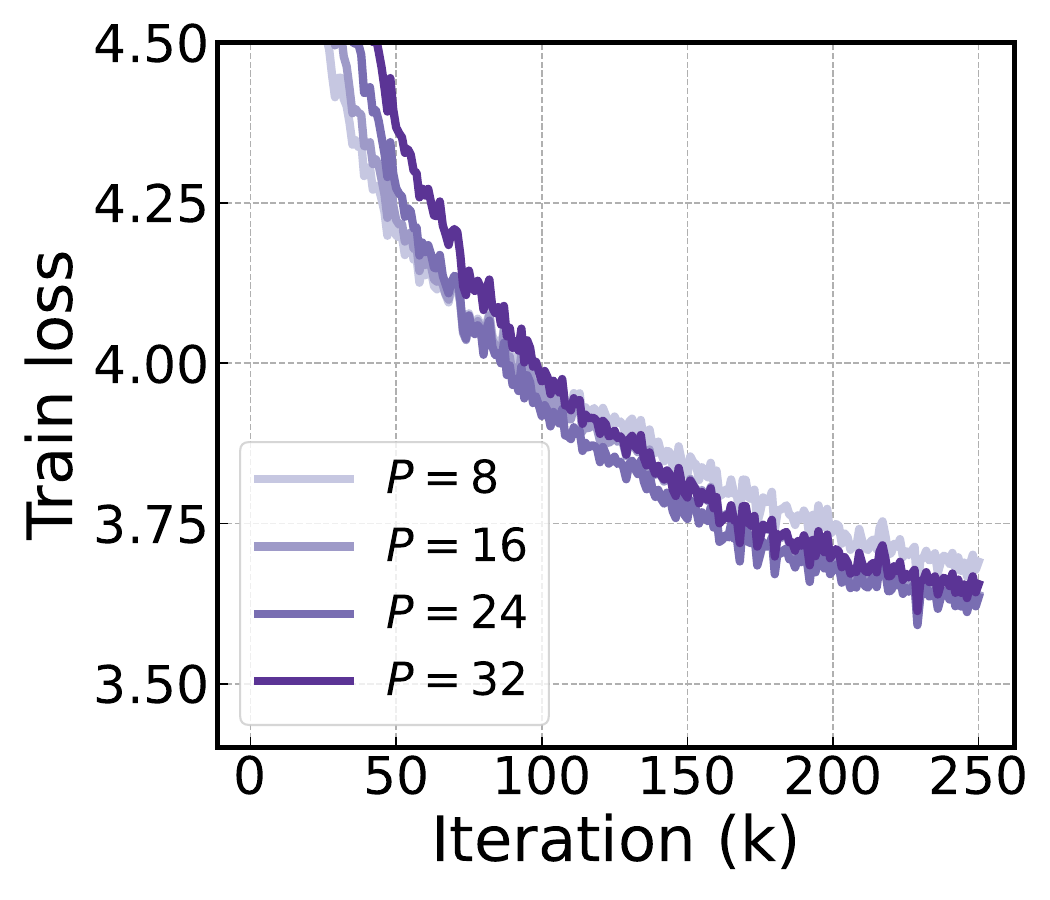}
        \caption{\mean{} / \twosided{}}
        \label{fig:result-appendix-blockscaling-rotatedadam-twosided-mean}
    \end{subfigure}
    \hspace{1em}
    \begin{subfigure}{0.2\linewidth}
        \includegraphics[width=\linewidth]{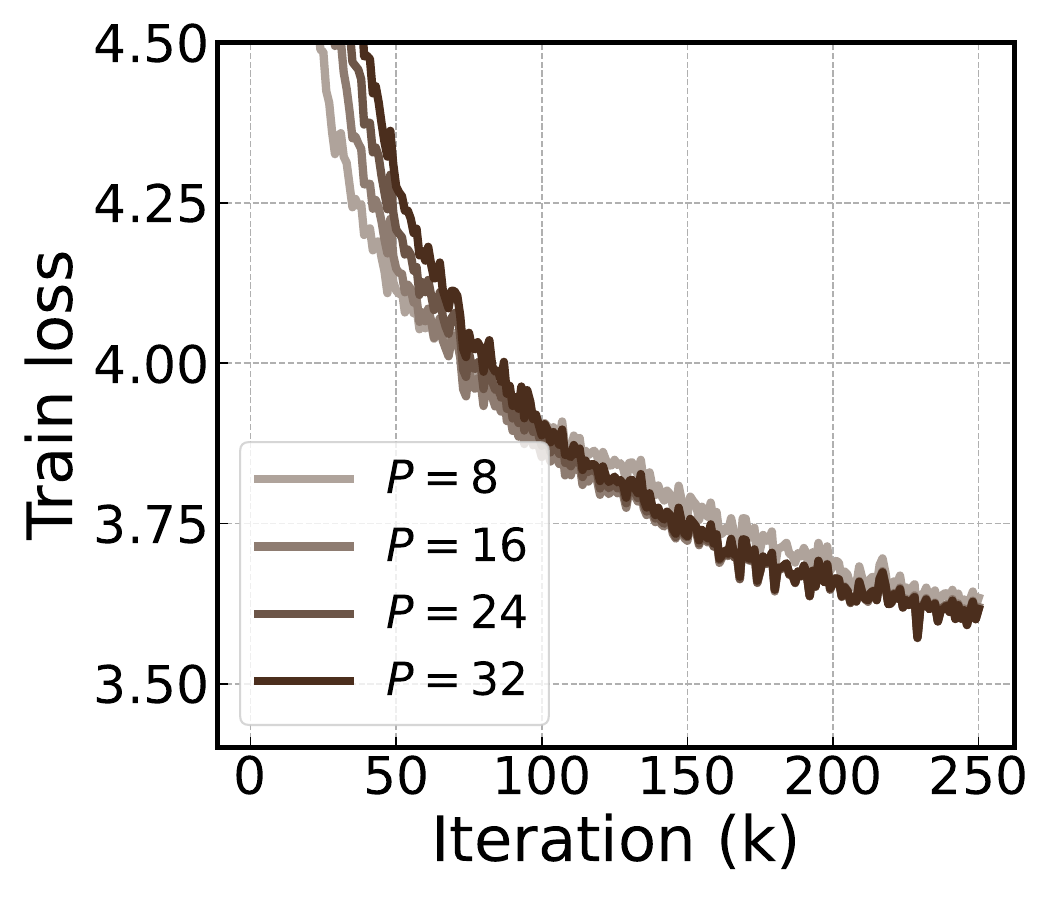}
        \caption{\covariance{} / \onesided{}}
        \label{fig:result-appendix-blockscaling-rotatedadam-onesided-cov}
    \end{subfigure}
    \hspace{1em}
    \begin{subfigure}{0.2\linewidth}
        \includegraphics[width=\linewidth]{figure/main-exp/large-scale/blockscaling/rotatedadam_two_sided_covariance.pdf}
        \caption{\covariance{} / \twosided{}}
        \label{fig:result-appendix-blockscaling-rotatedadam-twosided-cov}
    \end{subfigure}    
    \caption{
    Performance of each method when increasing the number of stages by scaling the number of blocks for baselines (a-c) and \ourmethod{} (d-g).
    While scaling the model rather increases the loss under asynchronous pipeline parallelism for baselines (a-c), scaling works well for \ourmethod{} especially for \covariance{}/\twosided{} strategy (f).
    }
    \label{fig:result-appendix-blockscaling}
\end{figure*}

\begin{figure*}[t]
    \centering
    \includegraphics[width=0.2\linewidth]{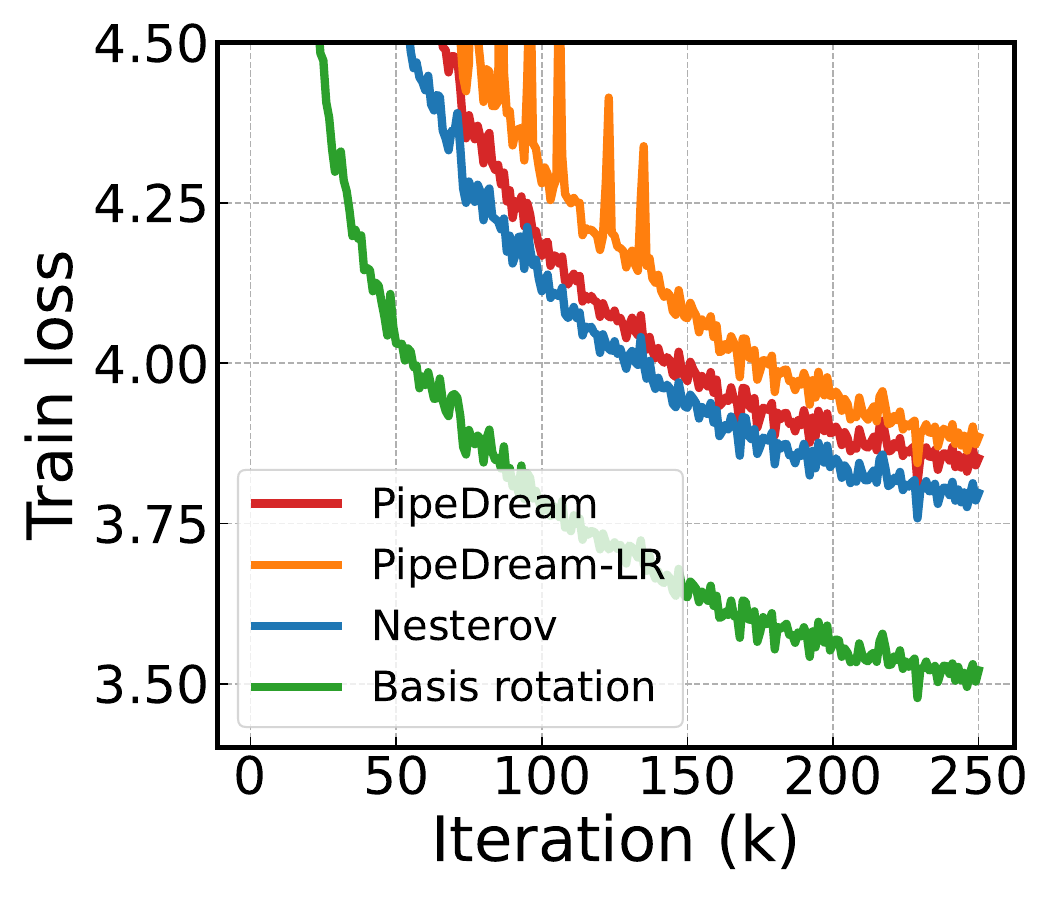}
    \caption{
    Performance of different methods when using PipeMare-style weight prediction \citep{yang2021pipemare} instead of weight stashing for $P=32$.
    Basis rotation outperforms baselines with a large margin also in this setting.}
    \label{fig:result-appendix-pipemare}
\end{figure*}

\begin{figure*}[t]
    \centering
    \begin{subfigure}{0.2\linewidth}
        \includegraphics[width=\linewidth]{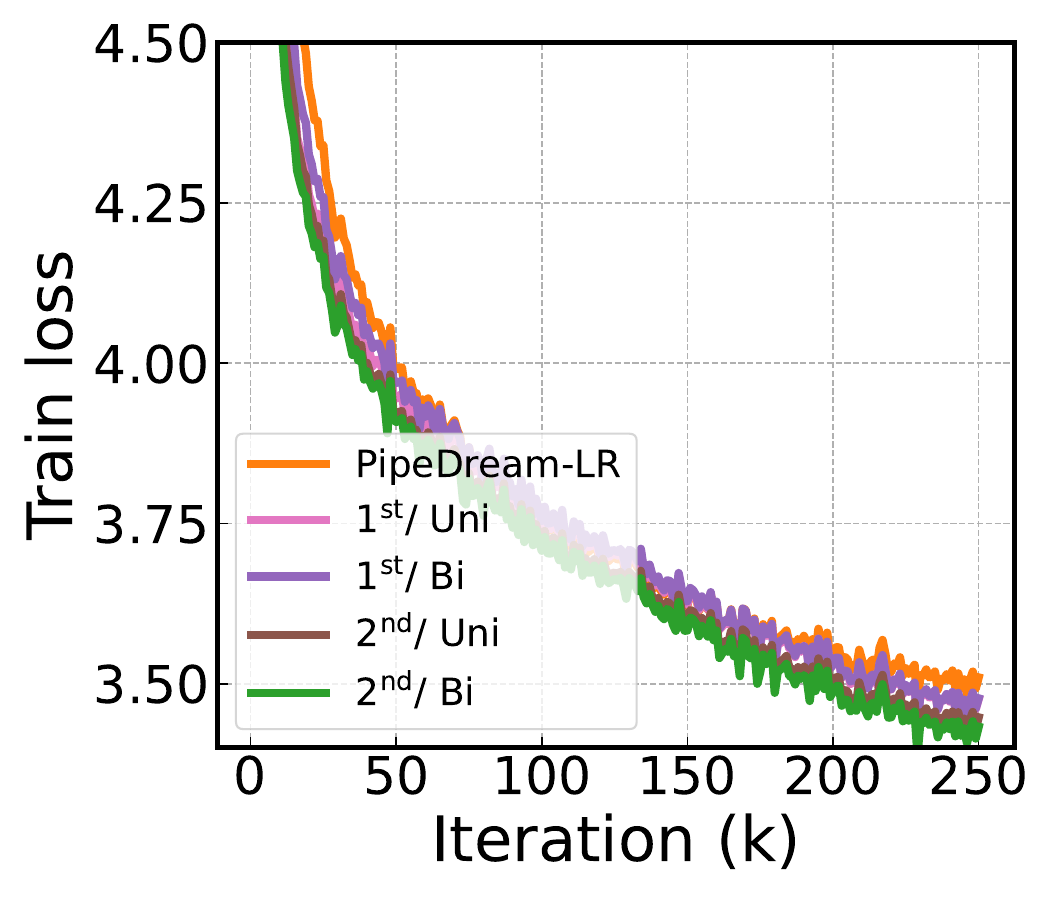}
        \caption{$P=1$}
        \label{fig:result-appendix-optrotated-stage1}
    \end{subfigure}
    \hspace{2em}
    \begin{subfigure}{0.2\linewidth}
        \includegraphics[width=\linewidth]{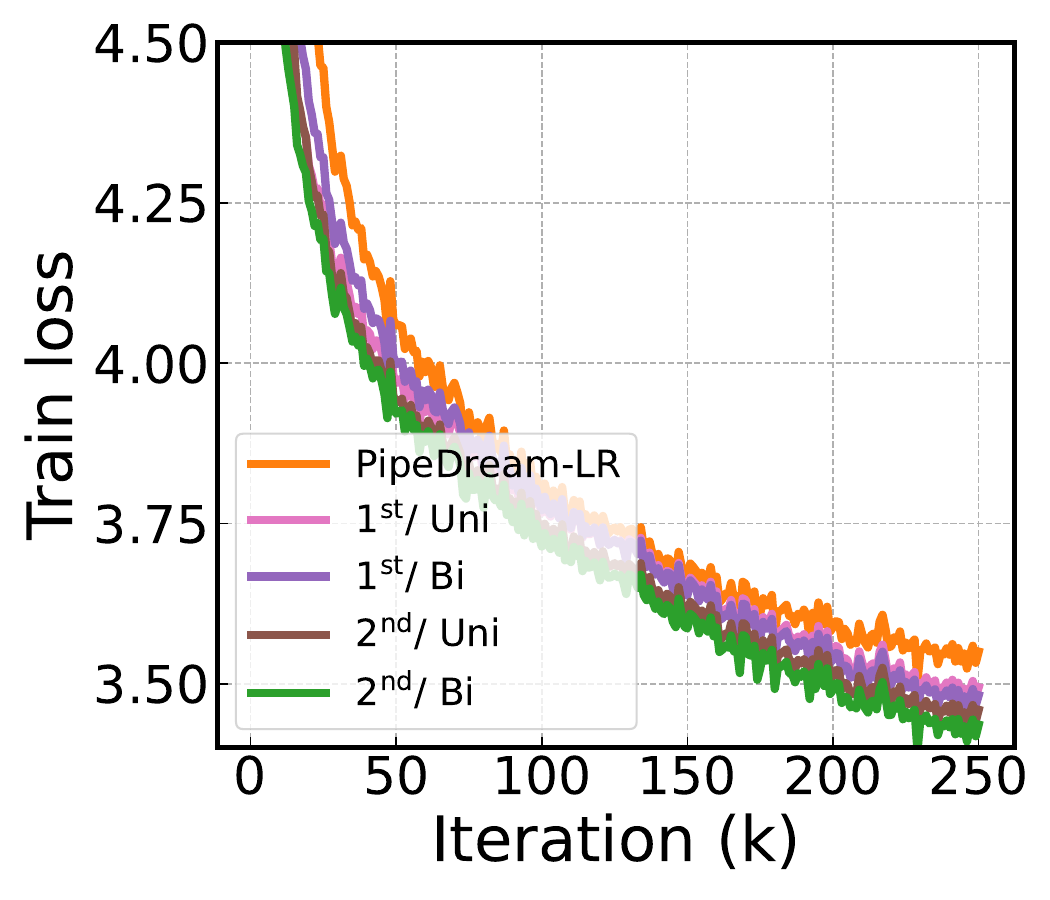}
        \caption{$P=2$}
        \label{fig:result-appendix-optrotated-stage2}
    \end{subfigure}
    \hspace{2em}
    \begin{subfigure}{0.2\linewidth}
        \includegraphics[width=\linewidth]{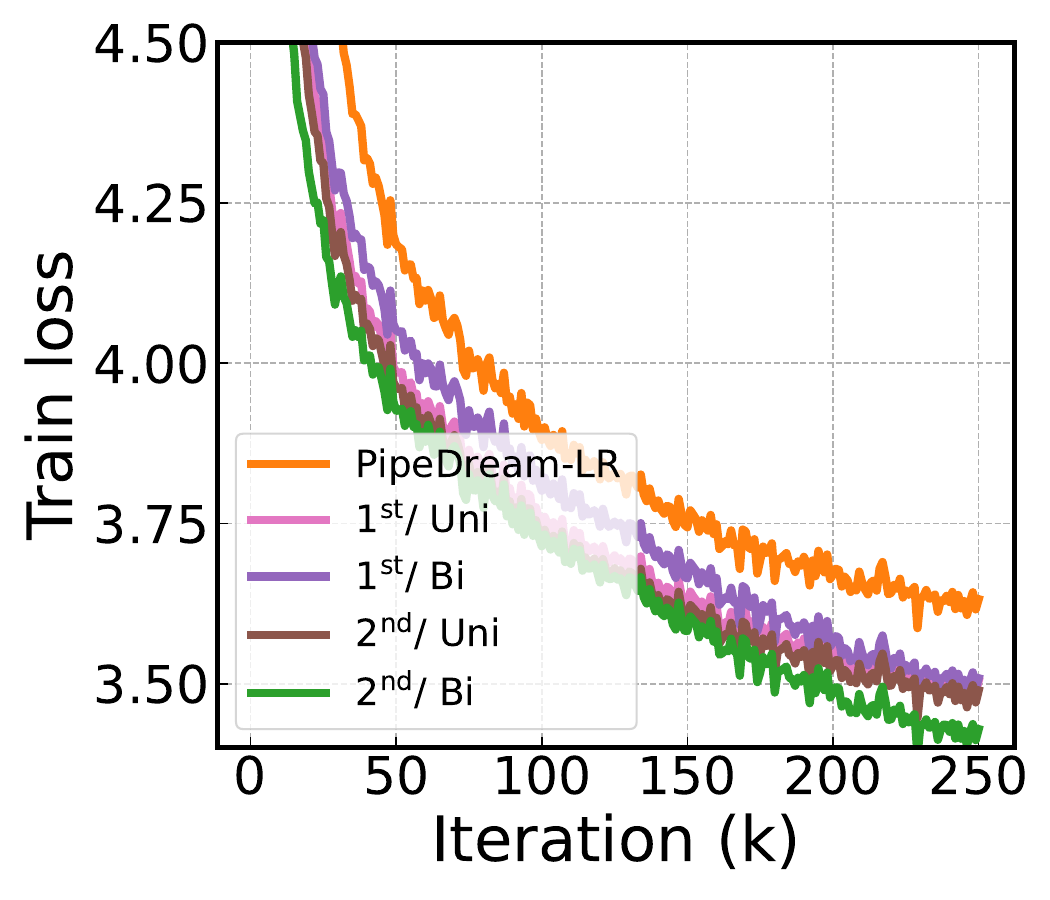}
        \caption{$P=4$}
        \label{fig:result-appendix-optrotated-stage4}
    \end{subfigure} \\
    \begin{subfigure}{0.2\linewidth}
        \includegraphics[width=\linewidth]{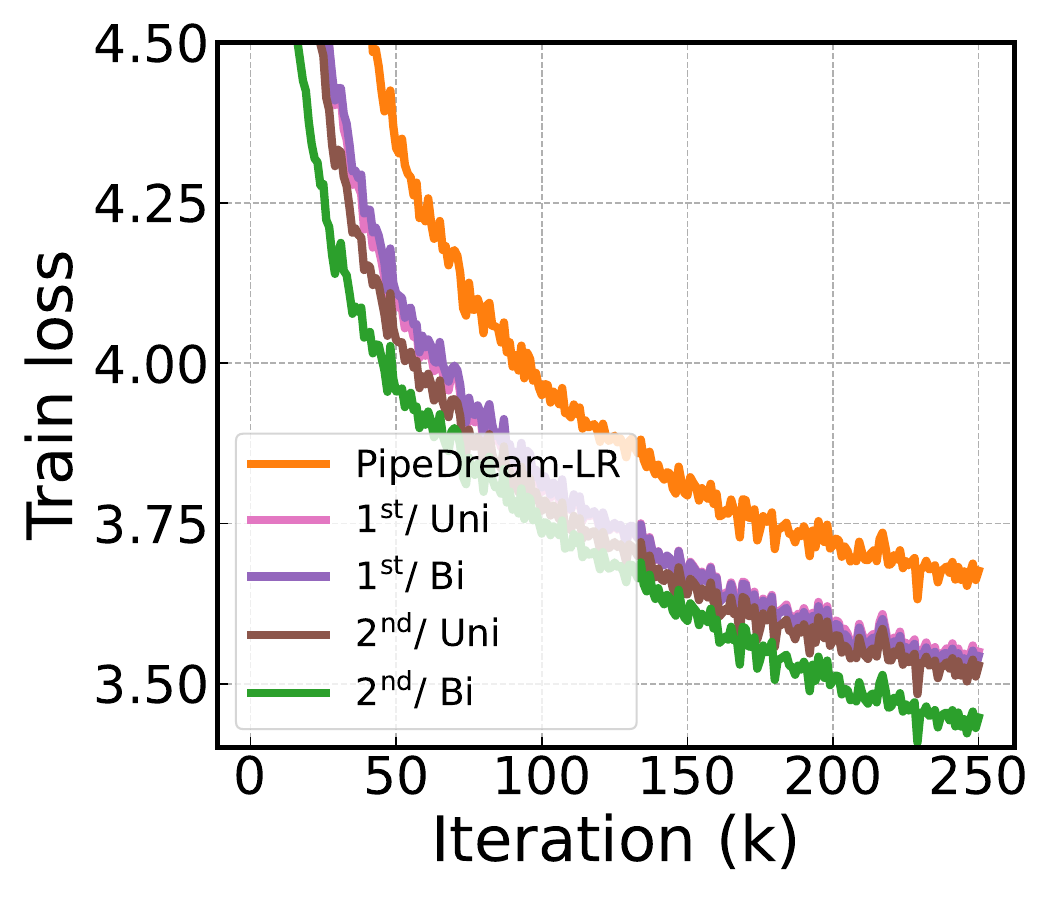}
        \caption{$P=8$}
        \label{fig:result-appendix-optrotated-stage8}
    \end{subfigure}    
    \hspace{2em}
    \begin{subfigure}{0.2\linewidth}
        \includegraphics[width=\linewidth]{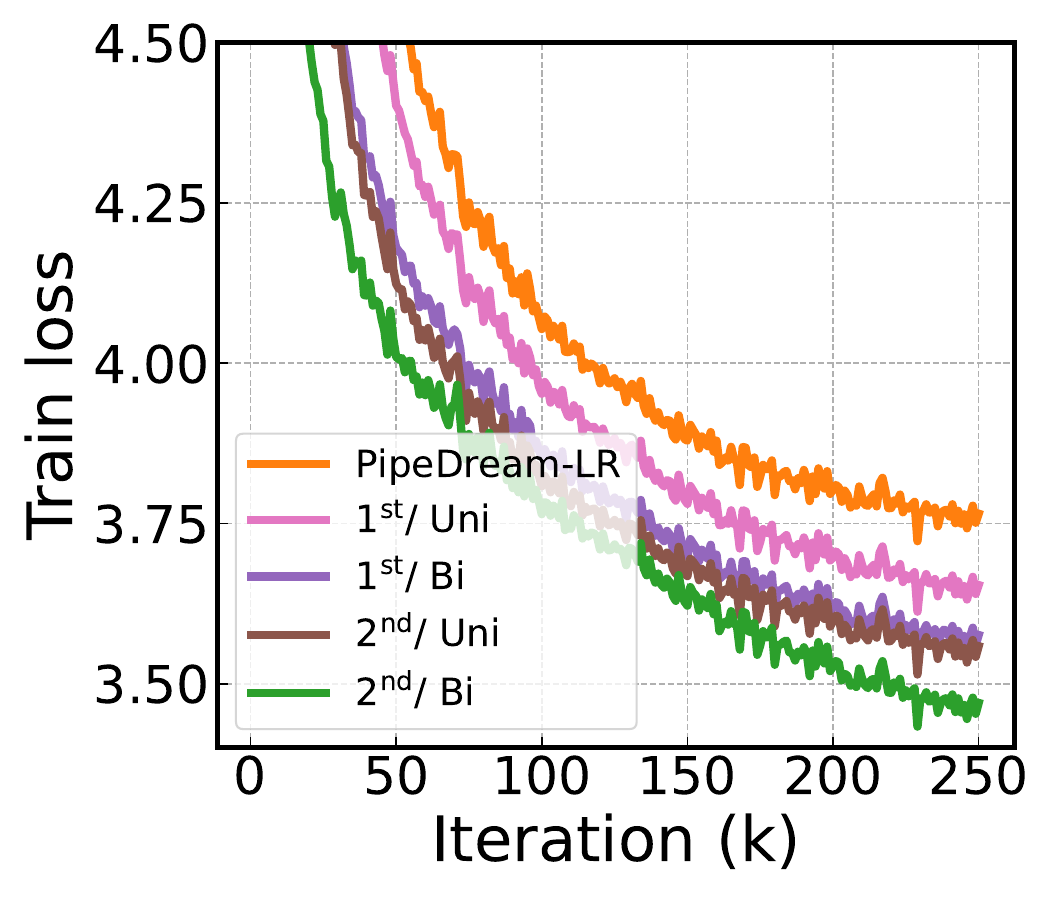}
        \caption{$P=16$}
        \label{fig:result-appendix-optrotated-stage16}
    \end{subfigure}    
    \hspace{2em}
    \begin{subfigure}{0.2\linewidth}
        \includegraphics[width=\linewidth]{figure/main-exp/large-scale/optimizer/stage=32_rotated.pdf}
        \caption{$P=32$}
        \label{fig:result-appendix-optrotated-stage32}
    \end{subfigure}    
    \caption{
    Comparison of different \estimation{} strategies for different number of stages $P$.
    }
    \label{fig:result-appendix-optrotated}
\end{figure*}

\begin{figure*}[t]
    \centering
    \includegraphics[width=0.2\linewidth]{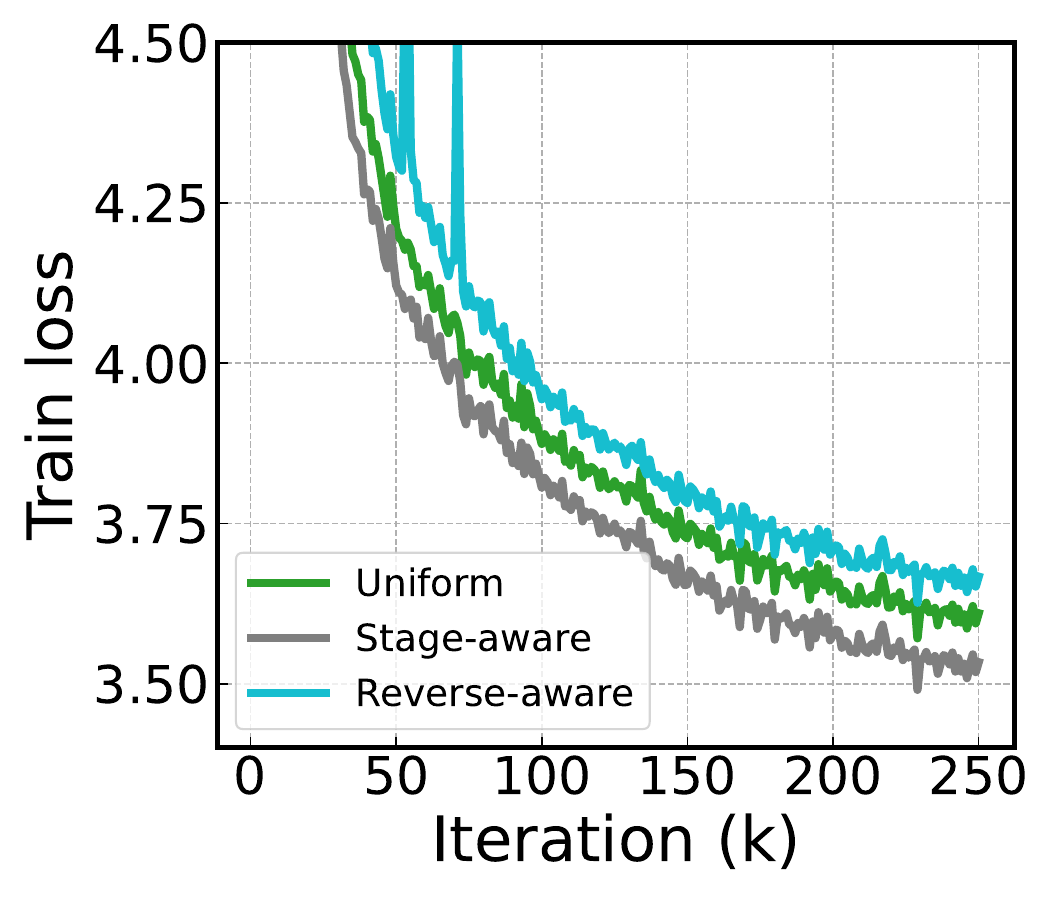}
    \caption{
    Stage-aware basis rotation and inverse stage-aware basis rotation. While stage-aware basis rotation outperforms uniform update frequency, inverse stage-aware basis rotation degrades the performance.}
    \label{fig:reverse-stage-wise-frequency}
\end{figure*}

\begin{figure*}[t]
    \centering
    \begin{subfigure}{0.2\linewidth}
        \includegraphics[width=\linewidth]{figure/main-exp/large-scale/optimizer/stage=32_main.pdf}
        \caption{Train loss}
        \label{fig:result-appendix-train-stage32}
    \end{subfigure}
    \begin{subfigure}{0.2\linewidth}
        \includegraphics[width=\linewidth]{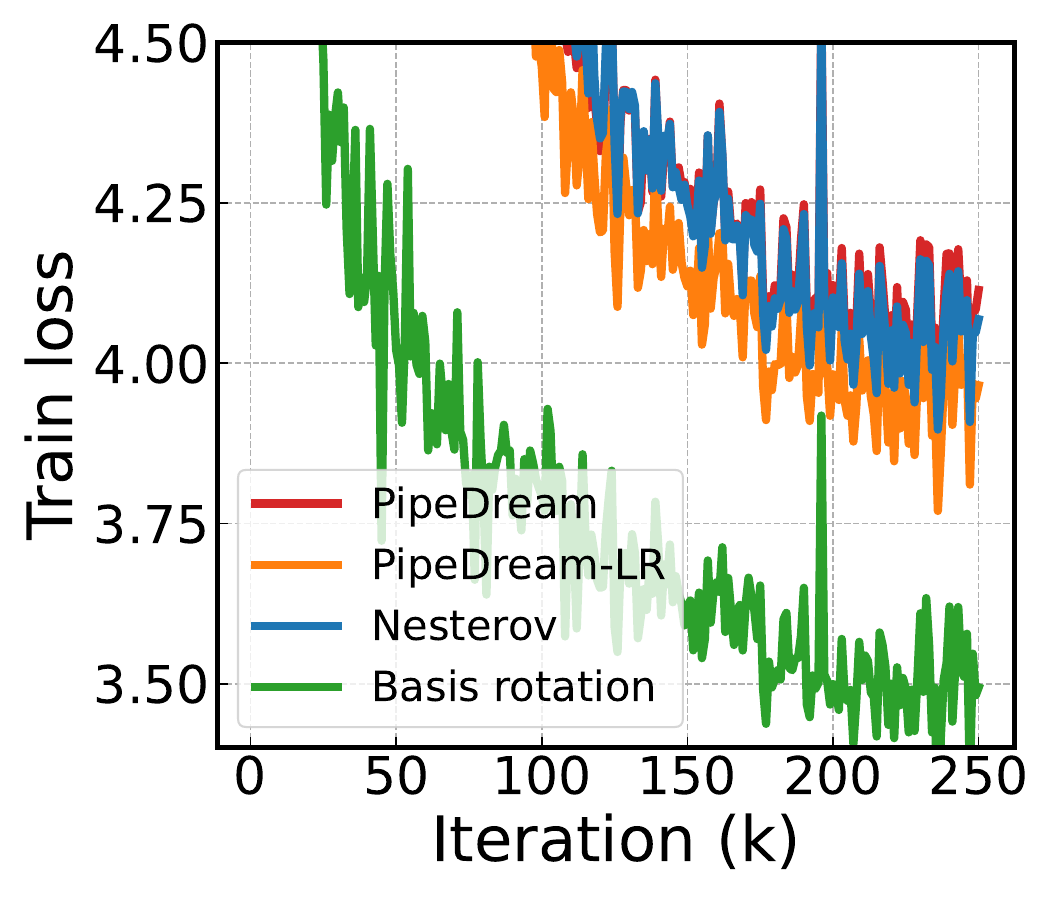}
        \caption{Validation loss}
        \label{fig:result-appendix-val-stage32}
    \end{subfigure}    
    \caption{
    Train loss (a) and validation loss (b) for different methods at $P=32$.
    The trend in validation loss closely follows the trend in train loss.
    }
    \label{fig:result-appendix-val}
\end{figure*}

\begin{figure*}[t]
    \centering
    \begin{subfigure}{0.2\linewidth}
        \includegraphics[width=\linewidth]{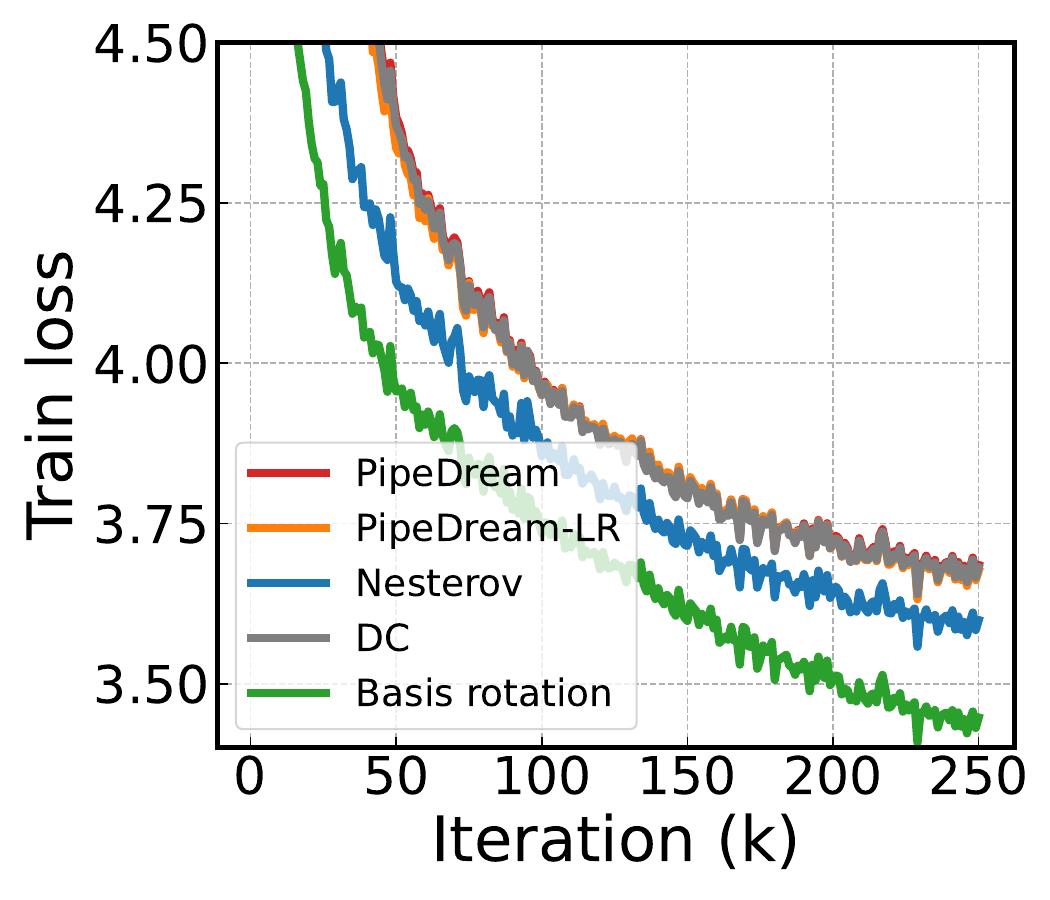}
        \caption{$P=8$}
        \label{fig:result-appendix-dc-stage8}
    \end{subfigure}    
    \hspace{2em}
    \begin{subfigure}{0.2\linewidth}
        \includegraphics[width=\linewidth]{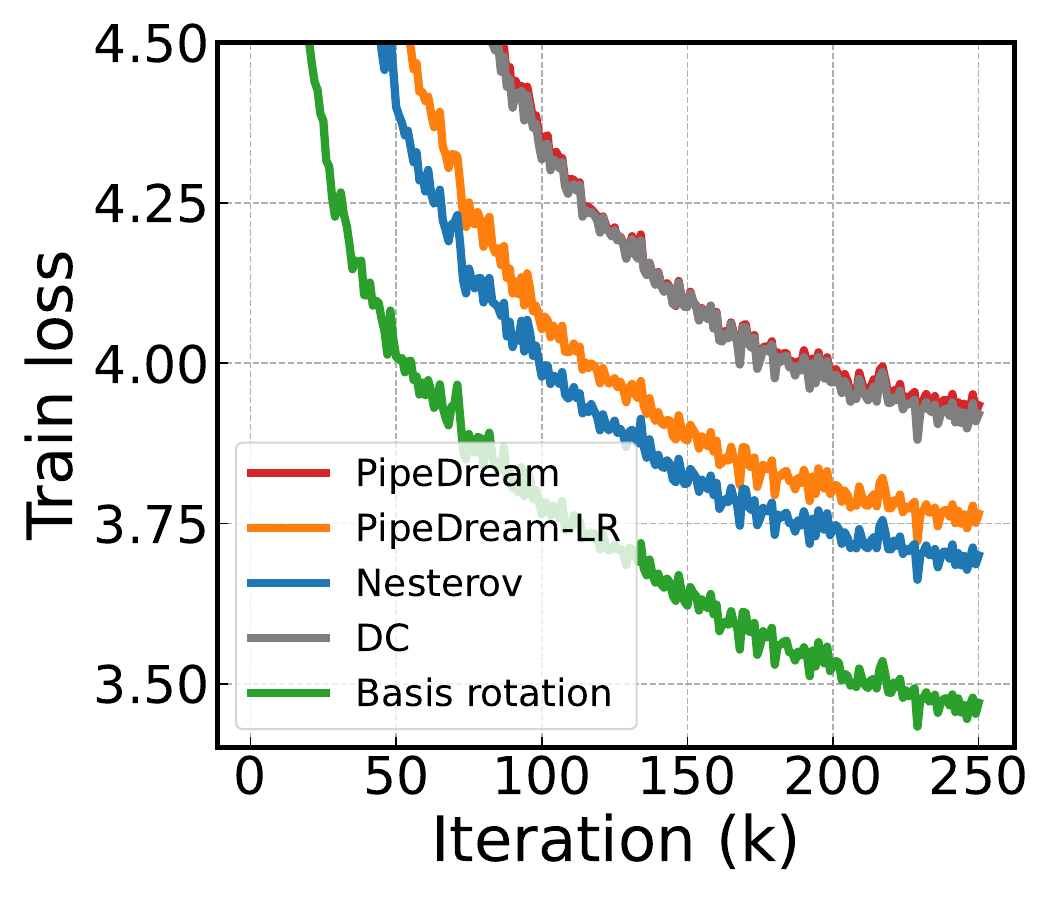}
        \caption{$P=16$}
        \label{fig:result-appendix-dc-stage16}
    \end{subfigure}
    \hspace{2em}
    \begin{subfigure}{0.2\linewidth}
        \includegraphics[width=\linewidth]{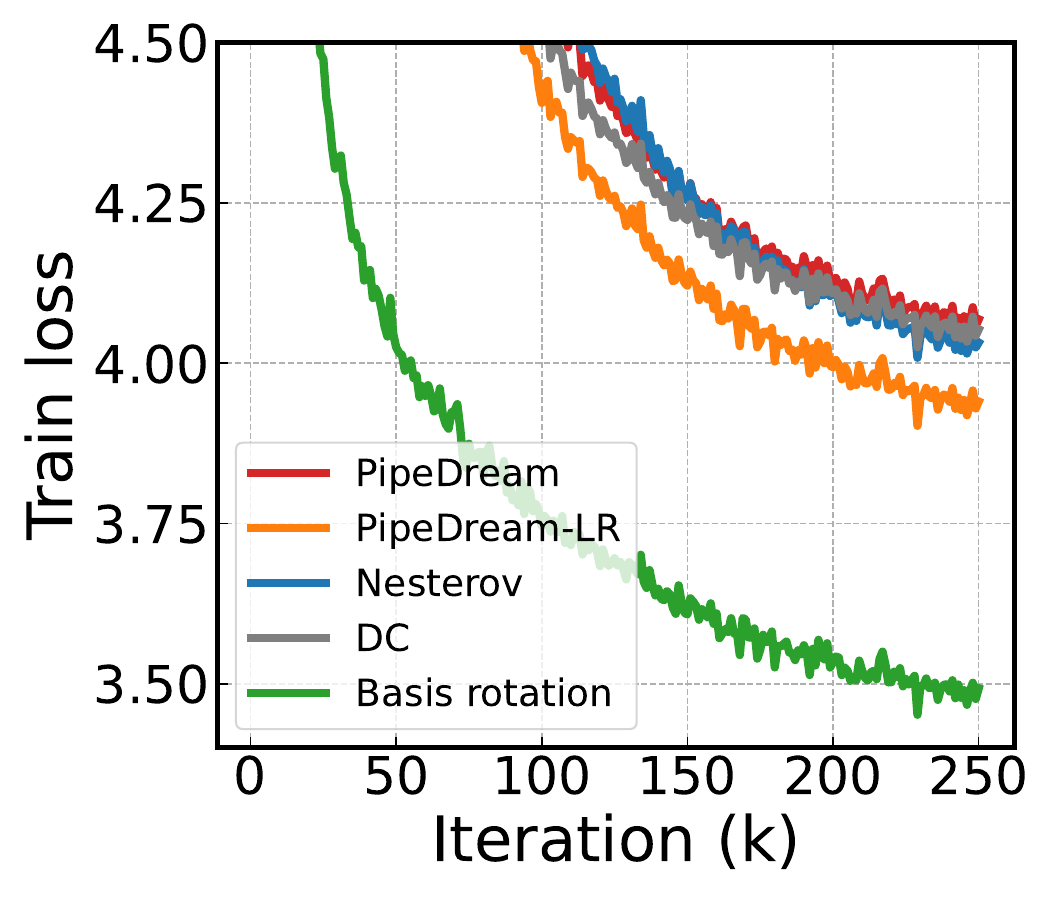}
        \caption{$P=32$}
        \label{fig:result-appendix-dc-stage32}
    \end{subfigure}    
    \caption{
    Comparison of each method for different number of stages $P$ including Delay Compensation (DC) algorithm \citep{zheng2017asynchronous} as baselines.
    DC shows similar performance to PipeDream.
    }
    \label{fig:result-appendix-dc}
\end{figure*}

\paragraph{Full Results of \cref{sec:exp-main}}

We plot the full results of \cref{fig:result-main} in \cref{fig:result-appendix-optmain,fig:result-appendix-slowdown,fig:result-appendix-blockscaling}.
As seen in \cref{fig:result-appendix-optmain}, the gap between \ourmethod{} and baselines becomes larger with increasing number of stages $P$.
We also find that the slowdown, defined as the iteration ratio required to reach target loss for $P=32$ relative to $P=1$, becomes significantly smaller for \ourmethod{} (see \cref{fig:result-appendix-slowdown}).
We note that this slowdown becomes smaller with more accurate \estimation{} strategies.
Finally, we plot the full results of \cref{fig:result-main-blockscaling} in \cref{fig:result-appendix-blockscaling} where we increase the number of stages by increasing the number of Transformer blocks.
While baseline methods invalidate the standard scaling law with increasing loss for larger models, \ourmethod{}, especially with high-fidelity \estimation{} strategies, recovers the scaling law by gradually decreasing the loss for larger models.

\paragraph{Full Results of \cref{sec:exp-more}}

We plot the full results of \cref{fig:result-rotation} in \cref{fig:result-appendix-optrotated}.
We note that the high-fidelity estimation strategy consistently outperforms its memory-efficient counterpart across all number of stages, with a widening gap for larger stages.
We also note that even the least accurate estimation strategy (\mean{}/\onesided{}) consistently outperforms the best-performing baseline with a widening gap for larger stages.

\paragraph{Results with weight prediction}

We also evaluate the performance of \ourmethod{} when employing PipeMare-style weight prediction \citep{yang2021pipemare} as an alternative strategy for reducing memory overhead. 
The results, shown in \cref{fig:result-appendix-pipemare}, demonstrate that \ourmethod{} maintains its performance advantage and remains significantly more robust than baseline methods even when using approximated weight versions. 
This consistent trend further underscores the resilience of our approach to the gradient inaccuracies typical of asynchronous pipeline parallelism.

\paragraph{Stage-aware basis rotation implementation and more results}

Here, we provide implementation details and additional experimental results for the stage-aware basis rotation strategy. 
Under a fixed computational budget, the subspace update frequency at each pipeline stage is allocated proportionally to the corresponding gradient delay, such that the overall computational cost remains constant across configurations.
Specifically, let $\tau$ denote the gradient delay at the stage. 
We define the update frequency $f$ at each stage as a function of $\tau$ as the following scheduling rule:
\begin{align*}
mid &= \left\lfloor \frac{P}{2} \right\rfloor - 1 \\
n &= \begin{cases} mid - \tau, & \text{if } \tau > mid \\ mid + 1 - \tau, & \text{if } \tau \le mid \end{cases} \\
f &= \left\lfloor \frac{f_0}{1 - \frac{n}{mid}} \right\rfloor
\end{align*}
To further validate the role of basis rotation, we additionally evaluate a reversed stage-aware allocation strategy, in which the update frequency is inversely assigned with respect to gradient delay. 
As shown in \cref{fig:reverse-stage-wise-frequency}, this reversed strategy leads to inferior convergence behavior, confirming the effect of basis rotation for mitigating staleness. 

\paragraph{Results for Validation Loss}

We also plot the validation loss in \cref{fig:result-appendix-val}.
Here, we measure validation loss for $200$ examples every $1000$ iteration. 
We observe that the trend is very similar to the results for train loss.

\paragraph{Experiments for Other Baselines}

We also compare our approach against Delay Compensation (DC) algorithm \citep{zheng2017asynchronous}.
DC was originally proposed for asynchronous data parallelism and later adapted to pipeline parallelism for small-scale vision models \citep{jang2023pipelined}.
Specifically, it uses a first-order Taylor expansion to approximate the fresh gradient $\nabla f(w_t)$ from the delayed gradient $\nabla f(w_{t-\tau})$:
\begin{equation*}
    \nabla f(w_t) \approx \nabla f(w_{t-\tau}) + \lambda \nabla f(w_{t-\tau}) \odot \nabla f(w_{t-\tau}) (w_t - w_{t-\tau}),
\end{equation*}
where $\lambda \in [0, 1]$ is a hyperparameter to control the scale the compensation term and the diagonal empirical Fisher $\nabla f(w_{t-\tau}) \odot \nabla f(w_{t-\tau})$ serves as a diagonal approximation of the Hessian.
We test $\lambda \in \{0.04, 0.1, 0.5, 1.0\}$.

The results are plotted in \cref{fig:result-appendix-dc}.
We observe that DC does not effectively address delayed gradients for large delays and shows similar performance to PipeDream.

\paragraph{Comparison with Preconditioned Methods}

\begin{table}[H]
    \centering
    \caption{
    Slowdown of preconditioned methods at $P=32$ relative to $P=1$.
    Slowdown is defined as the ratio of iterations required to reach a fixed loss threshold at $P=32$ versus $P=1$.
    Methods with explicit basis alignment (SOAP and \ourmethod{}) achieve substantially lower slowdown than those without (Muon and Scion), confirming that basis alignment is the primary factor for mitigating staleness under large delays.
    }
     \label{tab:result-preconditioned-slowdown}
    \resizebox{0.2\textwidth}{!}{
    \begin{tabular}{lc}
        \toprule
        Method & Slowdown \\
        \midrule
        PipeDream-LR   & $5.43\times$ \\
        Nesterov       & $4.24\times$ \\
        \midrule
        Muon           & $1.53\times$ \\
        Scion          & $2.10\times$ \\
        \midrule
        SOAP           & $1.34\times$ \\
        \ourmethod{}   & $1.27\times$ \\
        \bottomrule
    \end{tabular}
    }
\end{table}

We compare \ourmethod{} against preconditioned optimizers that have recently demonstrated strong performance in standard (non-delayed) training settings: SOAP \citep{vyas2024soap}, Muon \citep{jordan2024muon}, and Scion \citep{pethick2025training}.
While these methods were not originally designed for asynchronous pipeline parallelism, evaluating them in this setting provides insight into which algorithmic properties are most critical for delay robustness.
We measure the slowdown of each method, defined as the ratio of iterations required to reach a fixed loss threshold at $P=32$ relative to $P=1$.
The results are summarized in \cref{tab:result-preconditioned-slowdown}.

Muon and Scion, which do not perform explicit alignment with the Hessian eigenbasis, outperform standard baselines but still exhibit considerably higher slowdown than \ourmethod{}.
In contrast, SOAP achieves delay robustness close to \ourmethod{}, as both methods share the principle of operating in a rotation aligned with the Hessian eigenbasis.
These results support the central claim of our work: it is the basis alignment, rather than preconditioning per se, that is essential for robustness to gradient delay in asynchronous pipeline parallelism.
 
We note that the minor implementation differences between \ourmethod{} and SOAP---as described in \cref{subsec:solution:Hessian_approximation} and \cref{app:connection_recent_optimizers}---were not intended to improve upon SOAP's performance.
Rather, they were introduced to provide a controlled experimental environment for systematically analyzing different Hessian approximation strategies, free from optimizer-specific confounds.

\paragraph{Results on 3B Model}

\begin{figure*}[t]
    \centering
    \includegraphics[width=0.2\linewidth]{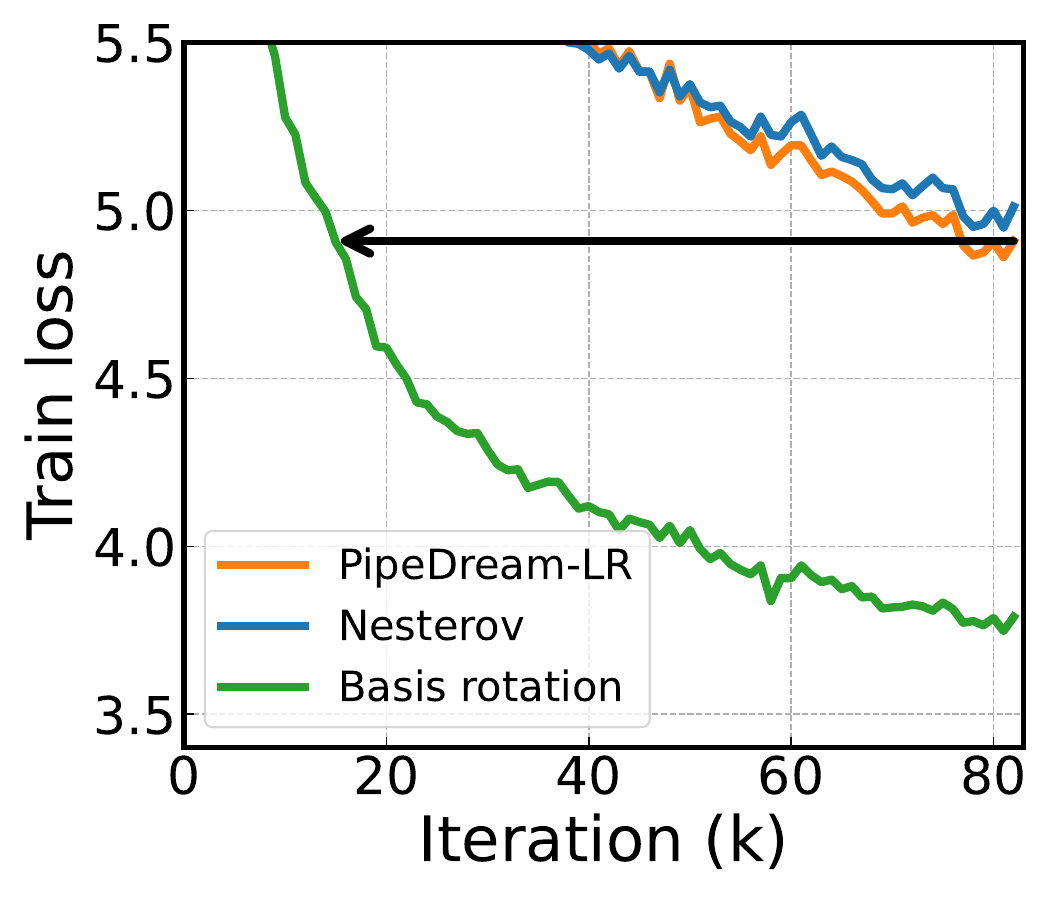}
    \caption{
    Performance of different methods at $P=32$ for a ${\approx}3$B parameter model (2688 embedding dimensions, 32 Transformer blocks).
    Results are shown for partial training up to 82K iterations out of a 250K-step schedule.
    \Ourmethod{} achieves the same training loss with $81.7\%$ fewer iterations than the best-performing baseline.
    }
    \label{fig:result-appendix-3b}
\end{figure*}

We evaluate \ourmethod{} on an approximately 3B parameter model ($3{,}047$M parameters; 2688 embedding dimensions, 32 Transformer blocks) with $P=32$ pipeline stages.
We search over learning rates in $\{10^{-4}, 3\times10^{-5}\}$ with a 250K-step cosine schedule, and report partial training curves up to 82K iterations.
As shown in \cref{fig:result-appendix-3b}, \ourmethod{} achieves the same training loss with $81.7\%$ fewer iterations than the best-performing baseline, further widening the advantage observed at 1B scale.
 
\paragraph{Generalization to MoE}

\begin{figure*}[t]
    \centering
    \includegraphics[width=0.2\linewidth]{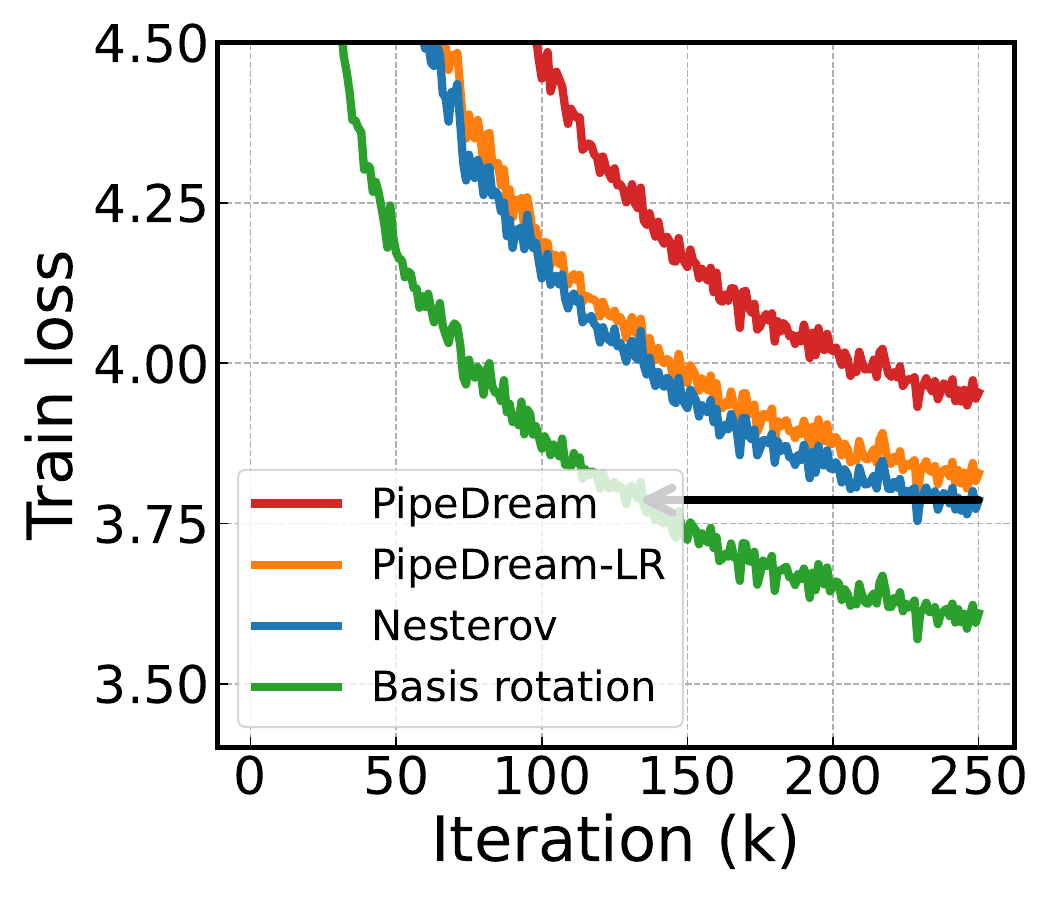}
    \caption{
    Performance of different methods on a 100M nanoMoE model (8 experts, top-2 activation).
    \Ourmethod{} achieves the lowest final training loss and reduces the number of iterations
    required to reach the same loss as the best-performing baseline by $46.8\%$.}
    \label{fig:result-appendix-moe}
\end{figure*}

We evaluate whether \ourmethod{} generalizes beyond standard Transformer FFN layers to Mixture-of-Experts (MoE) architectures.
Since MoE routing and expert computations are performed within a single GPU without altering the pipeline schedule, basis rotation can be applied to each expert's weight matrices independently without any modification to the pipeline scheduling logic.
We train a 100M nanoMoE model \footnote{\url{https://github.com/wolfecameron/nanoMoE}} with eight experts and top-2 activation on the same language modeling setup described in \cref{app:sub:LLM-exp-details}, and compare against the same baselines used in the main experiments.

The results in \cref{fig:result-appendix-moe} show that \ourmethod{} achieves the lowest final training loss among all methods, reducing the iterations required to reach the same loss as the best-performing baseline by $46.8\%$.
The consistent gains over baselines demonstrate that basis rotation effectively addresses gradient delay in MoE models, confirming that the benefits of our approach are not limited to dense Transformer architectures.